\documentclass[10pt]{article}

\usepackage[table,xcdraw,dvipsnames]{xcolor}

\usepackage[utf8]{inputenc} 
\usepackage[T1]{fontenc}    
\usepackage{fullpage}
\usepackage{url}            

\usepackage{hyperref}       

\usepackage{amssymb}
\usepackage{amsthm}
\usepackage{amsfonts}
\usepackage{amsmath}
\usepackage{cleveref}
\usepackage{amssymb}
\usepackage{algorithm}
\usepackage{algorithmic}
\usepackage{bbm}
\usepackage[square, numbers]{natbib}
\usepackage{tikz}
\usepackage{caption}
\usepackage{subfigure}
\usepackage{url}            
\usepackage{booktabs}       
\usepackage{nicefrac}       
\definecolor{crimson}{rgb}{0.86, 0.08, 0.24}
\definecolor{dodgerblue}{rgb}{0.12, 0.56, 1.0}
\hypersetup{colorlinks=true,linkcolor=crimson, citecolor=dodgerblue}
\usepackage{epsf}
\usepackage{graphics}
\usepackage{psfrag}

\usepackage{color}

\usepackage{mathtools}
\usepackage{amsfonts}
\usepackage{amsmath}
\usepackage{amssymb}

\newtheorem{theorem}{Theorem}
\newtheorem{proposition}{Proposition}
\newtheorem{lemma}{Lemma}

\newtheorem{definition}{Definition}
\newtheorem{remark}{Remark}

\newtheorem{assumption}{Assumption}
\newtheorem{exmp}{Example}

\newcommand{\order}{\ensuremath{\mathcal{O}}}

\newcommand{\w}{\bm{w}}
\newcommand{\al}{\bm{\alpha}}

\long\def\comment#1{}

\newcommand{\bm}[1]{\boldsymbol{#1}}
\newcommand{\llVert}{\left\lVert}
\newcommand{\rrVert}{\right\rVert}

\usepackage[english]{babel}
\usepackage[utf8]{inputenc}
\usepackage[ruled,algo2e]{algorithm2e}
\usepackage{appendix}
\usepackage{multirow}
\usepackage{pifont} 
\usepackage{chngcntr}
\usepackage{wrapfig}
\usepackage{enumerate}
\usepackage{scrwfile}
\usepackage[stable]{footmisc}

\newlength{\widebarargwidth}
\newlength{\widebarargheight}
\newlength{\widebarargdepth}

\let\origtheassumption\theassumption

\definecolor{aoenglish}{rgb}{0.0, 0.5, 0.0}
\definecolor{burgundy}{rgb}{0.5, 0.0, 0.13}

\usepackage{fontawesome}

\definecolor{crimson}{rgb}{0.86, 0.08, 0.24}
\definecolor{dodgerblue}{rgb}{0.12, 0.56, 1.0}

\SetCommentSty{mycommfont}
\newlength{\commentWidth}
\TOCclone[Table of Contents]{toc}{atoc}
\newcommand\StartAppendixEntries{}
\AfterTOCHead[toc]{%
  \renewcommand\StartAppendixEntries{\value{tocdepth}=-10000\relax}%
}
\AfterTOCHead[atoc]{%
  \edef\maintocdepth{\the\value{tocdepth}}%
  \value{tocdepth}=-10000\relax%
  \renewcommand\StartAppendixEntries{\value{tocdepth}=\maintocdepth\relax}%
}
\newcommand*\appendixwithtoc{%
\hypersetup{linkcolor=black}
  \addtocontents{toc}{\protect\StartAppendixEntries}
  \listofatoc
 \hypersetup{linkcolor=crimson}
}

\date{}
\usepackage{newtxtext}

\begin{document}

\begin{center}
{\bf{\LARGE{{Pareto Efficient Fairness in Supervised Learning: \\[5pt] From Extraction to Tracing}}}}
\vspace*{.2in}

{\large{
 \begin{tabular}{cccc}
  Mohammad Mahdi Kamani$^\dagger$ & Rana Forsati$^\star$ &   James Z. Wang$^\dagger$ & Mehrdad Mahdavi$^\dagger$\\
 \end{tabular}
 }}

 \vspace*{.2in}

 \begin{tabular}{cc}
 \begin{tabular}{c}
$^\dagger$The Pennsylvania State University \\
\texttt{ \{mqk5591,jzw11,mzm616\}@psu.edu}
\end{tabular}
&
\begin{tabular}{c}
$^\star$Microsoft Bing \\
\texttt{raforsat@microsoft.com}
\end{tabular}
 \end{tabular}

 \vspace*{.1in}
 
\end{center}

\begin{abstract}%
As algorithmic decision-making systems are becoming more pervasive, it is crucial to ensure such systems do not become mechanisms of unfair discrimination on the basis of gender, race, ethnicity, religion, etc. Moreover, due to the inherent trade-off between fairness measures and accuracy, it is desirable to learn fairness-enhanced models without significantly compromising the accuracy. In this paper, we propose Pareto efficient Fairness (PEF) as a suitable fairness notion for supervised learning, that can ensure the optimal trade-off between overall loss and other fairness criteria. The proposed PEF notion is definition-agnostic, meaning that any well-defined notion of fairness can be reduced to the PEF notion. To efficiently find a PEF classifier, we cast the fairness-enhanced classification as a bilevel optimization problem and propose a gradient-based method that can guarantee the solution belongs to the Pareto frontier with provable guarantees for convex and non-convex objectives. We also generalize the proposed algorithmic solution to extract and trace arbitrary solutions from the Pareto frontier for a given  preference over accuracy and fairness measures. This approach is generic and can be generalized to any multicriteria optimization problem to trace points on the Pareto frontier curve, which is interesting by its own right. We empirically demonstrate the effectiveness of the PEF solution and the extracted Pareto frontier on real-world datasets compared to state-of-the-art methods.
\end{abstract}
\section{Introduction}\label{sec:intro}
Attracting momentous attention during the past few years, machine learning models are significantly impacting nearly every aspect of modern society we live in. Hence, it is of paramount importance to study how these models are affecting our lives, and to what extent they  are upholding to the moral standards of our society. The degree of fairness in algorithmic decision-making systems has been the center of a heated debate in numerous fields, such as criminal justice~\citep{propublica,tolan2019machine}, advertising~\citep{dwork2018fairness}, admissions~\citep{marcinkowski2020implications}, and hiring~\citep{lambrecht2019algorithmic,bogen2018help}. Recently, the concerns about algorithmic fairness have resulted in a resurgence of interest to  develop fairness-aware  predictive models to ensure such models do not become a source of unfair discrimination on the basis of gender, race, ethnicity, religion, etc. Devising a fairness-enhanced mechanism necessitates i)  a precise   \textit{fairness measure} to quantify fairness by taking into account the proper legal, ethical, and social context and ii) developing an \textit{algorithmic solution} to learn models that guarantee such metrics during test time. 

Inspired by legal notions of fairness, a great deal of efforts has been invested  in defining  the right notion of  algorithmic fairness  such as demographic  (or statistical) parity~\citep{dwork2012fairness}, equalized odds and equality of opportunity~\citep{hardt2016equality}, individual fairness~\citep{dwork2012fairness,kusner2017counterfactual}, representational fairness~\citep{zemel2013learning}, and fairness under composition~\citep{dwork2018group}. Yet, there is a lack of consensus on which definition can best satisfy fairness requirements as different measures exhibit different advantages and disadvantages. Interestingly, several studies suggest that some of these definitions are incompatible with each other~\citep{pleiss2017fairness}.

From an algorithmic standpoint, we seek to design machine learning algorithms that yield predictive models, which are demonstrably fair  and  bolster to alleviate the impact of unfair decisions~\citep{hardt2016equality,zafar2015fairness,zafar2017parity,donini2018empirical,pleiss2017fairness}. This can be accomplished  by changing the learning procedure, whether in pre-processing,  in-process training, or post-processing stages~\citep{barocas2017fairness}. Most of the existing algorithmic or in-process approaches mainly aim at solving a constrained optimization problem  by imposing a constraint on the level of fairness while optimizing the main learning objective, {\it e.g.}, accuracy~\citep{donini2018empirical,woodworth2017learning,zafar2017parity}. However, due to the inherent trade-off between accuracy and fairness which is  demonstrated in recent studies both empirically and theoretically~\cite{kearns2019ethical,lipton2017does,zafar2017parity}, simply imposing fairness constraint to the main learning objective may \textit{significantly compromise accuracy}. That is because there is no such thing as a \textit{free lunch}, meaning, imposing fairness constraints to the main learning task, introduces a trade-off between these objectives. These trade-offs between fairness constraints and accuracy have been asserted in several studies~\citep{kearns2019ethical,lipton2017does,zafar2017parity,kamani2019efficient,wick2019unlocking,kamani2020multiobjective}.

In light of the incompatibility of fairness measures and the inherent trade-off between accuracy and fairness -- as perusing a higher degree of fairness will compromise accuracy, a fundamental research question is:  \emph{for a given fairness measure(s), can we learn a model that allows for higher fairness without significantly compromising accuracy?} And a more important question is: \emph{Can we find a set of solutions with different levels of the trade-off between accuracy and fairness, so a decision-maker can choose from depending on the preference of one objective over the other in an application?}

Answers to these questions need to be compatible with the fairness measure while minimizing the price of fairness, {\it i.e.}, the relative reduction in the accuracy under the fair solution compared to the best accuracy without fairness consideration. Dealing with multiple and possibly conflicting objectives, a conspicuous approach is to seek a Pareto optimal solution to guarantee optimal compromises between accuracy and fairness.  A model is Pareto  optimal if there is no alternative model to make one of the objectives (accuracy or fairness) better off without making the other worse. To find solutions with different level of trade-off between accuracy and fairness we ought to characterize their Pareto frontier in the objective space.

 The  main practical difficulty of the aforementioned constrained-based minimization approaches such as FERM~\citep{donini2018empirical} is that in general, the obtained solution might not be a Pareto optimal solution. As a result, the optimal trade-off between fairness measure and accuracy is not guaranteed.  Moreover, to manage non-convex constraints resulted from fairness constrains, one needs to employ some relaxations, such as linear relaxation, to make the optimization problem  more tractable.  Furthermore, most of the state-of-the-art algorithms for fairness have been proposed for binary sensitive ({\it i.e.}, protected) features in binary classification problems; however, the generalization of these methods to multiple group sensitive features, or multi-label classification tasks are limited or computationally inefficient~\citep{donini2018empirical,zafar2015fairness,samadi2018price}.

A remarkable attempt to achieve better accuracy-fairness trade-offs  has been made in~\cite{agarwal2018reductions}, where the fairness constrained optimization problem is cast as a saddle point problem, and  an approximate equilibrium is sought by utilizing  the exponentiated-gradient algorithm.  While this approach shows promising experimental results, however, the sequence of randomized classifiers returned by their algorithm requires solving a cost-sensitive classification at each computationally burdensome step. Also, the non-dominating solutions resulting from their algorithm are not necessarily from the Pareto frontier, and hence, they might not be Pareto efficient. Finally, the proposed grid-search reduction to learn a deterministic fair classifier from randomized classifiers is prohibitively costly for non-binary sensitive features. A similar attempt to find Pareto efficient points of the minmax problem for fairness-aware learning has been recently made in~\cite{martinezminimax}. Despite its similarity with our goal in finding the Pareto efficient points in fair learning, their approach is limited to only one fairness measure (equal risk among groups) and have different objective vector than our proposal, hence, they are seeking points from a different Pareto frontier. Also, since the total accuracy is not included in their objective vector, their solutions do not show the accuracy-fairness trade-off. In addition, they use a strong condition of convexity for Pareto frontier, which is not always met, even when all the objectives are convex. Finally, they do not provide a solution to find points from different parts of the Pareto frontier.

 To mitigate these issues, in this paper, we introduce the notion of Pareto efficient fairness (PEF)\footnote{Notion of PEF is also used in~\cite{balashankar2019fair} for their specific fairness measure that is different from our definition of the PEF here. Our notion is more generic and is not bounded to any fairness measure. For more discussion refer to Section~\ref{sec:related}.}, through which we indicate the optimal trade-off between fairness and accuracy. This notion is definition-agnostic, and encompasses many other previously studied definitions of fairness as special cases. Moreover, it is straightforward to adopt the proposed PEF notion to multiple group sensitive features or other machine learning problems beyond binary classification. To learn a PEF classifier,  we propose a bilevel Pareto descent optimization algorithm  that can be utilized by any classification model-- which is trainable  via gradient descent, to generate a fairness-enhanced model from Pareto frontier of the problem. 
We rigorously analyze the convergence of the proposed algorithm to a Pareto optimal solution which guarantees the optimal accuracy-fairness trade-off. 

An answer to the second question requires characterizing the entire Pareto frontier set to extract solutions with desired levels of trade-off between accuracy and fairness. This task by itself is quite challenging and still ongoing research question in learning with multiple objectives (please see~\citep{das1997closer,mahapatramulti,ma2020efficient} for  various dedicated studies trying to tackle this problem). In this paper, we propose a novel first-order algorithm to trace points from the Pareto frontier, which can be generalized to any other multi-criteria optimization problem. To the best of our knowledge, this is the first proposal that is specifically designed to trace points on the Pareto frontier of the fairness-aware learning problem.  

\paragraph{Contributions.}Below, we clarify the relationship and differences between our work and earlier research in fair classification. In particular, the main contributions of this paper consists of proposing an algorithmic solution unifying and extending existing  methods in several aspects:
\begin{itemize}
    \item This work introduces a general algorithmic framework for fairness-aware classification, that can be effectively  applied to different notions of fairness, with convergence guarantees to the Pareto stationary points of the optimization problem. Unlike prior methods for fairness-aware learning, the proposed framework does not employ any relaxation assumptions for objectives in the optimization problem, hence it could achieve state-of-the-art results using a gradient descent based method.
    \item We propose a  bilevel structure  to learn a single Pareto efficient classifier, which is not only novel from methodology perspective, but also grants us an efficient tool for convergence analysis of these optimization problems. This structure can pave the way for the convergence analysis of the stochastic multi-objective optimization in future work. 
    \item We propose an algorithmic solution to trace points on the Pareto frontier of the fairness-accuracy, which not only outperforms other state-of-the-art methods but also provides a set of solutions with different levels of compromises to choose from depending on the application. This novel approach is not bounded only to fairness-aware learning problems and can be applied to any optimization problem dealing with multiple objectives.
\end{itemize}

\paragraph{Organization.}The rest of this paper is organized as follows.
Section~\ref{sec:framework} introduces the Pareto efficient fairness and discusses how to reduce known notions of fairness to an instance of it. We then propose a novel bilevel optimization method to find a single PEF classifier using gradient descent. In Section~\ref{sec:pf}, we investigate the geometry of the Pareto frontier and propose our novel algorithm for tracing points on the Pareto frontier. Section~\ref{sec:converge} lays the theoretical groundwork for the proposed algorithm in various settings and analyzes its convergence rates. Section~\ref{sec:exp} provides extensive empirical results that corroborate our claims and theoretical findings. In Section~\ref{sec:related}, we discuss additional related works to this work. All the proofs as well as some discussions on multiobjective optimization are deferred to the appendix.   

\section{Pareto Efficient Fairness}\label{sec:framework}

In this section, we formally set up our problem and introduce a few key concepts for our algorithm design and analysis. In a typical fairness-aware supervised learning scenario, we have $n$ i.i.d. training examples in the form of $\mathcal{D}=\left\{(\bm{x}_1,a_1,y_1),\ldots,(\bm{x}_n,a_n,y_n) \right\}$, where $a_i \in \mathcal{A}$, is a sensitive feature that represents group membership of each sample among $c$ different groups, $\mathcal{A}=\left\{s_1,\ldots,s_c \right\}$. Our goal is to learn a function $f: \mathcal{X} \mapsto \mathcal{Y}$ from input space $\mathcal{X} \subseteq \mathbb{R}^d$ to output  space $\mathcal{Y}$ parametrized by a vector $\boldsymbol{w} \in \mathcal{W} \subseteq \mathbb{R}^d$.  Note that, sensitive feature $a$ might or might not be a part of the input feature $\bm{x}$. The performance of $\boldsymbol{w}$ is assessed using a loss function $\ell: \mathcal{W} \times \mathcal{X} \times \mathcal{Y} \mapsto \mathbb{R}_{+}$. In empirical risk minimization, we minimize the average  loss, namely, 
\begin{equation}\label{eq:total-loss}
    \arg\min_{\boldsymbol{w} \in \mathcal{W}}\mathcal{L}(\bm{w};\mathcal{D}) := \frac{1}{n}\sum_{(\bm{x}_i,y_i)\in\mathcal{D}} \ell\left(\bm{w}; (\bm{x}_i,y_i)\right).
\end{equation}
Solely minimizing the empirical risk would result in an unfair solution with respect to different sensitive groups. To learn a fair classifier, the main idea is to  define a notion of fairness and  impose the corresponding constraints during learning process in addition to minimizing the empirical risk.  Recently different notions of algorithmic fairness have been proposed in the literature including demographic parity~\citep{dwork2012fairness}, equalized odds and equal opportunity~\citep{hardt2016equality}, individual fairness~\citep{kearns2019average}, and disparate mistreatment~\citep{zafar2017fairness}. These constraints try to reduce the effects of the sensitive feature $a$ on the output of the classifier $\hat{y}$. For instance, equality of opportunity introduced by~\cite{hardt2016equality} desires to ensure that the true positive rate of each sensitive group, $\mathsf{TP}_k = \mathbb{P}\left[\hat{y}=1 |  a = s_k, y=1\right]$, is the same for all $k \in \{1,\ldots,c\}$. A stronger notion, equalized odds~\citep{hardt2016equality}, requires that classifier's output $\hat{y}$ and sensitive feature $a$ to be independent conditional on the true label $y$. This reflects on not only having the same true positive rate among different groups but also having an equal false positive rate for each group, $\mathsf{FP}_k = \mathbb{P}\left[\hat{y}=1 |  a = s_k, y=-1\right]$ for all $k \in \{1,\ldots,c\}$. Another notion, disparate mistreatment~\citep{zafar2017fairness}, calls for equal misclassification probability, $\mathsf{F}_k = \mathbb{P}\left[\hat{y}\neq y |  a = s_k\right]$ for all $k \in \{1,\ldots,c\}$.

\subsection{Pareto efficient fairness}
As it has alluded to before, the ultimate goal is to learn a classifier that satisfies fairness constraint (possibly) without  compromising the accuracy. Since no single solution would generally minimize both objectives  simultaneously, we start by defining the  notion of Pareto efficient fairness, that mathematically forms the optimal trade-off. Then we reduce some well-known definitions of fairness to the notion of Pareto efficient fairness. For some preliminary notions in multiobjective optimization and Pareto efficiency refer to Appendix~\ref{app:MOO_prem}.

\begin{wrapfigure}{r}{0.4\textwidth}
\vspace{-0.7cm}
  \begin{center}
    \includegraphics[width=0.36\textwidth]{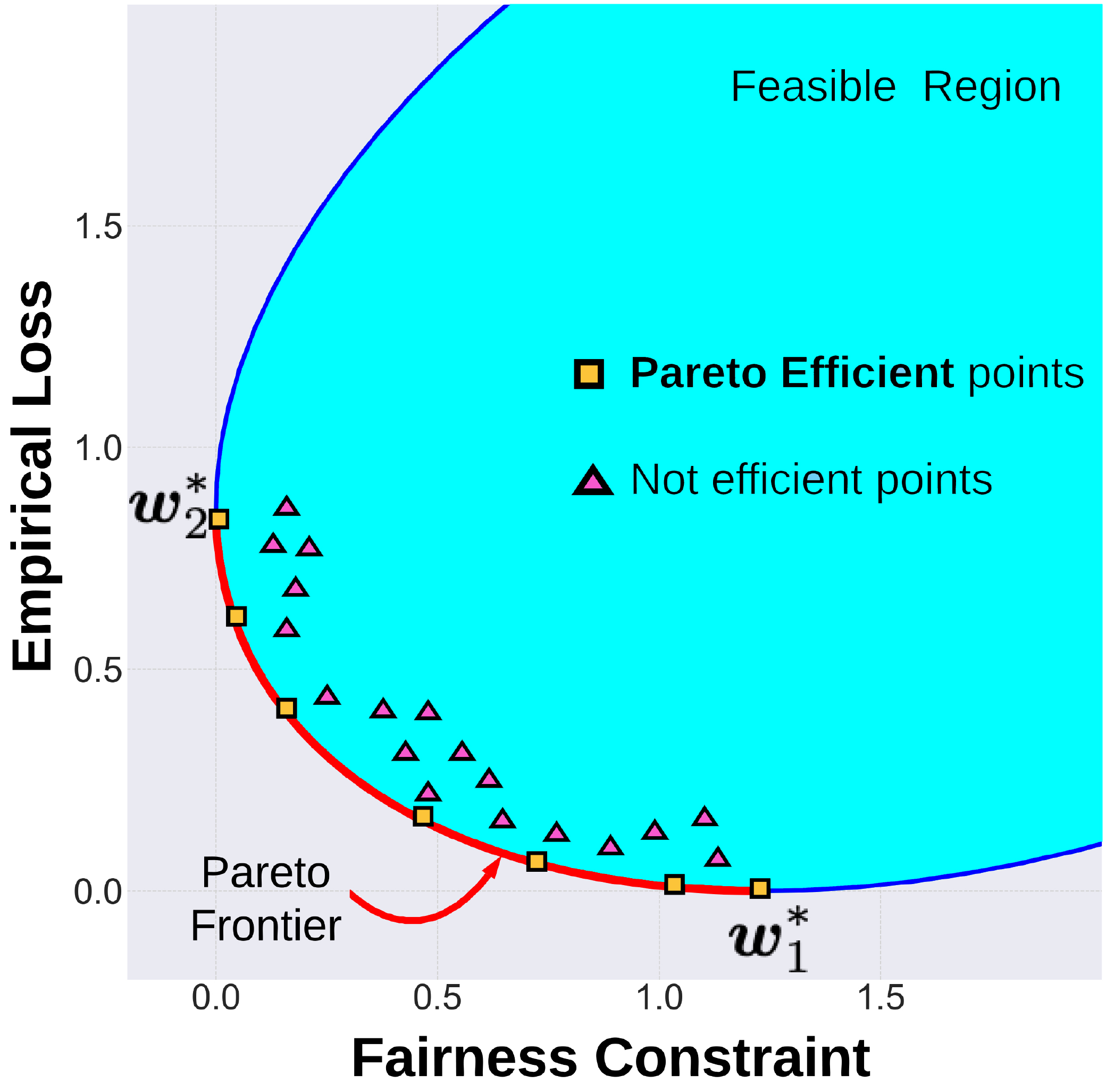}
  \end{center}
  \caption{Hypothetical trade-off between accuracy and any fairness constraint.  $\bm{w}_1^*$ is the optimum solution that minimizes the empirical risk, while $\bm{w}_2^*$ is the optimum solution to satisfy the fairness constraint. Each solution on the Pareto frontier cannot be dominated in terms of both empirical risk and the fairness constraint by any other feasible solution. }\label{fig:pf}
  \vspace{-0.7cm}
\end{wrapfigure}
As depicted in Figure~\ref{fig:pf}, in a trade-off between empirical loss of a learning problem as one objective and any fairness constraint as the second objective, we are seeking the Pareto efficient solutions that belong to the Pareto frontier of the problem. These points on the red line cannot be dominated by other points in the feasible set. The scalarization of this multi-objective optimization such as constrained optimization can potentially end up in any  non-efficient points in the feasible region; thereby lacking any guarantee on the optimally of trade-offs. Only considering one objective could end up in their respective optimal solutions $\bm{w}_1^*$ and $\bm{w}_2^*$. Thus, not only it is important to find a solution from the Pareto frontier of the problem, but also we should be able to choose the desired level of accuracy in the cost of fairness violation. In Section~\ref{sec:pf}, we incorporate our algorithm to find the desired levels of trade-off between accuracy and fairness constraints on Pareto frontier. 

We now turn to defining the notion of Pareto efficient fairness in classification problems.
\begin{definition}[Pareto efficient fairness]
Consider any fairness objectives $h_i(\bm{w}),\;i\in[m]$, that ought to be minimized in addition to the main learning objective, $\mathcal{L}(\bm{w})$. The solution $\bm{w}^*$ is called Pareto efficient fair if no other feasible solution $\bm{w} \in \mathcal{W}$ can dominate $\bm{w}^*$. That is, for the objective vector $\bm{\mathrm{h}}(\bm{w}) = \left[ \mathcal{L}(\bm{w}), \mathrm{h}_1(\bm{w}),\ldots,\mathrm{h}_m(\bm{w}) \right]$, we have 
$\nexists  \bm{w}\in\mathcal{W}$ such that $\bm{\mathrm{h}}(\bm{w}) \prec_{m} \bm{\mathrm{h}}(\bm{w}^*)$.
\label{def:pef}
\end{definition}

We call this point, with optimal compromises between the main loss and fairness objectives, Pareto efficient fair (PEF) solution. While achieving a PEF solution seems to be a conspicuous choice for this problem, and has been asserted in other studies~\citep{zafar2015fairness,kearns2019ethical,balashankar2019fair}, the crucial challenge of how to achieve such a point remains unsolved. In Section~\ref{sec:eff-opt}, we propose an iterative  efficient  algorithm that guarantees converges to a PEF solution of the described problem.

It is worthy to note that, by learning a fairness-enhanced model via PEF formulation, we are no longer confined to binary sensitive groups nor binary classification tasks, where most of the current algorithms are designed for. As a result, we could apply this notion to a broader range of learning tasks, to satisfy fairness objectives. We note that, based on the characteristics of a PEF solution from Definition~\ref{def:pef}, such a point is not unique. The set of these solutions is called the Pareto frontier of the learning problem. However, the points in the Pareto frontier cannot dominate each other, and hence, they all can be considered as a solution to the learning problem. With indicating the degrees of preference of fairness measures over the predictive performance of the model, we can find points from specific parts of this set, as it is discussed in Section~\ref{sec:pf}.

With the definition of PEF at hand, the task of learning  a fair classifier reduces to solving the following vector-valued optimization problem:
\begin{equation}\label{eq:mul-fair}
    \bm{w}^\star = \underset{\bm{w}\in \mathcal{W}}{\arg\min}\, \boldsymbol{\mathrm{h}}(\bm{w})\;,
\end{equation}
where $\boldsymbol{\mathrm{h}}:\mathcal{W} \to \mathbb{R}^m$ constitutes the empirical loss and fairness constraints.  Unlike existing methods that try to solve a relaxed version of the above optimization problem, either by relaxing to a single objective constrained optimization~\citep{donini2018empirical} or reducing to a saddle point problem~\citep{agarwal2018reductions}, we aim at directly solving the vector-valued problem.

\subsection{Reduction of known notions of fairness to PEF}
\label{sec:peeo}
To further elaborate on PEF, we now show  how different predefined notions of fairness such as equality of opportunity, equalized odds, and disparate mistreatment~\citep{zafar2017fairness} can be reduced to an instance of (\ref{eq:mul-fair}). Note that, we are not defining new notions of fairness, rather we are showing how we can reduce any notion of fairness to PEF. Here we only show the reduction of equality of opportunity notion to the PEF. More examples for reductions of fairness notions to PEF are deferred to Appendix~\ref{app:peeo}.

\paragraph{Pareto efficient equality of opportunity}
Using Definition~\ref{def:pef}, we introduce a variant of equality of opportunity dubbed as \textit{Pareto efficient equality of opportunity}. For simplicity and adjusting to current state-of-the-art fairness definitions, we consider binary learning case, that is $\mathcal{Y}=\{ -1, +1\}$; however, as we denote later, the proposed algorithm can be simply generalized to multi-class learning. To satisfy equalized opportunity criteria, which is having the same true positive rate among different groups, we translate it as having equal loss on subset of samples of each group with positive labels, {\it i.e.},  $\mathcal{S}_k^+ = \left\{ (\bm{x}_i,a_i,y_i) \in \mathcal{D} \, | \, a_i = s_k, y_i = 1,  1 \leq k \leq c\right\}$. This loss can be defined as:
\begin{equation}\label{eq:pos-loss}
    \mathcal{L}_k^+(\bm{w}) = \frac{1}{|\mathcal{S}_k^+|}\sum_{(\bm{x}_i,y_i)\in\mathcal{S}_k^+} \ell\left(\bm{w}; (\bm{x}_i,y_i)\right),\;\; k \in \{1,\ldots,c\}.
\end{equation}
To achieve the equality of opportunity, we can use this empirical loss instead of the probability of a positive outcome for each group using the classifier, as suggested by other studies~\citep{donini2018empirical,agarwal2018reductions}. In this scenario, to have an equal probability, we try to have equal empirical loss for each group for positive outcome $\mathcal{L}^+_k\left(\bm{w}\right)$. Hence, the fairness objectives  $\mathcal{H}_{i,j}^+$, in this case, reduces to the minimization of pairwise empirical losses between each pair of sensitive groups:
\begin{equation}\label{eq:eo_obj}
    \mathcal{H}_{i,j}^+(\bm{w}) = \phi\left(\mathcal{L}_i^+(\bm{w}) - \mathcal{L}_j^+(\bm{w}) \right), \;\; 1\leq i,j\leq c, i\neq j,
\end{equation}
where $\phi(\cdot):\mathbb{R}\to\mathbb{R}_{+}$ is any penalization function, such as $\phi(z)=|z|$, $\phi(z)=\frac{1}{2}z^2$, or $\phi(z)=e^{-z}$, however, for convergence analysis we will stick to smooth ones, like squared or exponential penalization.
Then, the objective vector for this fairness problem is $\boldsymbol{\mathrm{h}}_{\text{EO}}(\bm{w}) = \big[\mathcal{L}(\bm{w}), \mathcal{H}_{1,2}^+(\bm{w}), \ldots, \mathcal{H}_{c-1,c}^+(\bm{w}) \big]^\top \in \mathbb{R}^m_{+}$, where $m=1 + \binom{c}{2}$. A solution has the property of Pareto efficient equality of opportunity if it belongs to the PEF solution set of the optimization in~(\ref{eq:mul-fair}) using objective vector $\boldsymbol{\mathrm{h}}_{\text{EO}}(\bm{w})$.

\subsection{Learning  PEF Classifiers}\label{sec:eff-opt}
Having multiple, rather contradictory objectives it seems hard to find an optimal solution for (\ref{eq:mul-fair}) that has the best compromises between all the objectives. As mentioned earlier, most previous works reduce the problem into an instance of a constrained optimization problem. In contrast, we take an alternative approach and introduce a bilevel optimization, by which we can find the solution to the vector optimization problem of (\ref{eq:mul-fair}), with a guarantee of convergence to a PEF point. This  approach leads to a significantly efficient method and allows us to find an optimal trade-off between accuracy and fairness constraints. In Section~\ref{sec:pf}, we  extend the proposed idea to construct the Pareto frontier or learn a PEF solution with an apriori preference over objectives.

The proposed algorithm is motivated by the key drawbacks of scalarization methods. In scalarization methods, as one of the simplest methods to tackle vector-valued optimization problems, one aims at reducing the  optimization problem to a single objective one by combining the objectives into a single objective by assigning each objective function a non-negative weight, {\it e.g.}, a convex combination of the objectives $\sum{\alpha_i \mathrm{h}_i(\bm{w})}, \alpha_i \geq 0, \sum{\alpha_i} = 1$. By optimizing over various values of the parameters used to combine the objectives,  one is guaranteed to obtain a solution from the entire Pareto front. While being conceptually appealing, it is difficult to decide the weights apriori or it requires investigating exponentially many parameters~\cite{fukuda2014survey}. It is also  observed evenly distributed set of weights fails to produce an even distribution of Pareto minimizers in the front~\cite{das1998normal}. Finally, it might be impossible to find the entire Pareto front if some of the objectives are nonconvex. Recently, \citet{cortes2020agnostic} introduced ALMO algorithm to find optimal weights for a Pareto efficient point using agnostic learning and a minimax optimization for convex objectives. However, ALMO cannot trace different points from the Pareto frontier of the problem and converges to a single solution on the front.

To tackle the aforementioned issues,  we propose a bi-level programming idea to adaptively learn the combination weights that correspond to a single solution in Pareto front. The main idea stems from the fact that a solution is a first-order Pareto stationary if the convex hull of the individual gradients contains the origin. Specifically, let $\bm{w}^*$ be a PEF solution of the optimization in~(\ref{eq:mul-fair}). We know that there exists a vector $\al^* \in \mathbb{R}^m$, where $\w^* = \arg\min_{\w}\Psi\left(\bm{\mathrm{h}}\left(\bm{w}\right), \al^*\right)\triangleq \sum_{i=1}^{m}\alpha_i^*\mathrm{h}_i(\bm{w})$. Therefore, the key question is how to determine the optimal weights, $\al^* =  \left[\alpha_1^*,\ldots,\alpha_m^*\right]^\top$. To this end, we first note that based on Karush-Kuhn-Tucker (KKT) conditions~\cite{boyd2004convex}, the optimal weights should belong to  the following set:
\begin{equation}\label{eq:kkt}
  \left\{ \sum_{i=1}^{m} \alpha_{i}\bm{\mathrm{g}}_i(\bm{w}^*) = \bm{0},\;
    \alpha_{i} \geq 0, \; \forall\;1\leq i \leq m,\;
    \sum_{i=1}^{m}\alpha_{i} = 1\right\},
\end{equation}
where $\bm{\mathrm{g}}_i(\bm{w}^*) = \nabla \mathrm{h}_i(\bm{w}^*)$ is the normalized gradient vector of the $i$th objective at point $\bm{w}^*$ in the objective vector $\bm{\mathrm{h}}\left(\bm{w}\right)$.  

The above condition is a necessary condition for a point $\bm{w}^*$ to be a PEF point for the optimization problem (i.e., Pareto stationary) but  not sufficient as  some of the functions might be non-convex. 
To find an optimal weight, following the condition in~(\ref{eq:kkt}), we propose the following bilevel optimization problem:
\begin{equation}\label{eq:bilevel}
\begin{aligned}
 \bm{w}^* \in \arg\underset{\bm{w}}{\min} \; & {\;\Psi\left(\bm{\mathrm{h}}\left(\bm{w}\right), \bm{\alpha}^*\left(\bm{w}\right)\right) =  \sum_{i\in[m]}\alpha_i^*\left(\bm{w}\right)\mathrm{h}_i(\bm{w})} \\
 \mbox{s.t.} \quad & \bm{\alpha}^*(\bm{w}) \in \arg\underset{\bm{\alpha} \in \Delta_m}{\min} \Phi\left(\bm{\mathrm{h}}\left(\bm{w}\right),\bm{\alpha}\right) = \left\|\sum_{i=1}^m \alpha_i\bm{\mathrm{g}}_i\left(\bm{w}\right)\right\|_2^2,
\end{aligned}
\end{equation}
where $\Delta_m = \Big\{\boldsymbol{p}=[p_1, \ldots, p_m]^{\top} \in \mathbb{R}^m: p_i \geq 0, \sum_{i=1}^{m}{p_i} = 1 \Big\}$ is the $m$-dimensional simplex.

As can be seen, bilevel optimization consists of two nested optimization problems, namely inner and outer levels. Each level has its objective function, but the solution of the inner level is being used in the optimization of the outer level.  The following theorem shows that the solution of the outer-level optimization problem belongs to the PEF solution set of the optimization in~(\ref{eq:mul-fair}).
\begin{theorem}\label{theo:converge_pareto}
The solution returned by the optimization~(\ref{eq:bilevel}) is a Pareto efficient  solution to the  problem in~(\ref{eq:mul-fair}).
\end{theorem}
\begin{proof}
For the proof please refer to Appendix~\ref{app:theo:converge_pareto}.
\end{proof}
To solve the optimization in~(\ref{eq:bilevel}) using gradient descent algorithm, for every gradient step we take in the outer level, we have to find the optimal values for $\bm{\alpha}$ based on the optimization of the inner level. The following lemma shows that the gradient of the outer-level objective computed at the optimal solution of the inner-level problem is a descent direction  to all objectives in $\bm{\mathrm{h}}\left(\bm{w}\right)$.
\begin{lemma}[Pareto Descent]\label{theo:paret}
Using the solution of the inner level, $\bm{\alpha^*}$, the gradient of the outer level is either zero or a descent direction for all the objectives of $\bm{\mathrm{h}}\left(\bm{w}\right)$. That is, for $\nabla \Psi \left(\bm{\mathrm{h}}\left(\bm{w}\right),\bm{\alpha}^*\right)$, we have
\begin{equation}\label{eq:paretcond}
    - \left\langle\nabla \Psi \left(\bm{\mathrm{h}}\left(\bm{w}\right),\bm{\alpha}^*\right), \bm{\mathrm{g}}_i(\bm{w}) \right\rangle\leq 0, \;\; \forall i\in\{1,\ldots,m\}\;.
\end{equation}
\end{lemma}
\begin{proof}
The proof using contradiction is provided in Appendix~\ref{app:lemma}.
\end{proof}
\begin{remark}
Lemma~\ref{theo:paret} implies that in every step the gradient of the outer level is either zero, which means we reach a Pareto stationary point, or a descent direction to all objectives, which can be used to decrease all the objective by moving in the negative direction descent direction.
\end{remark}
Equipped with this descent direction, we can guarantee that all the objectives at every iteration are non-increasing, until we reach a point that this descent direction is $\bm{0}$, which means that we cannot improve all objectives, without hurting others. This indicates that we are in a Pareto stationary point of the problem. Note that from the first-order optimality condition in (\ref{eq:kkt}), there exists a pair of $(\al^*,\w^*)$  such that $\al^* \in \Delta_m$ and $\sum_{i=1}^{m} \al_{i}^*\bm{\mathrm{g}}_i(\w^*) = \bm{0}$.

To solve the optimization in~(\ref{eq:bilevel}), we propose an approximate procedure in Algorithm~\ref{alg:fairpareto}. having $m < d$ in our optimization, the inner level would be a strongly convex function, and hence, we can converge to its global minimum quickly. Using gradient descent, we can converge to an $\epsilon$-accurate solution in $\order\left(\log\frac{1}{\epsilon}\right)$ steps~\cite{bubeck2015convex}. This error will then be propagated to the outer level using $\hat{\bm{\alpha}}$. We will bound this error in Section~\ref{sec:converge} alongside the convergence analysis of the algorithm for convex and non-convex objectives under customary assumptions.

\begin{algorithm2e}[t]
\DontPrintSemicolon
\caption{Bilevel Pareto Descent Optimization (\texttt{PDO})}
\label{alg:fairpareto}
\SetNoFillComment
 \SetKwFunction{PF}{PDO}
 \SetKwProg{Fn}{function}{:}{\KwRet $\bm{w}^{(T)}$}
\textbf{input} $\hat{\bm{\alpha}}^{(0)} \in \mathbb{R}^m$, $\bm{w}^{(0)} \in \mathbb{R}^d$, $\rho$, $\eta$, $K$\\
\Fn{\PF{$\bm{\mathrm{h}}(\bm{w}), \eta, \rho$}}{
\For{$t=0,1,\ldots,T-1$}{
 \For{$k=0,1,\ldots, K-1$}{
  $\bm{\alpha}^{(k+1)} = \bm{\alpha}^{(k)} - \rho\cdot \nabla_{\bm{\alpha}}\Phi\left(\bm{\mathrm{h}}\left(\bm{w}^{(t)}\right),\al \right)\Big\vert_{\al = \al^{(k)}}$ \\ 
   $\bm{\alpha}^{(k+1)} = {\Pi}_{\Delta_m}\left(\bm{\alpha}^{(k+1)}\right)$ \\
}
  \textbf{set} $\hat{\bm{\alpha}}^{(t)} = \bm{\alpha}^{(K)}$ \\
  $\bm{w}^{(t+1)} = \bm{w}^{(t)} - \eta \cdot \nabla \Psi\left(\bm{\mathrm{h}}\left(\bm{w}\right),\hat{\al}^{(t)}\right)\Big\vert_{\w=\w^{(t)}}$
}
}
\end{algorithm2e}

\section{Tracing the Pareto Frontier}\label{sec:pf}
Using Algorithm~\ref{alg:fairpareto} (\texttt{PDO}), we can only guarantee converge to  a single Pareto stationary point  from the entire  set of the fairness-enhanced solution in Pareto frontier. However, from a practical point of view, it is more desirable to construct different solutions from the entire spectrum of  the Pareto frontier to be able to pick a solution with the desired trade-off. For instance, in some cases, the goal might be to keep  the model as accurate as possible while imposing some degree of fairness constraints; in some others, we want to strictly satisfy the fairness constraints, even if it hurts the accuracy of the model; in yet some others, we want a half-point compromises in between of these two extreme points. Thus, it is important to find a set of Pareto efficient points, which is called the Pareto frontier. Having access to the Pareto frontier would help a decision-maker choose from a variety of optimal solutions with different compromises between objectives.

Despite its importance, extracting points on the Pareto frontier is the one of main challenges of multi-objective optimization. 
In Appendix~\ref{app:pf_prem}, prior algorithmic solutions for finding points on the Pareto frontier is discussed.
Moreover, in Appendix~\ref{app:pf_geo}, we explore the geometry of the Pareto frontier and the relationship between optimal weights of objectives on the PEF solution ($\bm{\alpha}\left(\bm{w}^*\right)$) and the Pareto frontier surface. We show that if the Pareto frontier surface is smooth and convex a simple reweighting of the objectives with a good initialization will get us to the desired point on the Pareto frontier. However, these conditions are not met in most cases, even if all the objectives are smooth and convex. Hence, we need a concrete approach that can be applied to any case. Thus, we propose a novel algorithm that can extract points from the Pareto frontier of different objectives with only access to the first-order information. This approach is generic and can be applied to any multiobjective optimization problem.

\subsection{Pareto frontier extraction}\label{sec:pf_extract}
The goal of finding different points from the Pareto frontier and constructing a Pareto frontier set for a vector minimization problem has been investigated in a vast number of studies so far~\citep{hillermeier2001nonlinear}. However, achieving this goal is still a challenging problem that requires an algorithmic approach with tractable computational complexity for different scenarios. In general, we can categorize the main approaches that aim to find points from the Pareto frontier or trace the points on this set, into four groups. There are a number of heuristic approaches as well, however, for this paper we focus on the main analytic ones. These four main categories are:
(1)\textit{ Normal Boundary Intersection (NBI)}, (2) \textit{Geometrical Exploration}, (3) \textit{Weight Perturbation}, and (4) \textit{Preference-based solutions}. The detailed description of each of these approaches can be find in Appendix~\ref{app:pf_prem}. Our approach to finding points from the Pareto frontier belongs to the preference-based methods similar to~\cite{mahapatramulti}, but instead of targeting a single point, we use this approach to trace the points on the Pareto frontier before reaching that specific point. Traversing the points on the Pareto frontier in our approach is similar to the geometrical exploration approaches. Despite this similarity, our approach only uses \textit{first-order} information, which is computationally efficient than those methods, where they use second-order information.

\subsubsection{Preference-based Pareto descent optimization}
We are pursuing to find points from the Pareto frontier using a user predefined preference over different objectives. These preferences can be represented as a vector $\bm{\pi} \in \mathbb{R}^m_+$. In this setting, the ultimate goal is to reach a point from the Pareto frontier, where the ratio between objectives' value is proportional to the ratio of corresponding preference values in this vector. That is for this point we aim to satisfy:
\begin{equation}\label{eq:front_cond}
    \pi_1 \mathrm{h}_1\left(\bm{w}^*_{\mathsf{PB}}\right) = \pi_2 \mathrm{h}_2\left(\bm{w}^*_{\mathsf{PB}}\right) = \ldots = \pi_m \mathrm{h}_m\left(\bm{w}^*_{\mathsf{PB}}\right),
\end{equation}
where $\bm{w}^*_{\mathsf{PB}}$ is a preference-based solution from the Pareto frontier that satisfies this condition\footnote{The assumption is that such a point on the Pareto frontier that satisfies this condition exists. If this point does not exist, in practice, we reach its closest point on the Pareto frontier.}. This condition implies that if we increase the weight for one objective, we require to find a Pareto optimal solution with lower objective value for that specific objective. Besides, the condition in~(\ref{eq:front_cond}) suggests that this point is the intersection of the Pareto frontier and a line in the objective space, determined by:
\begin{equation}\label{eq:front_line}
    \bm{\mathrm{h}} = \left[\mathrm{h}_1,\ldots,\mathrm{h}_m\right] = \bm{p} \circ \left(c\mathbf{1}_m\right),
\end{equation}
where $\bm{p} = \left[1/\pi_1,\ldots,1/\pi_m\right]$, $\mathbf{1}_m$ is an all-ones vector with size $m$, $c \in \mathbb{R}$ is an independent variable, and $\bm{a}\circ \bm{b}$ shows the Hadamard or element-wise multiplication of two vectors $\bm{a}$ and $\bm{b}$. We call this line the ``preference line''. Using the following simple example for a bi-objective problem ($m=2$), this line and its intersection point for an arbitrary $\bm{\pi}$ and $c$ is depicted in Figure~\ref{fig:pref_line}. This line intercepts the origin and $\bm{p}$.

\begin{exmp}~\label{ex:pf}
For the sake of exposition for generating Pareto frontier, we use a simple and classic example~\citep{peitz2018gradient}, where we have a bi-objective optimization with two objectives as:
\begin{equation}\label{eq:pf_ex}
    \mathrm{h}_1\left(\bm{w}\right) = 1 - e^{-\frac{\llVert\bm{w}-\bm{\nu}\rrVert^2}{s^2}},\quad \mathrm{h}_2\left(\bm{w}\right) = 1 - e^{-\frac{\llVert\bm{w}+\bm{\nu}\rrVert^2}{s^2}},
\end{equation}
where $\bm{\nu}\in\mathbb{R}^d$ and $s\in\mathbb{R}$ are the mean and variance parameters. For instance, the objective values for $\bm{w}\in\mathbb{R}^2$, $\bm{\nu}=\mathbf{1}_2$, and $s=1.5$ are depicted in Figure~\ref{fig:exp:pf}.
\end{exmp}

\begin{figure}
    \centering
    \begin{minipage}{0.31\textwidth}
    \centering
    \includegraphics[width=\textwidth]{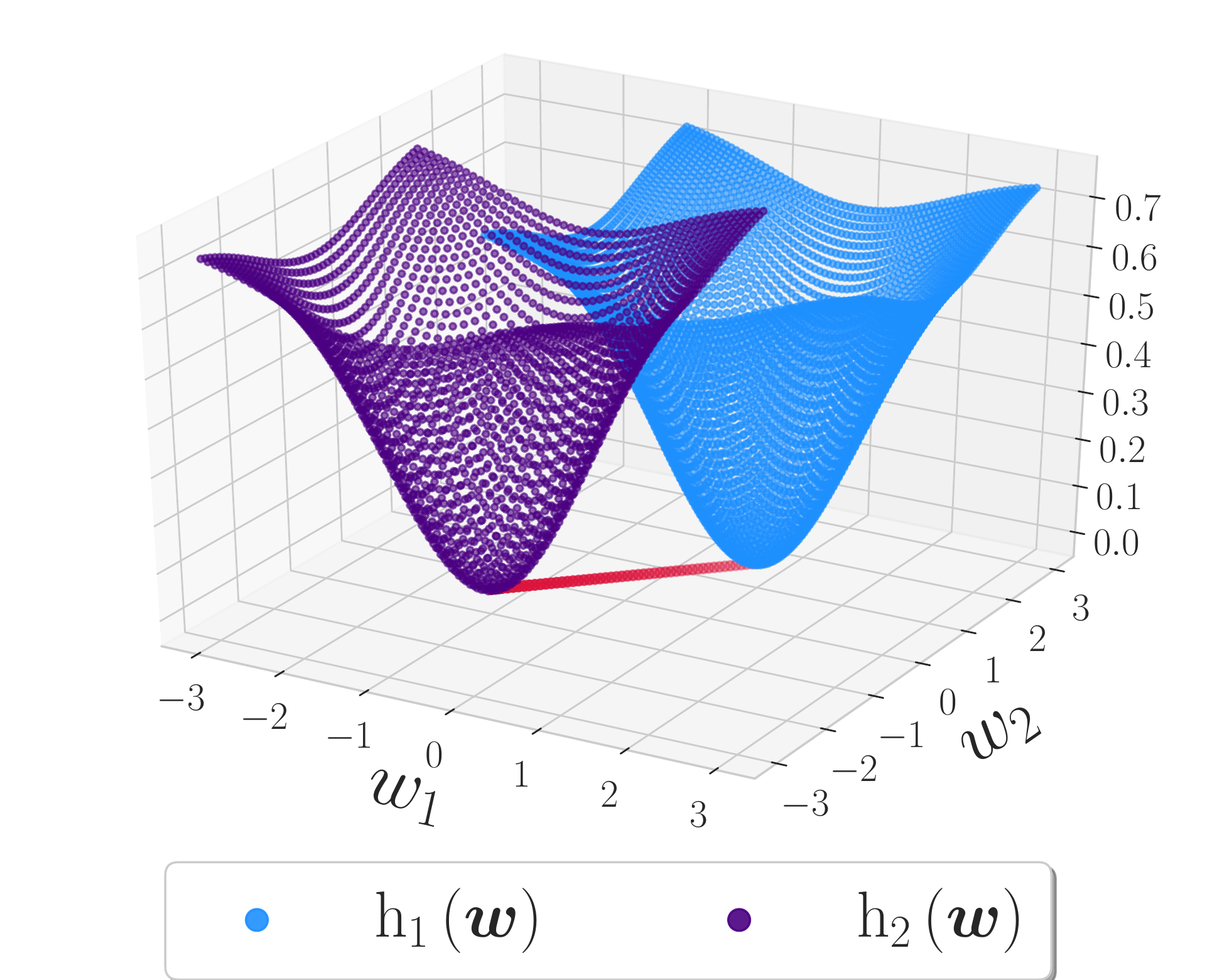}\vspace{0.15cm}
    \caption{Example objectives, defined in~(\ref{eq:pf_ex}), in the parameter space with $\bm{w}\in\mathbb{R}^2$, $\bm{\nu}=\mathbf{1}_2$, and $s=1.5$. The red line is the projection of the Pareto frontier from the objective space on the parameter space.}
    \label{fig:exp:pf}
    \end{minipage}
    \hspace{0.2cm}
    \begin{minipage}{0.64\textwidth}
    \centering
    \includegraphics[width=0.86\textwidth]{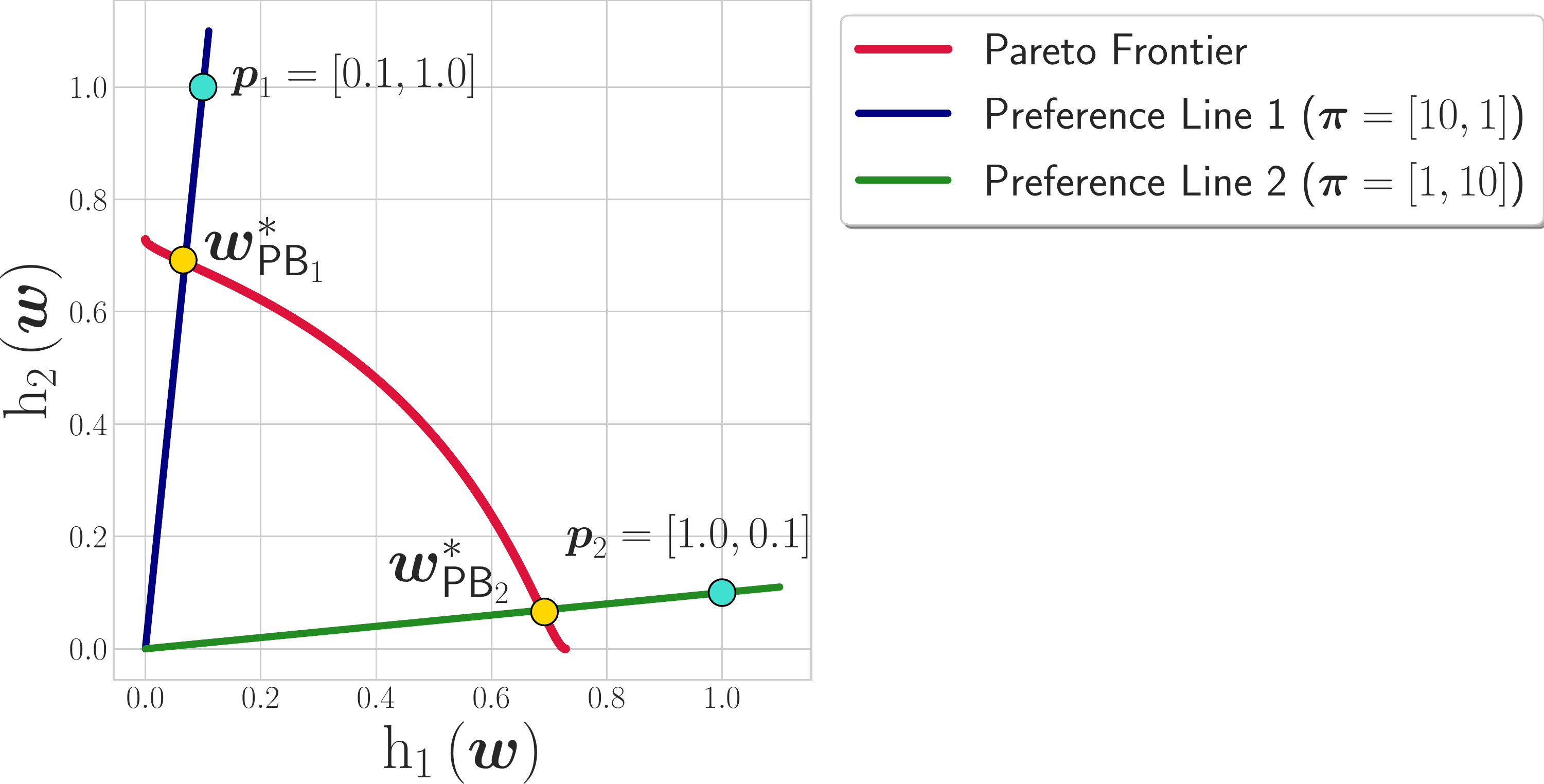}
    \caption{An arbitrary Pareto frontier with two different preference lines parameterized by their $\bm{\pi}$ or $\bm{p}$ vectors. The Pareto frontier is from the bi-objective problem depicted in Figure~\ref{fig:exp:pf}. The desired solution for each line is on the intersection of that preference line and the Pareto frontier, which are $\bm{w}_{\mathsf{PB}_1}^*$ and $\bm{w}_{\mathsf{PB}_2}^*$ in the figure. It can be inferred that when we increase one objective's preference value, the desired solution is getting closer to that objective's minimizer.}
    \label{fig:pref_line}
    \end{minipage}
\end{figure}

To converge to the desired point $\bm{w}_{\mathsf{PB}}^*$, we should define a new objective that measures how the distribution of objective values on each solution point are deviating from the condition in~(\ref{eq:front_cond}). We can project the vector $\bm{\pi} \circ \bm{\mathrm{h}}$ into a $\Delta_m$ simplex, for each arbitrary solution point $\bm{w}$, using a Softmax function:
\begin{equation}\label{eq:sigmoid}
    \sigma_i\left(\bm{w},\bm{\pi}\right) = \frac{e^{\pi_i \mathrm{h}_i\left(\w\right)}}{\sum_{j\in[m]}e^{\pi_j \mathrm{h}_j\left(\w\right)}},\;\; \forall{i} \in [m],
\end{equation}
then the $\sigma_i\left(\bm{w}\right)$ values can be considered as a probability distribution. Now, the condition in~(\ref{eq:front_cond}) reduces to:
\begin{equation}\label{eq:uniform}
    \sigma_i\left(\bm{w}^*_{\mathsf{PB}},\bm{\pi}\right) = \frac{1}{m},\;\; \forall{i} \in [m],
\end{equation}
which is the uniform distribution. Thus, the best choice for the objective function to measure the discrepancy between probability values in~(\ref{eq:sigmoid}) and~(\ref{eq:uniform}) seems to be KL-divergence between these two distributions:
\begin{align}\label{eq:kl-loss}
    \mathrm{h}_{\mathsf{KL}}\left(\bm{w},\bm{\pi}\right) & = \mathsf{KL}\left(\bm{\sigma}\left(\bm{w},\bm{\pi}\right) \| \bm{\sigma}\left(\bm{w}^*_{\mathsf{PB}},\bm{\pi}\right)\right) \nonumber \\ 
    & = \sum_{i\in[m]} \sigma_i\left(\bm{w},\bm{\pi}\right) \log\left(\frac{\sigma_i\left(\bm{w},\bm{\pi}\right)}{ \sigma_i\left(\bm{w}^*_{\mathsf{PB}},\bm{\pi}\right)}\right) \nonumber\\
    & = \sum_{i\in[m]} \sigma_i\left(\bm{w},\bm{\pi}\right) \log\left(m \sigma_i\left(\bm{w},\bm{\pi}\right)\right)\;.
\end{align}
We want to minimize this objective, which indicates that we want to minimize the entropy of $\bm{\sigma}\left(\bm{w},\bm{\pi}\right)$ and maximize the cross entropy between $\bm{\sigma}\left(\bm{w},\bm{\pi}\right)$ and $\bm{\sigma}\left(\bm{w}^*_{\mathsf{PB}},\bm{\pi}\right)$. This objective has its minimum value of zero on all the points on the line defined in~(\ref{eq:front_line}).

The way to minimize the objective $\mathrm{h}_{\mathsf{KL}}$ in~(\ref{eq:kl-loss}), in addition to the main objective vector $\bm{\mathrm{h}}$, is to add it as another objective to the objective vector and have an $m+1$-dimensional vector to minimize. Using this approach we make sure that at each step we find a direction that is descent for all the objectives in $\bm{\mathrm{h}}$ as well as $\mathrm{h}_\mathsf{KL}$. Hence, in general, we consider the following objective vector as the main vector to minimize for our problem:
\begin{equation}
    \bm{\mathrm{h}}_{\mathsf{PB}} \left(\bm{w},\bm{\pi}\right) = \left[\mathrm{h}_1\left(\bm{w}\right), \ldots, \mathrm{h}_m\left(\bm{w}\right),  \mathrm{h}_{\mathsf{KL}}\left(\bm{w},\bm{\pi}\right)\right] \in \mathbb{R}^{m+1},
\end{equation}
where $\mathsf{PB}$ in $ \bm{\mathrm{h}}_{\mathsf{PB}}$ stands for a preference-based objective with the preference vector indicated by $\bm{\pi}$. To minimize this vector, we need to compute the gradient of the $\mathrm{h}_{\mathsf{KL}}$ at each iteration, similar to other objectives indicated in~(\ref{eq:bilevel}). The following proposition computes this gradient for this problem.
\begin{proposition}\label{prop:grad_kl}
The gradient vector of the objective  $\mathrm{h}_{\mathsf{KL}}\left(\bm{w},\bm{\pi}\right)$ with respect to any arbitrary solution point $\bm{w} \in \mathbb{R}^d$, is denoted by $\bm{\mathrm{g}}_{\mathsf{KL}}\left(\bm{w},\bm{\pi}\right) \in \mathbb{R}^d$ has the following form:
\begin{align}\label{eq:kl_grad}
    \bm{\mathrm{g}}_{\mathsf{KL}}\left(\bm{w},\bm{\pi}\right)  &= \nabla_{\bm{w}} \mathrm{h}_{\mathsf{KL}}\left(\bm{w},\bm{\pi}\right) = \sum_{i\in[m]} \lambda_i \bm{\mathrm{g}}_i \\ \nonumber
    \mbox{s.t.} \quad &\;\;\;  \lambda_i = \pi_i \sigma_i\left(\bm{w},\bm{\pi}\right) \left( \log\left(m\sigma_i\left(\bm{w},\bm{\pi}\right)\right) - \mathrm{h}_{\mathsf{KL}}\left(\bm{w},\bm{\pi}\right)\right),
\end{align}
where $\bm{g}_i, i\in[m]$ are the gradients of objectives in the objective vector $\bm{\mathrm{h}}$.
\end{proposition}
\begin{proof}
The proof is deferred to Appendix~\ref{app:grad_kl}.
\end{proof}

The form of the gradient direction for $\mathrm{h}_\mathsf{KL}$ indicated by Proposition~\ref{prop:grad_kl} shows that similar to the descent direction of the main objective vector $\bm{\mathrm{h}}$, this gradient is a linear combination of objective's gradients. We will use this gradient direction in our algorithm to minimize the loss for the preference-based objective.

In general, minimizing the preference-based objective vector $\bm{\mathrm{h}}_{\mathsf{PB}}$ would converge to the desired point $\bm{w}^*_{\mathsf{PB}}$, however, there are two cases depending on the position of the initial point in the problem that might not be able to converge to that desired solution. These two cases happen when (I) the initial solution is too close to the preference line defined in~(\ref{eq:front_line}); or (II) it is far away from the desired solution $\bm{w}^*_{\mathsf{PB}}$. The reason that minimizing the vector $\bm{\mathrm{h}}_{\mathsf{PB}}$ might not converge in these two cases is that we reach a local minimizer of either of its two objectives (a point from the preference line in~(\ref{eq:front_line}) for $\mathrm{h}_{\mathsf{KL}}$ or a point from the Pareto frontier of $\bm{\mathrm{h}}$) before the desired point $\bm{w}^*_{\mathsf{PB}}$. In either of these cases, the resulting descent direction would be zero and it cannot escape that point. These two situations are depicted in Figure~\ref{fig:pf_cases_un}, where the model converges to either $\bm{w}_{\text{I}}$ or $\bm{w}_{\text{II}}$ for cases I and II, respectively. Hence, we need to develop an algorithm to address such cases. It is worth mentioning that, even in these two cases the algorithm is converging to a stationary point of the preference-based objective
$\bm{\mathrm{h}}_{\mathsf{PB}}$, which is in line with the theoretical results. However, since the goal is to reach a desired point on the Pareto frontier of $\bm{\mathrm{h}}$, these stationary points are not satisfying that goal.

\begin{figure}[t]
    \centering
    \subfigure[]{
    \includegraphics[width=0.35\linewidth]{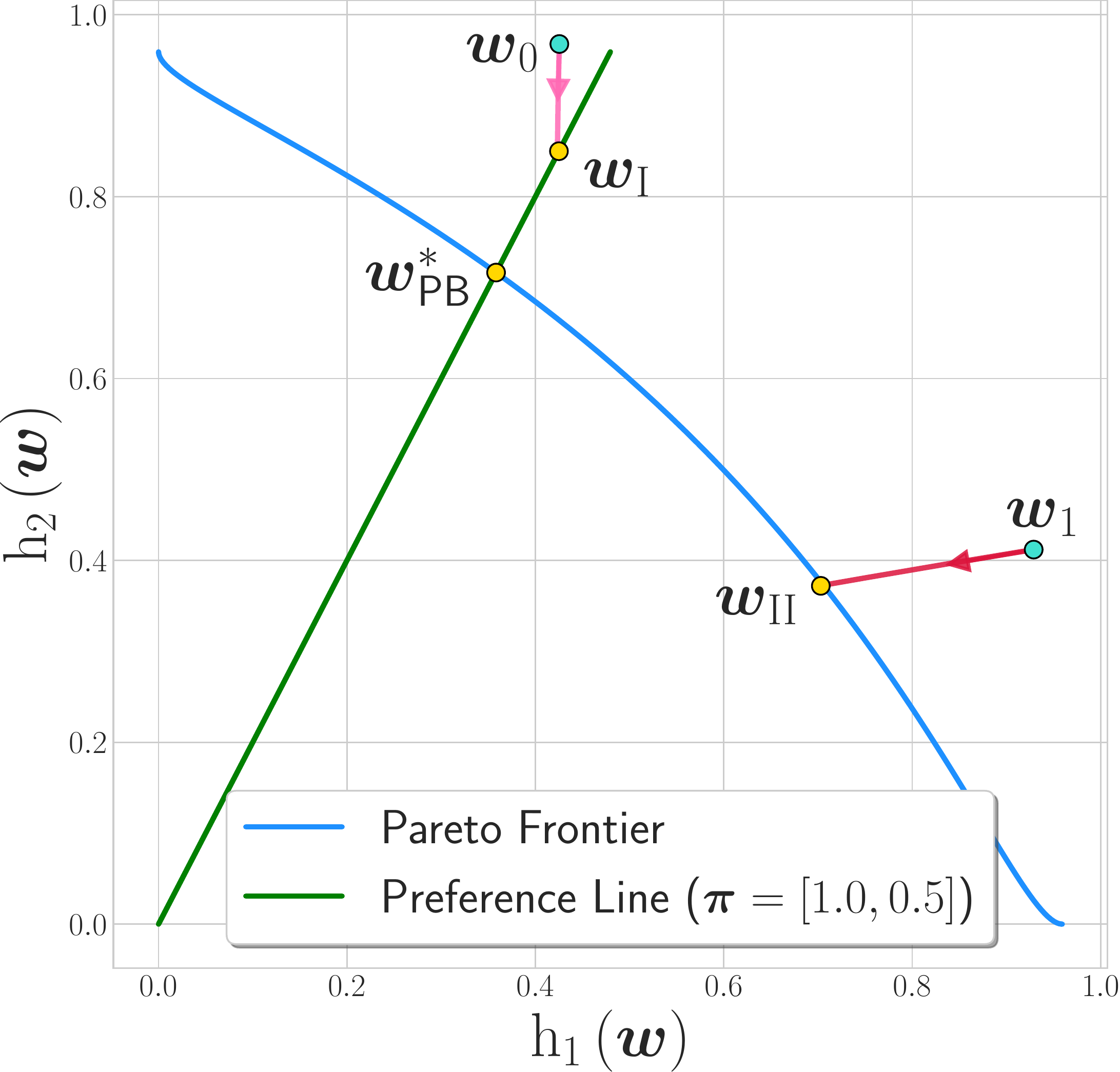}
    \label{fig:pf_cases_un}
    }
    \subfigure[]{
    \includegraphics[width=0.35\linewidth]{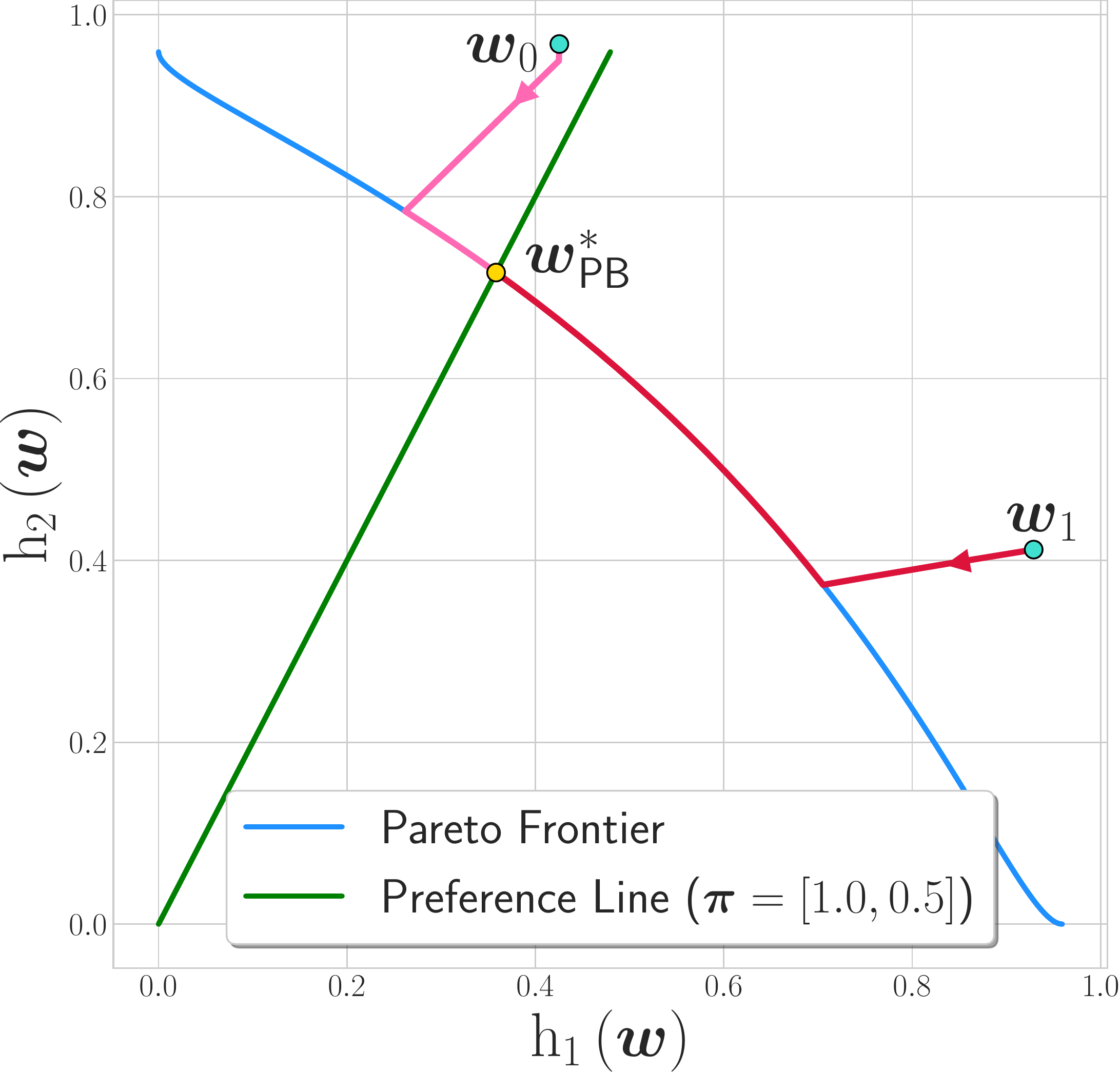}
    \label{fig:pf_cases_sol}
    }
\caption{Using objectives in Example~\ref{ex:pf} with $\bm{w}\in \mathbb{R}^{20}$ to (a) demonstrate two challenging cases, which cannot converge to the desired point $\bm{w}_{\mathsf{PB}}^*$ by only using $\bm{\mathrm{h}}_{\mathsf{PB}}$ as the objective; and (b) how \texttt{PB-PDO} by adaptively choosing objectives can achieve this goal. In (a) we converge to points $\bm{w}_{\text{I}}$ and $\bm{w}_{\text{II}}$ for cases I and II, respectively, which are stationary points of the $\bm{\mathrm{h}}_{\mathsf{PB}}$ but they are not the desired point.}\label{fig:pf_cases}
\end{figure}

\begin{algorithm2e}[t]
\DontPrintSemicolon
\caption{Preference-Based Pareto Descent Optimization (\texttt{PB-PDO})}
\label{alg:PBpareto}
\SetNoFillComment
 \SetKwFunction{PF}{PB-PDO}
 \SetKwProg{Fn}{function}{:}{\KwRet $\bm{w}^{(T)}$}
 \SetKwFunction{DD}{DescentDirection}
 \SetKwProg{Sb}{function}{:}{\KwRet $\nabla \Psi\left(\bm{\mathrm{f}}\left(\bm{w}^{(t)}\right),\hat{\bm{\alpha}}^{(t)}\right)$}
\textbf{input} $\hat{\bm{\alpha}}^{(0)} \in \mathbb{R}^m$, $\bm{w}^{(0)} \in \mathbb{R}^d$, $\rho$, $\eta$, $T$, $K$, $\bm{\pi}$, $\epsilon_1$, $\epsilon_2$\\
\Fn{\PF{$\bm{\mathrm{h}}(\bm{w}), \bm{\pi}, \eta, \rho, \epsilon_1, \epsilon_2$}}{
\For{$t=0,1,\ldots,T-1$}{
  Compute $\bm{\mathrm{g}}_{\mathsf{KL}}\left(\bm{w}^{(t)},\bm{\pi}\right)$ using (\ref{eq:kl_grad})\\ 
  \uIf(\tcp*[f]{Case (I)}){$\llVert\bm{\mathrm{g}}_\mathsf{KL}\left(\bm{w},\bm{\pi}\right)\rrVert \leq \epsilon_1$}{
  $\bm{\mathrm{d}}^{(t)}$ = \DD{$\bm{\mathrm{h}}$, $\bm{w}^{(t)}$, $\rho$} 
  }\Else{
  $\bm{\mathrm{d}}^{(t)}$ = \DD{$\bm{\mathrm{h}_\mathsf{PB}}$, $\bm{w}^{(t)}$, $\rho$}
  }
  \If(\tcp*[f]{Case (II)}){$\frac{\llVert \bm{\mathrm{d}}^{(t)}\rrVert}{\llVert \bm{\mathrm{g}}_\mathsf{KL}\left(\bm{w},\bm{\pi}\right)\rrVert} \leq \epsilon_2$}{
   $\bm{\mathrm{d}}^{(t)} = \bm{\mathrm{g}}_{\mathsf{KL}}\left(\bm{w}^{(t)},\bm{\pi}\right)$
  }
  $\bm{w}^{(t+1)} = \bm{w}^{(t)} - \eta \cdot \bm{\mathrm{d}}^{(t)}$\\
}
}
\Sb{\DD{$\bm{\mathrm{f}}$, $\bm{w}^{(t)}$, $\rho$}}{
\For{$k=0,1,\ldots, K-1$}{
  $\bm{\alpha}^{(k+1)} = \bm{\alpha}^{(k)} - \rho\cdot \nabla_{\bm{\alpha}}\Phi\left(\bm{\mathrm{f}}\left(\bm{w}^{(t)}\right),\al\right)\Big\vert_{\al = \al^{(k)}}$ \\ 
   $\bm{\alpha}^{(k+1)} = {\Pi}_{\Delta_m}\left(\bm{\alpha}^{(k+1)}\right)$ \\
}
  \textbf{set} $\hat{\bm{\alpha}}^{(t)} = \bm{\alpha}^{(K)}$ \\
}
\end{algorithm2e}
 
We design Algorithm~\ref{alg:PBpareto} that introduces Preference-Based Pareto Descent Optimization (\texttt{PB-PDO}). In this algorithm, the user determines a preference vector $\bm{\pi}$ over the set of different objectives we have, and the algorithm finds the best $\bm{w}^*_{\mathsf{PB}}$ on the Pareto frontier of $\bm{\mathrm{h}}$ that satisfies the condition in~(\ref{eq:front_cond}). To avoid the two cases described before we need to adaptively define our set of objectives for finding the descent direction at each iteration. In general, our algorithm at each iteration finds a descent direction $\bm{\mathrm{d}} = \nabla \Psi\left(\bm{\mathrm{h}}_{\mathsf{PB}}\left(\bm{w},\bm{\pi}\right),\bm{\alpha}^*\left(\w\right)\right)$, where $\bm{\alpha}^*\left(\w\right)$ is determined similar to~(\ref{eq:bilevel}), for the objective vector $\bm{\mathrm{h}}_{\mathsf{PB}}$. If either of the following conditions met, which are indicators of the aforementioned cases, we will tune our objective vector and accordingly its descent direction, to escape those points toward the desired point:
\begin{enumerate}[(I)]
    \item The first case happens when we reach a point on the preference line in~(\ref{eq:front_line}), \textit{other} than the desired point on the Pareto frontier $\bm{w}^*_{\mathsf{PB}}$. The reason is that the initial point is close to the preference line (see Figure~\ref{fig:pf_cases_un} and the trajectory from $\bm{w}_0$ toward $\bm{w}_\text{I}$). For this point we have $\bm{\mathrm{g}}_\mathsf{KL} = \bm{0}$, however, this point is not on the Pareto Frontier of the main objective vector $\bm{\mathrm{h}}$. It is worth noting that this point is a stationary point of $\bm{\mathrm{h}}_\mathsf{PB}$, however, it is not a stationary point of $\bm{\mathrm{h}}$. Thus, whenever we have:
    \begin{equation}
        \llVert \bm{\mathrm{g}}_\mathsf{KL}\left(\bm{w},\bm{\pi}\right)\rrVert \leq \epsilon_1\;,
    \end{equation}
    we use the main objective vector $\bm{\mathrm{h}}$ instead of the preference-based one  $\bm{\mathrm{h}}_\mathsf{PB}$, to find the descent direction. This would help us to reach a point from the Pareto frontier of $\bm{\mathrm{h}}$ before reaching the preference line. In this condition, $\epsilon_1$ is a small arbitrary threshold. See Figure~\ref{fig:pf_cases_sol} and the trajectory from $\bm{w}_0$ toward $\bm{w}_{\mathsf{PB}}^*$, where using this condition would not allow the model to reach a point on the preference line before touching the Pareto frontier of $\bm{\mathrm{h}}$.
    \item The second case happens when we reach a point from the Pareto frontier, \textit{before} the desired point $\bm{w}^*_{\mathsf{PB}}$ that satisfies the condition in~(\ref{eq:front_cond}). This scenario occurs when the initial solution is far away from the desired point, and hence, using descent directions would get us to a point from the Pareto frontier sooner than that desired point (see Figure~\ref{fig:pf_cases_un} and the trajectory from $\bm{w}_1$ toward $\bm{w}_\text{II}$). In this case, when we reach a point from the Pareto frontier, the $\ell_2$-norm of the descent direction would be almost zero since the point is a stationary point of the main objective vector $\bm{\mathrm{h}}$. However, since the $\mathrm{h}_{\mathsf{KL}}$ is not minimized yet, the $\ell_2$-norm of its gradient  $\bm{g}_\mathsf{KL}$ is large. Therefore, the case happens when we have:
    \begin{equation}\label{eq:cond1}
        \frac{\llVert \bm{\mathrm{d}}\rrVert}{\llVert \bm{\mathrm{g}}_\mathsf{KL}\left(\bm{w},\bm{\pi}\right)\rrVert} \leq \epsilon_2\;,
    \end{equation}
    where $\epsilon_2$ is a small arbitrary threshold. Whenever the condition in~(\ref{eq:cond1}) is met, which means that we reach such a point, instead of using the descent direction from one of the objective vectors ($\bm{\mathrm{h}}$, or $\bm{\mathrm{h}}_\mathsf{PB}$), we only use $\bm{g}_\mathsf{KL}\left(\bm{w},\bm{\pi}\right)$. This means that we are only minimizing $\mathrm{h}_\mathsf{KL}$. This case, in fact, is beneficial for tracing the Pareto frontier, which we elaborate on later. In Figure~\ref{fig:pf_cases_sol}, the trajectory from $\bm{w}_1$ toward $\bm{w}_{\mathsf{PB}}^*$ shows that using this approach we can trace points from the Pareto frontier of the main objective $\bm{\mathrm{h}}$.
\end{enumerate}
The full procedure is defined in Algorithm~\ref{alg:PBpareto}. It is worth mentioning that it is crucial to have $\mathrm{h}_\mathsf{KL}$ in our objective vector to gets us as close as possible to the preference line and speed up the convergence. In practice, if we do not include it and the initial point is far away from the preference line we might not be able to converge to the desired point. Now, the following remarks regarding the novelties of the proposed algorithm (\texttt{PB-PDO}) over previous methods and how to trace points from the Pareto frontier are in order.
\begin{remark}
The proposed algorithm (\texttt{PB-PDO}), despite some of the previous approaches, does not require the minimizers of individual objectives. Also, since the preference vector is not bounded to a simplex, the ratio between preference values of different objectives could be any arbitrary positive number. Hence, this approach can be applied to objectives with different scales, which was a limitation for some other approaches. Finally, this algorithm only uses first-order information to reach the desired point and traverse the Pareto frontier, which is a huge computational advantage over similar counterparts.
\end{remark}
\begin{remark}
Using the proposed algorithm (\texttt{PB-PDO}), we can traverse points on the Pareto frontier, which is missing from other preference-based approaches such as~\citep{lin2019pareto,mahapatramulti}. The closest approach to ours~\citep{mahapatramulti} uses ascent directions in some iterations to reach the desired point, and hence, needs to set some constraints to avoid divergence. Their approach makes sure that does not touch the Pareto frontier before reaching that desired point. However, our approach, using only descent directions, not only converges to the desired point but also able to trace points on the Pareto frontier. The comparison between these two algorithms as well as Algorithm~\ref{alg:fairpareto} (\texttt{PDO}) is depicted in Figure~\ref{fig:pf_compare}.
\end{remark}

\begin{figure}[t]
    \centering
    \subfigure[\texttt{PDO}]{
    \includegraphics[width=0.3\linewidth]{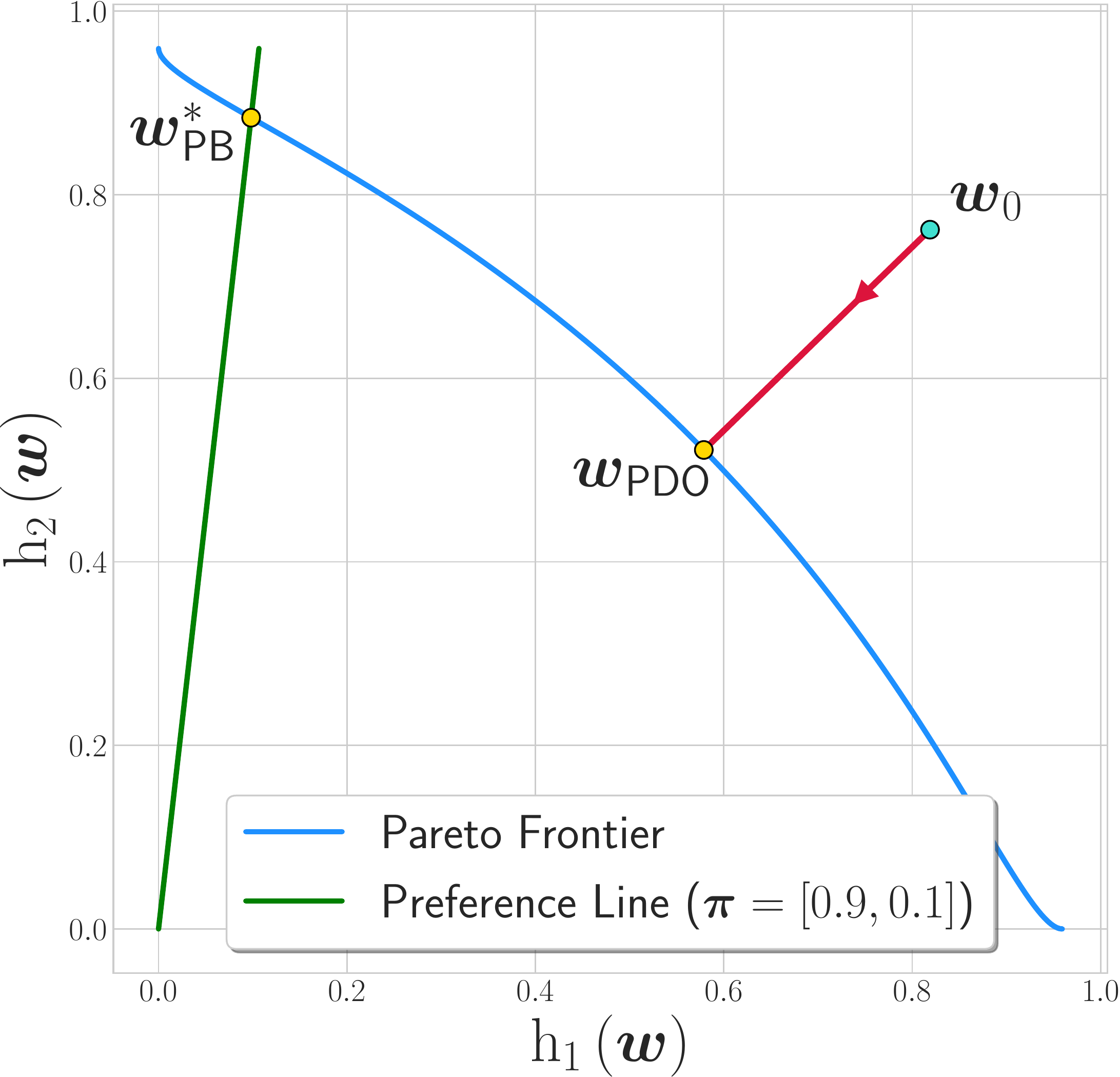}
    \label{fig:pf_compare_PDO}
    }
    \subfigure[Preference-based~\citep{mahapatramulti}]{
    \includegraphics[width=0.3\linewidth]{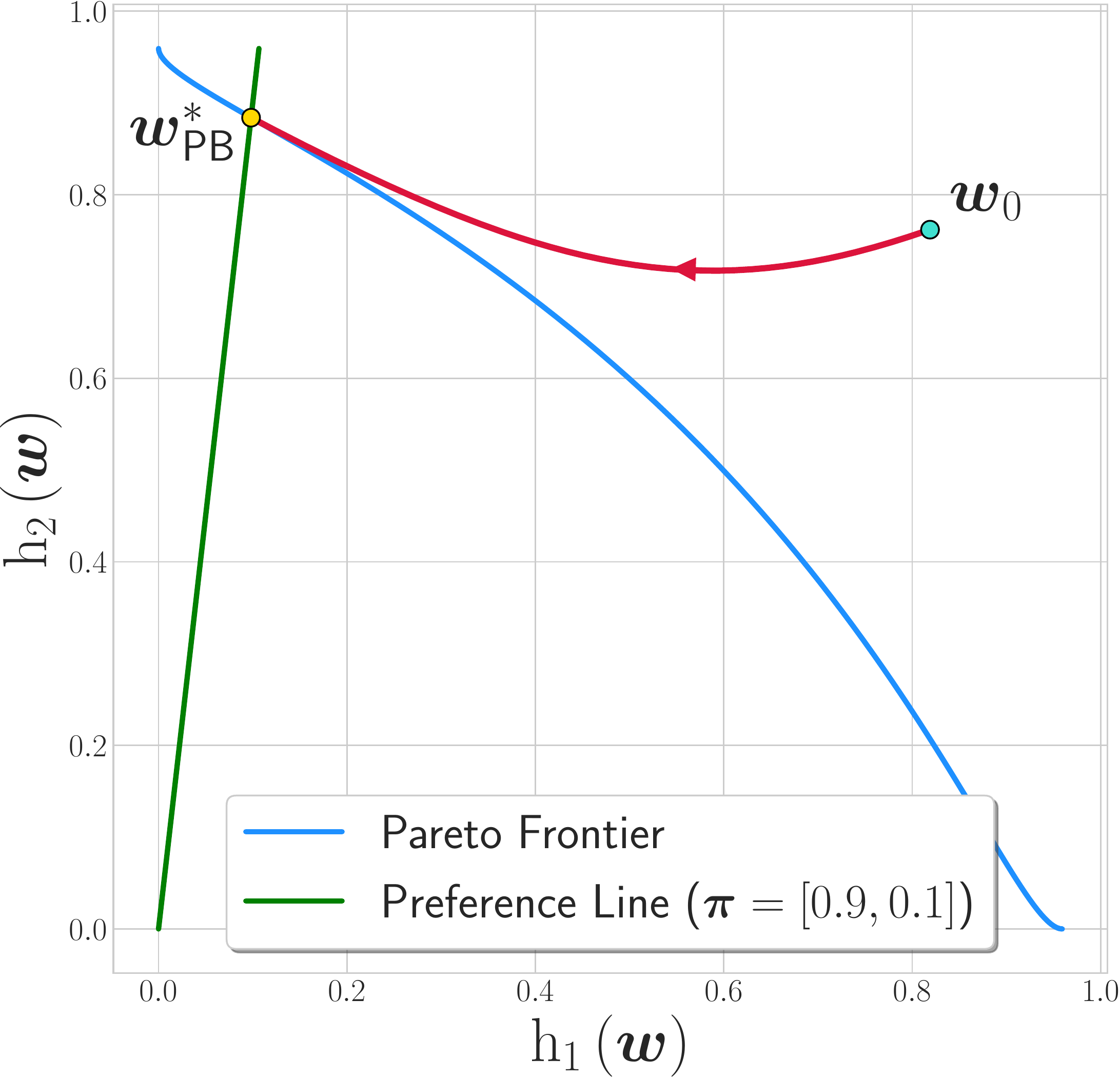}
    \label{fig:pf_compare_KL}
    }
    \subfigure[\texttt{PB-PDO}]{
    \includegraphics[width=0.3\linewidth]{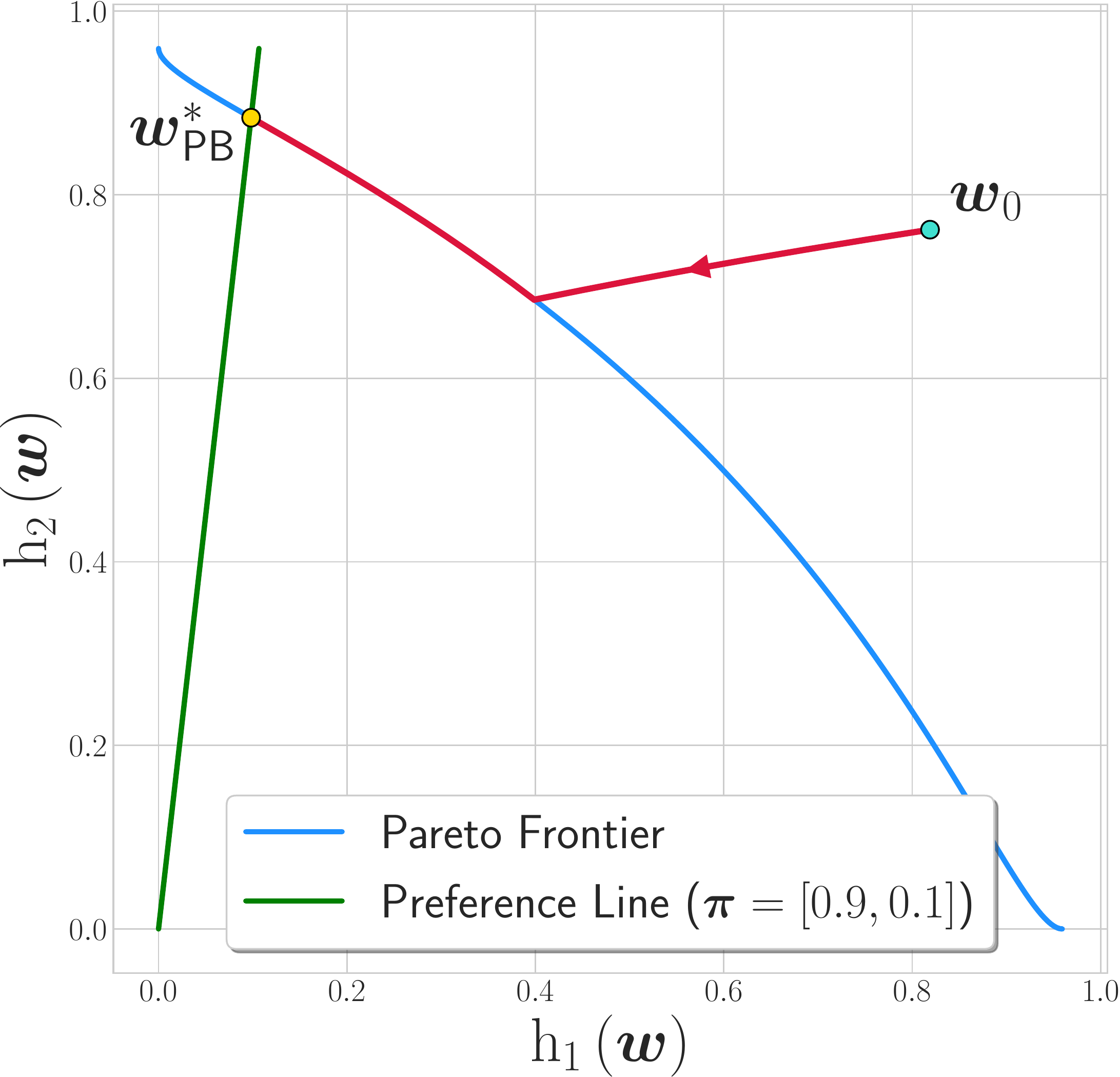}
    \label{fig:pf_compare_PBPDO}
    }
\caption{Comparing the convergence trajectory of our Algorithm~\ref{alg:fairpareto} (\texttt{PDO}), the preference-based algorithm introduced in~\cite{mahapatramulti}, and our proposed Algorithm~\ref{alg:PBpareto} (\texttt{PB-PDO}), when they all start from the same initial solution $\bm{w}_0$. \texttt{PDO} uses descent directions based on $\bm{\mathrm{h}}$ to reach a point from the Pareto frontier. The preference-based approach introduced in~\cite{mahapatramulti} uses ascent and descent directions to converge to the desired point and not other points from the Pareto frontier. On the other hand, \texttt{PB-PDO} using descent directions on $\bm{\mathrm{h}_\mathsf{PB}}$ to converge to the desired points, while tracing other points from the Pareto frontier during the convergence. }\label{fig:pf_compare}
\end{figure}

To find points from different parts of the Pareto frontier using the proposed algorithm we can set different preference vectors and run the algorithm several times. Then, by choosing non-dominating points in the trajectories of these runs, we can extract points from different parts in the Pareto frontier. For instance, for the Adult dataset (description in Section~\ref{sec:exp}), by running \texttt{PB-PDO} for only 10 times using different preference vector each time we can extract its Pareto frontier as in Figure~\ref{fig:adult_pf}. In this experiment, we use gender as the sensitive feature and equality of opportunity as the fairness loss per its definition in~(\ref{eq:eo_obj}). For more results on extracting Pareto frontier refer to Section~\ref{sec:exp}.

\begin{figure}[t]
    \centering
    \includegraphics[width=0.4\textwidth]{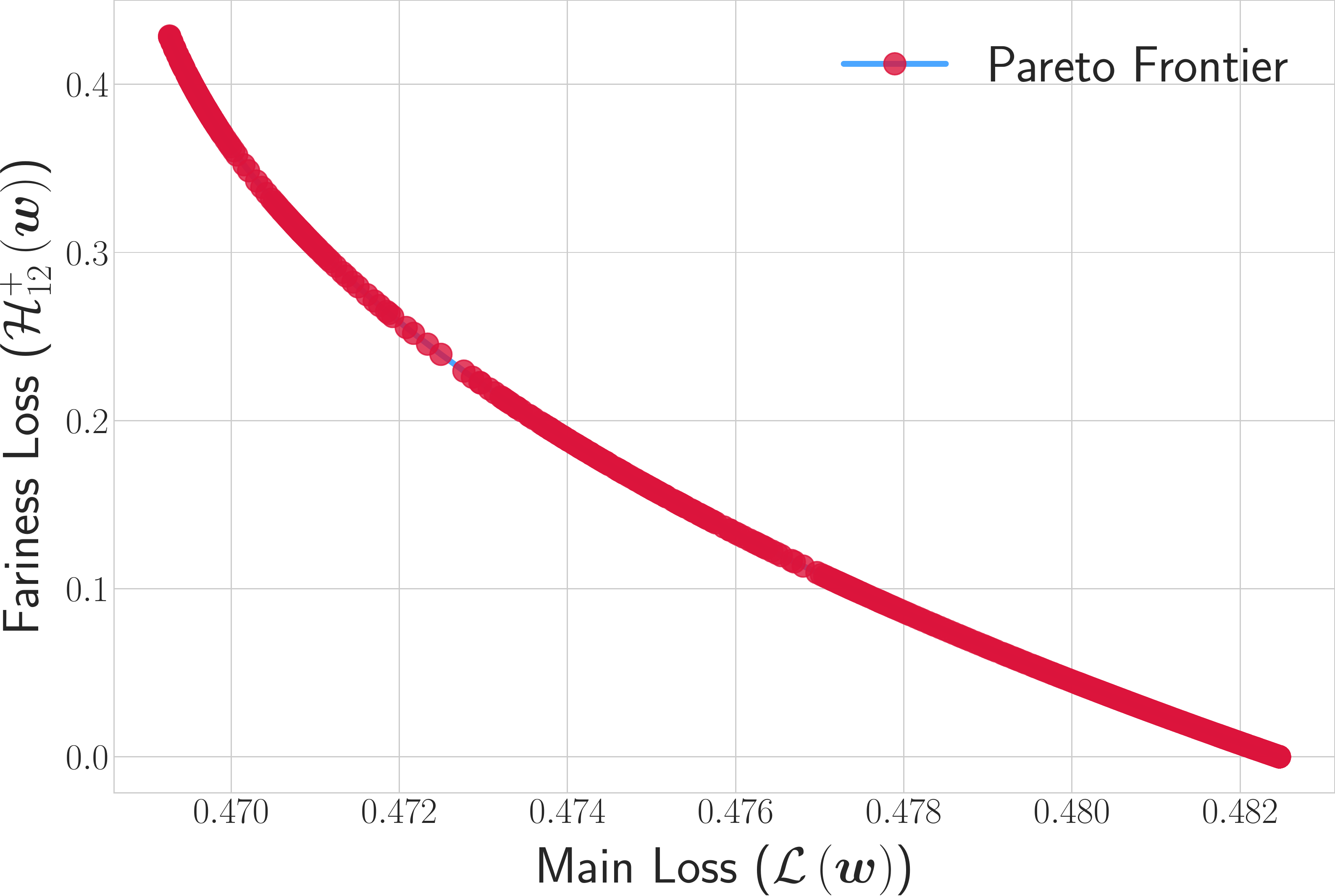}
    \caption{Pareto frontier extracted for the Adult dataset using the proposed \texttt{PB-PDO} algorithm with logistic regression as the main loss. It shows the trade-off between empirical risk and the fairness loss (equality of opportunity) defined in~(\ref{eq:eo_obj}). We run \texttt{PB-PDO} for $10$ times each with different preference vector, and it results in more than $15K$ non-dominated points on the Pareto frontier for this objective vector.}
    \label{fig:adult_pf}
\end{figure}
\section{Convergence Analysis}\label{sec:converge}
We now turn to analyzing the convergence of the proposed algorithm. To prove the convergence of the bilevel optimization introduced in~(\ref{eq:bilevel}), we first need to discuss the effect of the approximation of the inner level's solution on that of the outer level. Because we are solving both levels with gradient descent, a residual error from the inner level will be propagated through the outer level updates. Before diving into the convergence analysis, we indicate the assumptions used in the analysis. The first set of assumptions is related to the outer-level objective function, which is the weighted sum of objective functions in vector optimization of~(\ref{eq:mul-fair}).
\edef\oldassumption{\the\numexpr\value{assumption}+1}
\setcounter{assumption}{0}
\renewcommand{\theassumption}{\oldassumption.\alph{assumption}}
\begin{assumption}\label{ass:f}
Objective functions $\mathrm{h}_i\left(\bm{w}\right),\, i\in[m]$, are bounded by a constant  $\mathsf{D}_{\bm{\mathrm{h}}} = \underset{\w\in\mathcal{W}}{\max}  \llVert \bm{\mathrm{h}} \left(\w\right)\rrVert$.
\end{assumption}
\begin{assumption}\label{ass:f1}
Objective functions $\mathrm{h}_i\left(\bm{w}\right),\, i\in[m]$, are differentiable and have bounded gradients of $\llVert \bm{\mathrm{g}}_i \left(\w\right)\rrVert \leq \mathsf{B}_{\bm{\mathrm{h}}}, \forall i \in [m]$.
\end{assumption}
\begin{assumption}\label{ass:f2}
Objective functions $\mathrm{h}_i\left(\bm{w}\right),\, i\in[m]$, are smooth with Lipschitz constant of $L_i, i\in[m]$. That is, for every $\bm{w}_1,\bm{w}_2 \in \mathcal{W}$ and for every $i\in[m]$, we have:
\begin{equation}
    \llVert \nabla \mathrm{h}_i\left(\bm{w}_1\right) - \nabla \mathrm{h}_i\left(\bm{w}_2\right)\rrVert \leq L_i \llVert\bm{w}_1 - \bm{w}_2\rrVert\;.
\end{equation}
\end{assumption}
\let\theassumption\origtheassumption
Now, we elaborate on the set of assumptions for the inner level objective function $\Phi\left(\bm{\mathrm{h}}\left(\bm{w}\right),\bm{\alpha}\right)$ as follows:
\edef\oldassumption{\the\numexpr\value{assumption}-1}
\setcounter{assumption}{0}
\renewcommand{\theassumption}{\oldassumption.\alph{assumption}}
\begin{assumption}\label{ass:inner}
By having $m < d$, the inner-level objective function $\Phi\left(\bm{\mathrm{h}}\left(\bm{w}\right),\al\right)$ is strongly convex by parameter $\mu_\Phi$. That is, $\nabla^2_{\w\w} \Phi\left(\bm{\mathrm{h}}\left(\bm{w}\right),\al\right) \succeq \mu_\Phi \bm{I}$\;.
\end{assumption}
\begin{assumption}\label{ass:inner1}
The inner objective function $\Phi\left(\bm{\mathrm{h}}\left(\bm{w}\right),\al^*\left(\w\right)\right)$ is smooth with Lipschitz parameter $L_\Phi$. That is, for every $\bm{w}_1,\bm{w}_2 \in \mathcal{W}$, we have:
\begin{equation}
    \llVert \nabla \Phi\left(\bm{\mathrm{h}}\left(\bm{w}_1\right),\al^*\left(\w_1\right)\right) - \nabla \Phi\left(\bm{\mathrm{h}}\left(\bm{w}_2\right),\al^*\left(\w_2\right)\right) \rrVert 
    \leq L_\Phi \llVert\bm{w}_1 - \bm{w}_2\rrVert\;. 
\end{equation}
\end{assumption}
\begin{assumption}\label{ass:inner2}
The second-order derivative of the inner-level objective $\nabla^2_{\w\al}\Phi\left(\bm{\mathrm{h}}\left(\bm{w}\right),\al\right)$ is $L_{\w\al}$-Lipschitz continuous and bounded by $\llVert \nabla^2_{\w\al}\Phi\left(\bm{\mathrm{h}}\left(\bm{w}\right),\al\right) \rrVert \leq \mathsf{H}_{\Phi}$.
\end{assumption}
The last assumption for the inner-level objective is being commonly used by other machine learning and optimization frameworks in~\citep{nesterov2006cubic,jin2017escape,rajeswaran2019meta,ghadimi2018approximation}. Now, using these assumptions, we turn into the convergence analysis of Algorithm~\ref{alg:fairpareto}. First, we should bound the error on the solution of the inner level.
\begin{lemma}\label{lemma:sconv-converg}
Let $\Phi\left(\bm{\mathrm{h}}\left(\bm{w}\right),\bm{\alpha}\right)$ be a strongly convex and smooth function with respect to $\bm{\alpha}$ with $\mu_\Phi$ and $L_\Phi$ as its strong convexity and smoothness parameters, respectively. Then, if the inner learning rate is $\rho\leq \frac{2}{L_\Phi}$, the solution found by the inner level of Algorithm~\ref{alg:fairpareto} after $K$ updates is bounded by:
\begin{equation}\label{eq:strongconvex}
    \left\lVert \hat{\bm{\alpha}}^{(t)} - \bm{\alpha}^*\left(\bm{w}^{(t)}\right)\right\rVert^2 \leq \exp\left(-\frac{K}{\kappa}\right)\left\lVert \bm{\alpha}^{(0)} - \bm{\alpha}^*\left(\bm{w}^{(t)}\right) \right\rVert^2,
\end{equation}
where $\bm{\alpha}^*\left(\bm{w}^{(k)}\right)$ is the optimal weights for point $\bm{w}^{(k)}$ and $\kappa = \frac{\mu_\Phi}{L_\Phi}$ is the condition number of function $\Phi\left(\bm{\mathrm{h}}\left(\bm{w}\right),\bm{\alpha}\right)$. 
\end{lemma}
\begin{proof}
This is the standard convergence rate for a strong convex and smooth function using gradient descent. For the detailed proof refer to~\cite{bubeck2015convex}.
\end{proof}
Then, we can write the gradient of the outer level function based on the inner level using chain rule as follows:
\begin{equation}
    \nabla \Psi\left(\bm{\mathrm{h}}\left(\bm{w}\right),\hat{\bm{\alpha}}\left(\bm{w}\right)\right) = \nabla_{\bm{w}} \Psi\left(\bm{\mathrm{h}}\left(\bm{w}\right),\hat{\bm{\alpha}}\left(\bm{w}\right)\right) + \nabla \hat{\bm{\alpha}}\left(\bm{w}\right)\nabla_{\bm{\alpha}} \Psi\left(\bm{\mathrm{h}}\left(\bm{w}\right),\hat{\bm{\alpha}}\left(\bm{w}\right)\right)\;.
\end{equation}
The issue here is that we do not have the exact value of $\nabla \hat{\bm{\alpha}}\left(\bm{w}\right)$, and that's where the approximation error comes into effect. From the optimality of the inner level, which is $\nabla_{\bm{\alpha}} \Phi\left(\bm{\mathrm{h}}\left(\bm{w}\right),\bm{\alpha}^*\left(\bm{w}\right)\right) = 0$, if we take the gradient from both sides, we will have:
\begin{equation}\label{eq:grad2}
    \nabla \bm{\alpha}^*\left(\bm{w}\right) = - \nabla^2_{\bm{w}\bm{\alpha}}\Phi\left(\bm{\mathrm{h}}\left(\bm{w}\right),\bm{\alpha}^*(\bm{w})\right)\left[\nabla^2_{\bm{\alpha}\bm{\alpha}}\Phi\left(\bm{\mathrm{h}}\left(\bm{w}\right),\bm{\alpha}^*(\bm{w})\right)\right]^{-1}.
\end{equation}
This is valid for the optimal point only, but we can use this approximation on the solution point found by Algorithm~\ref{alg:fairpareto}. Then,
equipped with Lemma~\ref{lemma:sconv-converg}, we can bound the error in this approximation for the gradient of the outer level using the following lemma.
\begin{lemma}[Gradient error]\label{lemma:ge}
Consider that the inner- and outer-level functions in optimization~(\ref{eq:bilevel}) are following the Assumptions~\ref{ass:f},~\ref{ass:f1},~\ref{ass:inner}, and~\ref{ass:inner2} and that the solution of the inner level is found by the Algorithm~\ref{alg:fairpareto}, then the error in the gradient of the outer-level function  is bounded by
\begin{equation}
    \llVert \nabla \Psi\left(\bm{\mathrm{h}}\left(\bm{w}\right),\hat{\al}\left(\w\right)\right) - \nabla \Psi\left(\bm{\mathrm{h}}\left(\bm{w}\right),\al^*\left(\w\right)\right)\rrVert \leq A_\Psi \llVert \hat{\al} - \al^* \rrVert\;,
\end{equation}
where $A_\Psi = \mathsf{B}_{\bm{\mathrm{h}}} \sqrt{m} +  \frac{\mathsf{D}_{\bm{\mathrm{h}}} \sqrt{m}}{\mu_\Phi} L_{\w\al}$.
\end{lemma}
\begin{proof}
The proof is provided in Appendix~\ref{app:lemma:ge}, where we use the equation and its approximation for non-optimal points in~(\ref{eq:grad2}), in addition to the Assumptions mentioned above.
\end{proof}
This lemma shows that the error in the gradient of the outer level because of an inexact solution of the inner level can be bounded by the error in the solution of the inner level, which is then bounded by~(\ref{eq:strongconvex}). 

In the assumptions, we did not talk about the smoothness of the outer-level function. Next, using the smoothness of each objective function separately, we show that the outer-level function is also smooth, which is deliberated by the following lemma.
\begin{lemma}\label{lemma:smooth}
The outer level function $\Psi\left(\bm{\mathrm{h}}\left(\bm{w}\right),\al\right)$ is smooth with Lipschitz constant $L_\Psi$, that is:
\begin{equation}\label{eq:smooth}
    \llVert \nabla \Psi\left(\bm{\mathrm{h}}\left(\bm{w}_1\right),\al^*\left(\w_1\right)\right) - \nabla \Psi\left(\bm{\mathrm{h}}\left(\bm{w}_2\right),\al^*\left(\w_2\right)\right)\rrVert \leq L_\Psi \llVert\w_1 - \w_2\rrVert\;,
\end{equation}
where $L_\Psi = L_\text{max} + \frac{\mathsf{H}_{\Phi}\mathsf{B}_{\bm{\mathrm{h}}} \sqrt{m}}{\mu_\Phi}$, in which $L_\text{max} = \max\left\{L_1,\ldots,L_m\right\}$ and $L_i, i\in [m]$ is the smoothness parameter of the $i$-th objective from Assumption~\ref{ass:f2}.
\end{lemma}
Now, equipped with Lemmas~\ref{lemma:sconv-converg},~\ref{lemma:ge}, and~\ref{lemma:smooth}, we can turn into the convergence analysis of the optimization in Algorithm~\ref{alg:fairpareto}. Note that, we do not have any convexity assumption for the outer-level function in our assumptions set. Hence, we consider two cases, where the outer-level function is convex or non-convex. First, we consider the case where all the objectives are convex functions. Because in $\Psi(\w,\al)$ the objectives are summed with positive weights from a simplex, $\Psi(\w,\al)$ is convex as well. Hence, we have the following theorem for the convergence of Algorithm~\ref{alg:fairpareto}.
\begin{theorem}[Convex Objectives]\label{theorem:convex_convergence}
Let $\bm{\mathrm{h}}\left(\w\right) = \left[\mathrm{h}_1\left(\w\right), \ldots, \mathrm{h}_m\left(\w\right)\right]$ be the convex objective vector that follows the properties in Assumptions~\ref{ass:f},~\ref{ass:f1}, and~\ref{ass:f2}. Then the function $\Psi\left(\bm{\mathrm{h}}\left(\bm{w}\right),\al\right)$ is a convex and smooth function with Lipschitz constant of $L_\Psi$ from Lemma~\ref{lemma:smooth}. Also, the inner function of optimization~(\ref{eq:bilevel}) has the properties in Assumptions~\ref{ass:inner},~\ref{ass:inner1}, and~\ref{ass:inner2}. Then, for the sequence of solutions $\w^{(0)},\ldots, \w^{(T)}$, generated by the Algorithm~\ref{alg:fairpareto}, by setting $\eta\leq 1/L_\Psi$, we have
\begin{align}\label{eq:convex_convergence}
    \Psi\left(\bm{\mathrm{h}}\left(\bar{\bm{w}}\right),\bar{\al}\left(\bar{\bm{w}}\right) \right) - \Psi\left(\bm{\mathrm{h}}\left(\bm{w}^*\right),\al^*\left(\w^*\right)\right) &\leq \frac{1}{2\eta T}\llVert \w^{(0)} - \w^* \rrVert^2 \\ \nonumber
    & \quad + \frac{RA_\Psi }{T}\exp\left(-\frac{K}{2\kappa}\right)\sum_{t=0}^{T-1}  \llVert \al^{(0)}\left(\w^{(t)}\right) - \al^*\left(\w^{(t)}\right)\rrVert\;,
\end{align}
where $R$ is the bound on the domain of the solutions and $\bar{\w} = \frac{1}{T}\sum_{t=1}^{T}\w^{(t)}$. Here $\bar{\al} = \left[\bar{\alpha}_1,\ldots,\bar{\alpha}_m\right]$, and $\bar{\alpha}_i = \min_t \alpha^*_i\left(\w^{(t)}\right)$ for $1\leq t\leq T$ and $i \in [m]$.
\end{theorem}
\begin{proof}
The proof using convexity assumption as well as Lemmas~\ref{lemma:sconv-converg},~\ref{lemma:ge}, and~\ref{lemma:smooth} is provided in Appendix~\ref{app:theorem:convex_convergence}.
\end{proof}
\begin{remark}
The convergence inequality for convex objectives in~(\ref{eq:convex_convergence}) has two terms. The first term is the standard convex optimization term, and the second term comes from the error in finding the solution of the inner level at each iteration of the outer level.
\end{remark}
To achieve an $\epsilon$-accurate PEF solution for the outer level, we need to take $T=\order\left(\frac{1}{\epsilon}\right)$ steps of gradient descent. In this setting, the error of approximation for the inner level can be intensified by increasing the number of fairness objectives $m$. This is due to $A_\Psi = \order\left(\sqrt{m}\right)$.

Whenever the objective functions of the problem are not convex, following the standard non-convex optimization, we have:
\begin{theorem}[Non-convex Objectives]\label{theorem:nonconvex_convergence}
Let $\bm{\mathrm{h}}\left(\w\right) = \left[\mathrm{h}_1\left(\w\right), \ldots, \mathrm{h}_m\left(\w\right)\right]$ be the objective vector that follows the properties in Assumptions~\ref{ass:f},~\ref{ass:f1}, and~\ref{ass:f2}. Then the function $\Psi\left(\bm{\mathrm{h}}\left(\bm{w}\right),\al\right)$ is a smooth function with Lipschitz constant of $L_\Psi$ from Lemma~\ref{lemma:smooth}. Also, the inner function of optimization~(\ref{eq:bilevel}) has the properties in Assumptions~\ref{ass:inner},~\ref{ass:inner1}, and~\ref{ass:inner2}. Then, for the average squared norm of the gradient of the outer function in the sequence of solutions $\w^{(0)},\ldots, \w^{(T)}$, generated by Algorithm~\ref{alg:fairpareto}, by setting $\eta\leq 1/L_\Psi$, we have
\begin{align}
    \frac{1}{T}\sum_{t=0}^{T-1} \left\lVert \nabla\Psi\left(\bm{\mathrm{h}}\left(\bm{w}^{(t)}\right) \right. \right. & \left.\left. ,\al^*\left(\w^{(t)}\right)\right) \right\rVert^2 \nonumber\\
    &\leq \frac{2}{\eta T}\left(\Psi\left(\bm{\mathrm{h}}\left(\bm{w}^{(0)}\right),\al^*\left(\w^{(0)}\right)\right) - \Psi\left(\bm{\mathrm{h}}\left(\bm{w}^*\right),\al^*\left(\w^*\right)\right)\right) \nonumber \\
    &\quad + \frac{L_\Psi\eta A_\Psi^2}{ T}\exp\left(-\frac{K}{\kappa}\right)\sum_{t=0}^{T-1} \llVert \al^{(0)}\left(\w^{(t)}\right) - \al^*\left(\w^{(t)}\right)\rrVert^2\;.
    \label{eq:non-convex-converge}
\end{align}
\begin{proof}
The proof using smoothness assumption of the outer level in Lemma~\ref{lemma:smooth}, and also Lemmas~\ref{lemma:sconv-converg}, and~\ref{lemma:ge} is deferred to Appendix~\ref{app:theorem:nonconvex_convergence}.
\end{proof}
\begin{remark}
Using the convergence analysis of non-convex outer-level objective function in~(\ref{eq:non-convex-converge}), we can bound the minimum squared norm of its gradient to be less than $\epsilon$, with taking $\order\left(\frac{1}{\epsilon}\right)$ gradient steps. Note that, when this gradient is zero, we are in a Pareto stationary point of the problem.
\end{remark}
\end{theorem}

\section{Experimental Results}\label{sec:exp}

In this section, we empirically examine the efficacy of the proposed algorithms. The experiments are designed to answer the following questions:
\begin{enumerate}
    \item How the proposed Pareto descent algorithm performs compared to a normal classifier in minimizing the empirical risk and fairness violation with a binary sensitive feature?
    \item How the proposed Pareto descent algorithm performs compared to a normal classifier when there is a multiple group sensitive feature?
    \item How the proposed Pareto descent algorithm performs compared to state-of-the-art algorithms for fairness-aware learning, in particular, the fairness constrained empirical risk minimization proposed in~\cite{donini2018empirical} and reduction to minimax optimization proposed in~\cite{agarwal2018reductions}?
    \item How the proposed preference-based Pareto descent optimization can find solutions from different parts of the Pareto frontier? How it performs compared to a state-of-the-art algorithm proposed in~\cite{agarwal2018reductions}?
\end{enumerate}
In all the experiments, we conduct multiple runs and report the average and the respective variance if applicable. To compare with both state-of-the-art algorithms (FERM~\citep{donini2018empirical} and minimax reduction~\citep{agarwal2018reductions}) we run algorithms using linear SVM as the main objective function because FERM is designed for this objective. In the fourth part, since we are finding the points from the Pareto frontier and only the minimax reduction approach~\citep{agarwal2018reductions} claims to find such points, we compare with their algorithm using linear SVM and logistic regression as the main objective function. Nonetheless, our proposed algorithms are not bounded to any objective functions and can be implemented in any setting.

\paragraph{Datasets} We will use two real-world datasets: The Adult income dataset\footnote{\url{https://archive.ics.uci.edu/ml/datasets/Adult}} and the COMPAS dataset~\citep{propublica}. The meta-data for these datasets is provided in Appendix~\ref{app:exp}. In the Adult dataset, the goal is to predict whether the income of each person is greater or less than $\$50$K per year, based on census data. In the COMPAS dataset, the task is to predict whether the criminal defendant would recidivate within the next two years or not based on historical data. In both datasets, the positive outcome (have income more than $\$50$K per year in the Adult dataset or not recidivate within the next two years in the COMPAS dataset) is considered beneficial. In the Adult dataset, we have two sensitive features, \texttt{gender}, and \texttt{race}, where gender is a binary feature, while race in this dataset is a multiple-group feature with 5 categories. In the COMPAS dataset, also, we have two sensitive features, \texttt{sex}, and \texttt{race}, both of which are binary sensitive features.

\paragraph{\ding{182}~Binary sensitive feature: Pareto vs. Normal} In the first set of experiments, we examine the effectiveness of the proposed algorithm in satisfying the Pareto efficient equality of opportunity as defined in Section~\ref{sec:peeo}. Note that, under this condition, the goal is to achieve equal true positive rates among sensitive groups. We apply Algorithm~\ref{alg:fairpareto} to the Adult dataset with gender as the sensitive feature and also the COMPAS dataset with race and sex as its sensitive features. The results are shown in Figure~\ref{fig:result} for our algorithm compared to normal training using a linear SVM, which indicates that the proposed algorithm can superbly satisfy the notion of equality of opportunity. 
As it can be inferred from the first column, in all three experiments the total (and also each sensitive group) accuracy for our proposed algorithm is almost the same as the normal training one. However, in the second column, we can see that our algorithm has equal true positive rates among sensitive groups, while the gap between different groups is huge for normal training. This is achieved by the proposed algorithm, with almost the same total true positive rate as the normal training except for the first one (the Adult dataset). In this case, this degradation is compensated by smaller false positive rates for the proposed algorithm compared to the normal training. Also note that the Pareto descent algorithm achieves equal false positive rates among sensitive groups, to some degree, which is not in the objectives of the  equality of opportunity.

\begin{figure}[t]
    \centering
    \subfigure[Adult (gender)]{
		\centering
		\includegraphics[width=0.98\textwidth]{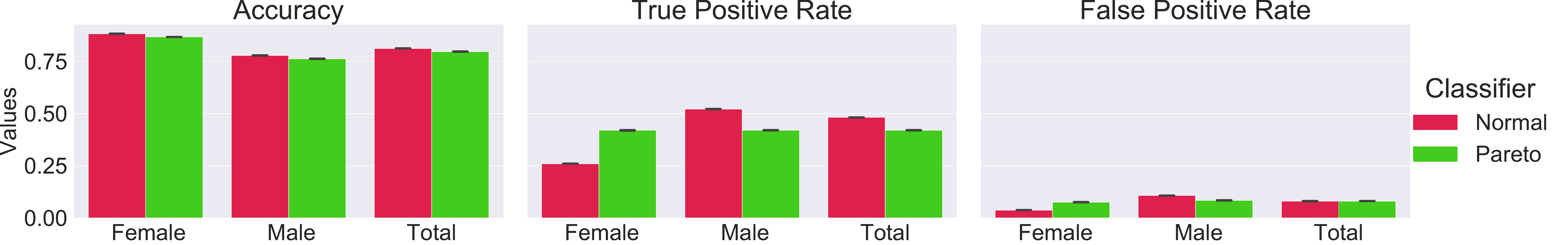}
		\label{fig:adult_all_9}
		}

	\subfigure[COMPAS (race)]{
	\vspace{0.5cm}
		\centering
		\includegraphics[width=0.98\textwidth]{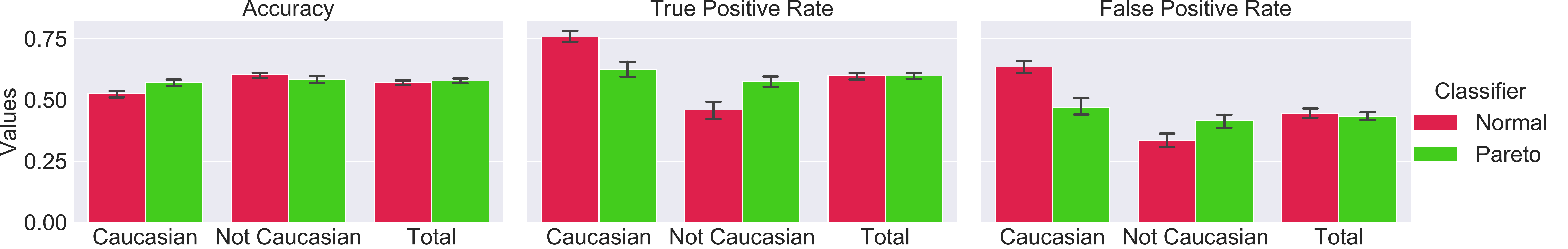}
		\label{fig:compas_all_2}\vspace{-0.3cm}
	}
	\subfigure[COMPAS (sex)]{
		\vspace{0.5cm}
		\centering
		\includegraphics[width=0.98\textwidth]{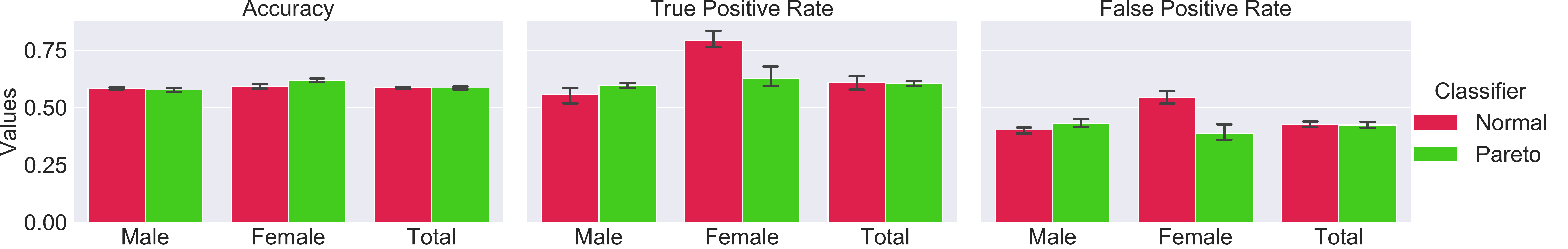}
		\label{fig:compas_all_1}
	}
    
    \caption{The results of applying Algorithm~\ref{alg:fairpareto} to satisfy the Pareto efficient equality of opportunity for (a) Adult dataset with gender as a sensitive feature, (b) COMPAS dataset with race as a sensitive feature (c) COMPAS dataset with sex as a sensitive feature. The drop in accuracy is small in all cases, while the goal of equal opportunity is satisfied perfectly. By looking at false positive rates, it seems that this notion is capable of satisfying equalized odds as well as equality of opportunity.} 
    \label{fig:result}
\end{figure}

\paragraph{\ding{183}~Multiple group sensitive feature: Pareto vs. Normal} In Figure~\ref{fig:res-multi}, we show the results for applying the proposed algorithm to the Adult dataset with considering its multiple group sensitive feature, race. As it can be inferred, our algorithm can superbly satisfy the fairness constraints, while maintaining high accuracy close to the baseline. The figure in the middle shows that the Pareto descent algorithm achieves equal, and also higher, true positive rates among sensitive groups, though this comes at the price of slightly increasing false positive rates for them.
Note that, for this case, we have $10$ fairness objectives in addition to the main learning objective, which demonstrates the power of the proposed algorithm in finding an optimal point in terms of compromises between objectives. On the other hand, the applications of FERM~\citep{donini2018empirical} to multiple group sensitive features are not straightforward. Also, applying the minimax reduction approach~\citep{agarwal2018reductions} to multiple group sensitive features is heavily computationally expensive. 
\begin{figure}[t]
    \centering
    \includegraphics[width=\textwidth]{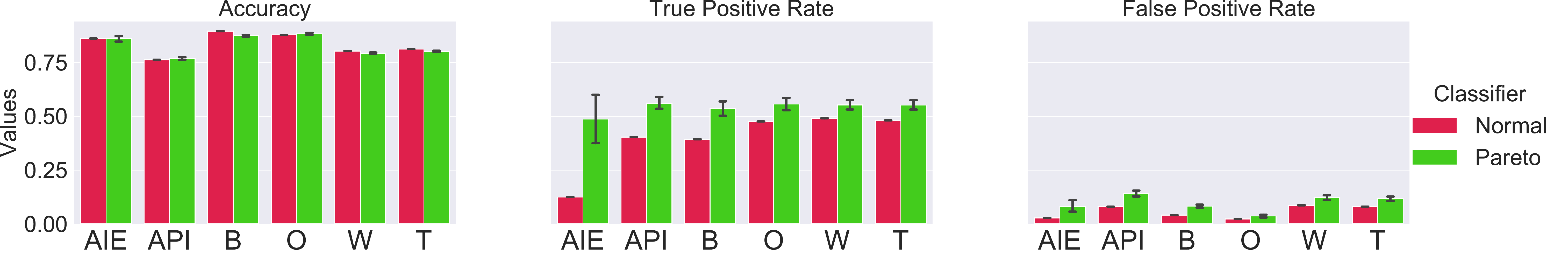}
    \caption{Adult dataset with race as its sensitive feature, which has 5 groups. In this case, our proposed algorithm can superbly satisfy the property of Pareto efficient equality of opportunity, even though we have 11 objectives in the minimization. The groups are AIE: Amer-Indian-Eskimo, API: Asian-Pac-Islander, B: Black, O: Other, W: White, and T represents the total.}
    \label{fig:res-multi}
\end{figure}

\paragraph{\ding{184}~Pareto vs. the state of the art}
From a more general perspective, the overarching goal of fairness-aware learning algorithms is to satisfy multiple objectives at the same time before, during, or after training a model. This view has been reflected as a constraint optimization problem at the algorithmic level in various approaches. In general, this optimization can be written as:
\begin{equation}
\begin{array}{rrclcl}
\displaystyle \min_{\bm{w}} & \multicolumn{3}{l}{\mathcal{L}(\bm{w};\mathcal{D})} \\
\textrm{s.t.} & \mathrm{h}_i\left(\bm{w};\mathcal{D}\right) & \leq & \epsilon_i, \quad \forall i \in [m], \\
\end{array}
\label{eq:cons_op}
\end{equation}
where $\mathrm{h}_i\left(\bm{w};\mathcal{D}\right)$ is the $i$th fairness constraint we required to satisfy. In the quest for achieving fairness using constrained optimization in~(\ref{eq:cons_op}), two well-known studies by~\cite{donini2018empirical} and~\cite{agarwal2018reductions} introduced their frameworks using different methods. They attempted to solve this constrained optimization problem either by linear approximation of their constraints or forming the Lagrangian function to solve the saddle point minimax problem. In the latter,~\cite{agarwal2018reductions} used a grid search technique for binary sensitive feature problems to find the best weights for constraints in the Lagrangian objective. The greatest challenge in these approaches is to set the violation parameter $\epsilon_i$ for each fairness constraint in the optimization. Also, finding the weights for each constraint in these approaches does not guarantee the Pareto efficiency of the solution. Not to mention the high computational demand of approaches such as grid search in~\cite{agarwal2018reductions}.

Now, we can compare the results generated with Algorithm~\ref{alg:fairpareto} with these two state-of-the-art algorithms for satisfying fairness measures, known as FERM~\citep{donini2018empirical} and minimax reduction~\citep{agarwal2018reductions}. To compare with both, we run the experiments using a linear SVM loss function since FERM is designed for this loss function. We also use the DEO to show fairness measures in addition to accuracy on the test dataset. Table~\ref{tab:res} summarizes the results we got from applying a normal linear SVM without fairness constraints, FERM linear SVM, minimax reduction on linear SVM, and our algorithm with gradient descent implementation of SVM. For the minimax reduction algorithm~\citep{agarwal2018reductions}, each grid search learns 100 different classifiers, and the best non-dominated ones are selected. The numbers reported in Table~\ref{tab:res} are the average of best non-dominated points for several grid-search runs in terms of DEO.  Overall, it shows that our algorithm can achieve a superb DEO compared to FERM and minimax reduction while having a better accuracy as well. The results show that in all three datasets our algorithm's results dominate the solution of both state-of-the-art algorithms.

\begin{table}[t]
\centering
\resizebox{0.8\textwidth}{!}{%
\begin{tabular}{c|c|c|c|c|c|c|}
\cline{2-7}
 & \multicolumn{2}{c|}{Adult~(gender)} & \multicolumn{2}{c|}{COMPAS~(sex)} & \multicolumn{2}{c|}{COMPAS~(race)} \\ \cline{2-7} 
 & Acc & DEO & ACC & DEO  & ACC & DEO  \\ \hline
\multicolumn{1}{|c|}{Normal} & ${0.8123}$ & $0.2615$ & \begin{tabular}[c]{@{}c@{}}$0.6037$\\ $\pm0.0058$\end{tabular} & \begin{tabular}[c]{@{}c@{}}$0.2643$\\ $\pm0.0269$\end{tabular} & \begin{tabular}[c]{@{}c@{}}$0.6037$\\ $\pm0.0058$\end{tabular} & \begin{tabular}[c]{@{}c@{}}$0.2643$\\ $\pm0.0269$\end{tabular} \\ \hline
\multicolumn{1}{|c|}{\begin{tabular}[c]{@{}c@{}}FERM \\ \citep{donini2018empirical}\end{tabular}} & $0.7937$ & $0.0173$ & \begin{tabular}[c]{@{}c@{}}$0.5126$\\ $\pm0.0489$\end{tabular} & \begin{tabular}[c]{@{}c@{}}$0.09337$\\ $\pm0.0724$\end{tabular} & \begin{tabular}[c]{@{}c@{}}$0.5862$\\ $\pm0.0186$\end{tabular} & \begin{tabular}[c]{@{}c@{}}$0.0877$\\ $\pm0.0645$\end{tabular}\\ \hline
\multicolumn{1}{|c|}{\begin{tabular}[c]{@{}c@{}} Minimax Reduction \\ \citep{agarwal2018reductions}\end{tabular}} & \begin{tabular}[c]{@{}c@{}}$0.7748$\\ $\pm0.0002$\end{tabular} & \begin{tabular}[c]{@{}c@{}}$0.0335$\\ $\pm0.0169$\end{tabular} & \begin{tabular}[c]{@{}c@{}}$0.5805$\\ $\pm0.0047$\end{tabular} & \begin{tabular}[c]{@{}c@{}}$0.0378$\\ $\pm0.0230$\end{tabular} & \begin{tabular}[c]{@{}c@{}}$0.5833$\\ $\pm0.0087$\end{tabular} & \begin{tabular}[c]{@{}c@{}}$0.0264$\\ $\pm0.0212$\end{tabular} \\ \hline
\multicolumn{1}{|c|}{\begin{tabular}[c]{@{}c@{}}\texttt{PDO}\\ (ours)\end{tabular}} & \begin{tabular}[c]{@{}c@{}}$0.7969$\\ $\pm0.0003$\end{tabular} & \begin{tabular}[c]{@{}c@{}}\textcolor{dodgerblue}{$\bm{0.0019}$}\\ \textcolor{dodgerblue}{$\bm{\pm0.0025}$}\end{tabular} & \begin{tabular}[c]{@{}c@{}}${0.5871}$\\ $\pm0.0216$\end{tabular} & \begin{tabular}[c]{@{}c@{}}\textcolor{dodgerblue}{$\bm{0.0354}$}\\ \textcolor{dodgerblue}{$\bm{\pm0.0293}$}\end{tabular}  & \begin{tabular}[c]{@{}c@{}}$0.6029$\\ $\pm0.0117$\end{tabular} & \begin{tabular}[c]{@{}c@{}}\textcolor{dodgerblue}{$\bm{0.0078}$}\\ \textcolor{dodgerblue}{$\bm{\pm0.0092}$}\end{tabular}\\ \hline
\end{tabular}%
}
\vskip 0.1cm
\caption{Comparison of the proposed \texttt{PDO} algorithm with FERM~\citep{donini2018empirical} and the minimax reduction~\citep{agarwal2018reductions}. Experiments are repeated multiple times with reporting the average and variance. Blue colors show the best results for the DEO measure, which in all cases belong to the proposed Pareto descent algorithm. Moreover, in terms of accuracy, the Pareto descent algorithm outperforms both state-of-the-art algorithms.}\label{tab:res}
\end{table}

\paragraph{\ding{185}~Pareto frontier: Pareto vs. Minimax Reduction~\citep{agarwal2018reductions}} 
Using the proposed algorithm \texttt{PB-PDO}, we can extract the points from the Pareto frontier of the vector objective we are minimizing. In~\cite{agarwal2018reductions} also, authors claim that using their approach they can find some points from the Pareto frontier of the problem. However, they are in fact extracting some non-dominated points from multiple runs of their algorithm that are not necessarily on the Pareto frontier. In this part, we show how our algorithm performs on extracting the Pareto frontier and compare it to the points found by the minimax reduction approach. We run the \texttt{PB-PDO} algorithm 10 times, for each dataset and each time with different preference vector, then we find the points that are not dominated in the trajectory. We set both $\epsilon_1$ and $\epsilon_2$ to $1e^{-2}$ for the \texttt{PB-PDO} in all algorithms. For the minimax reduction, we run their algorithm to learn 200 classifiers and then find the non-dominated points for reporting. 

We apply these algorithms to linear SVM and logistic regression. Figure~\ref{fig:pf-data_SVM} shows the results for applying both algorithms using linear SVM on the Adults dataset with gender and COMPAS with sex and race as their sensitive features. Figure~\ref{fig:pf-data} shows the same results for both algorithms using logistic regression. In both figures, the first column is showing the trade-off between the main loss and the fairness loss (here equality of opportunity as defined in~(\ref{eq:eo_obj})). The proposed \texttt{PB-PDO} algorithm is training the Pareto frontier for this trade-off. Then, we map the points on this space (loss-loss) to the error versus DEO space. It is worth mentioning that this mapping is from a continuous space to a non-continuous one, and hence, we will lose some of the points as they map to the same point or getting dominated by other points in the error-DEO space. This is exacerbated when the dataset is smaller like the COMPAS dataset. The second and the third columns are showing this mapping on the training and test dataset, respectively. From the results in Figures~\ref{fig:pf-data_SVM} and~\ref{fig:pf-data}, it is clear that almost all the points found by the minimax reduction approach as their non-dominated solutions are dominated by the points found by our proposed Pareto descent optimization approach. Also, an interesting observation in Figure~\ref{fig:pf-data_SVM}, last row for the COMPAS dataset with race as a sensitive feature, is that there is almost no trade-off between accuracy and DEO. Hence, we can decrease the error rate without too much increase in DEO.

\begin{figure*}[t!]
	\centering
	\subfigure{
		\centering
		\makebox[20pt]{\raisebox{40pt}{\rotatebox[origin=c]{90}{Adult (gender)}}}\hspace{3pt}
		\includegraphics[width=0.285\textwidth]{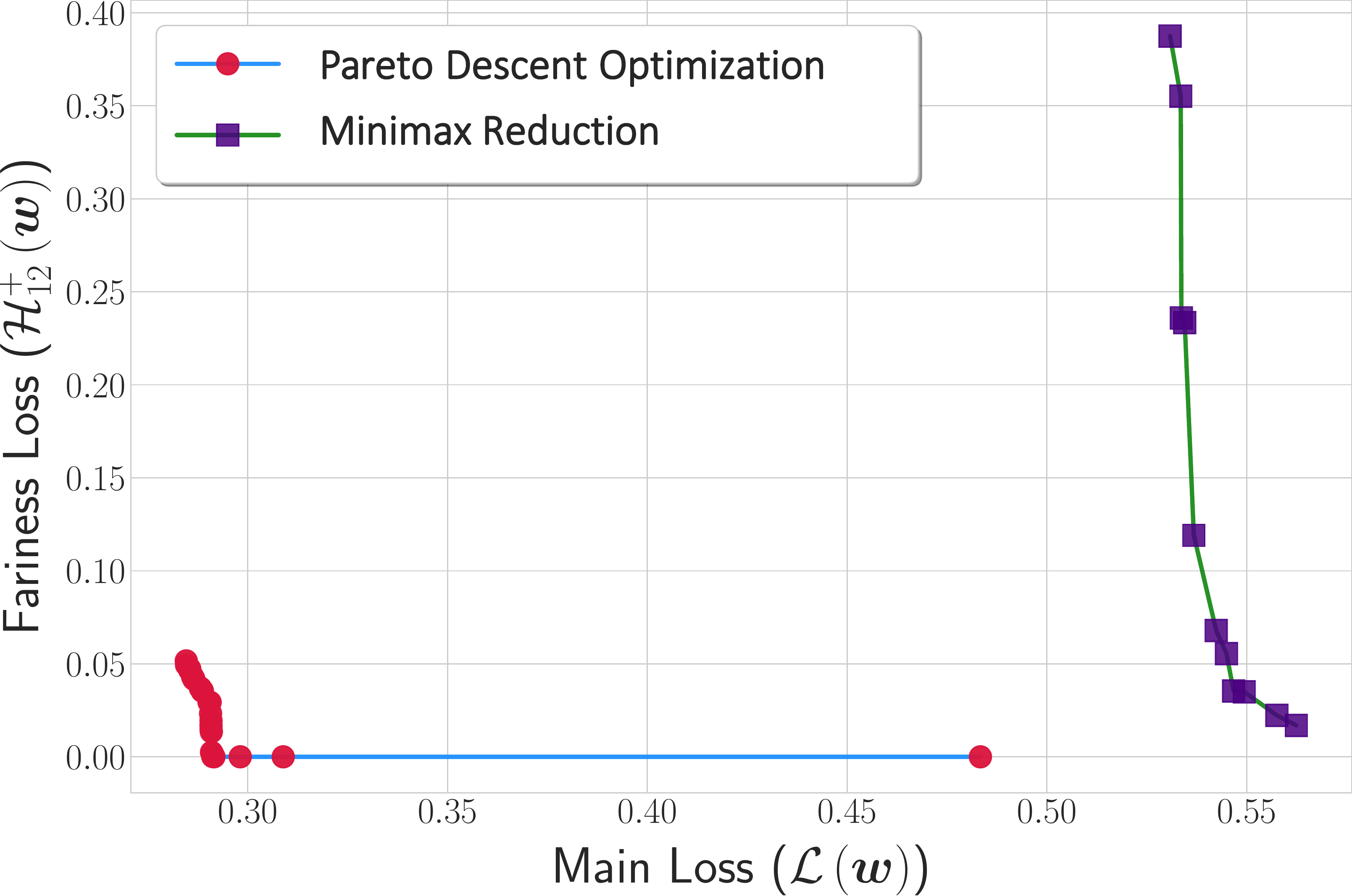}
		\label{fig:pf_loss_adult_SVM}
		}
		\hfill
		\subfigure{
			\centering 
			\includegraphics[width=0.285\textwidth]{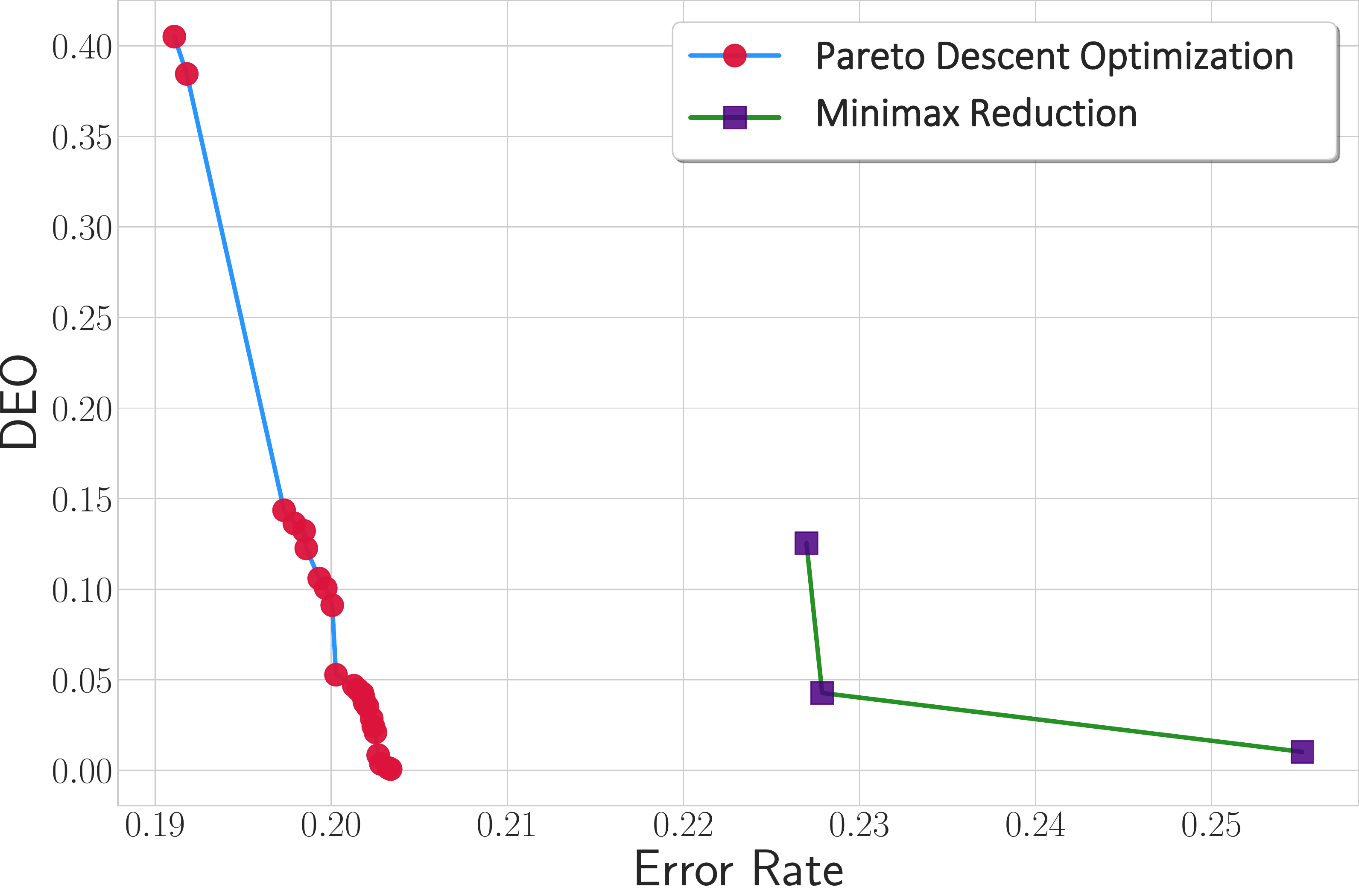}
			\label{fig:pf_err_train_adult_SVM}
			}
			\hfill
	\subfigure{  
		\centering 
		\includegraphics[width=0.285\textwidth]{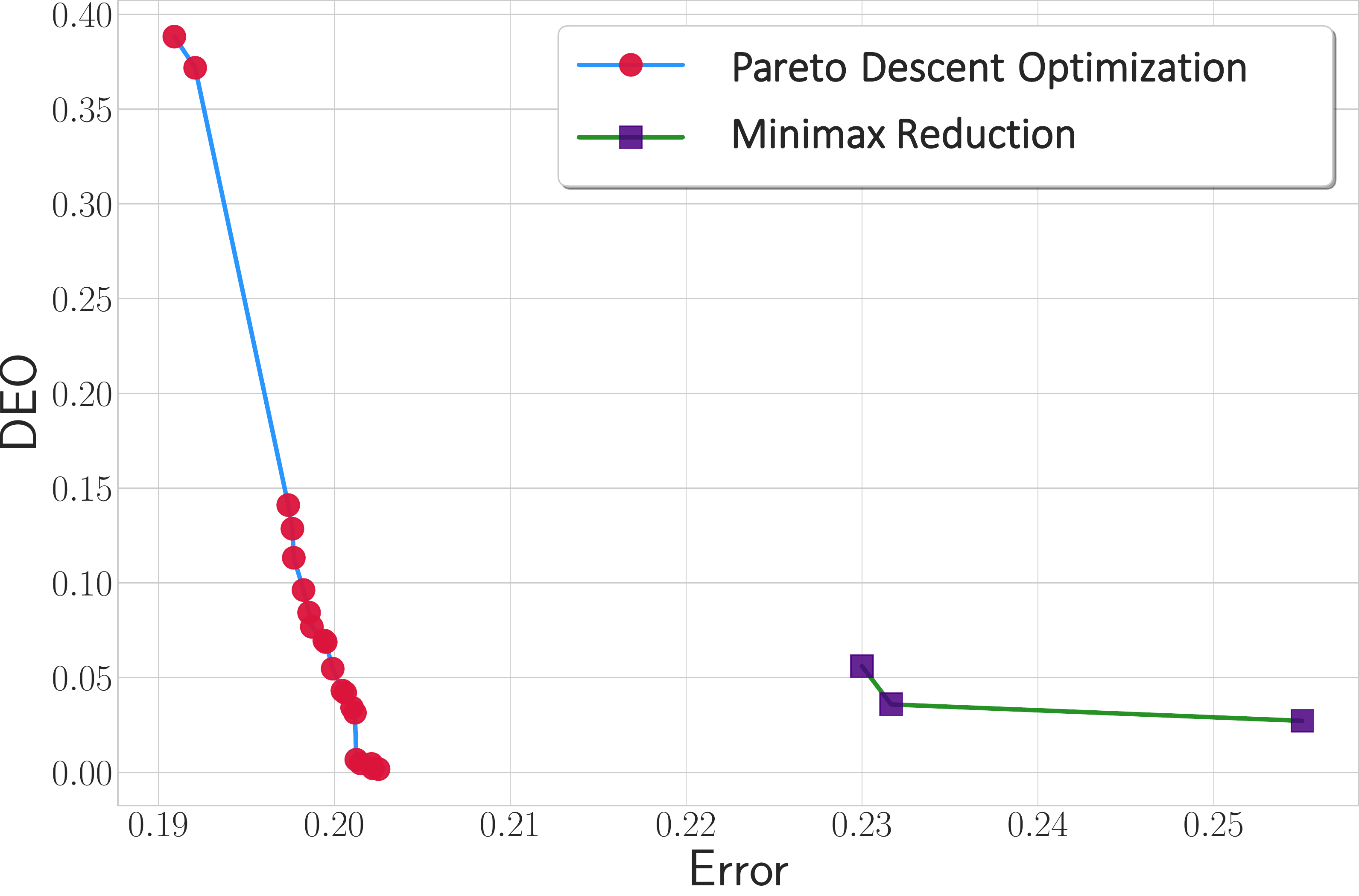}
		\label{fig:pf_err_test_adult_SVM}
		}
	
	\subfigure{
		\centering
		\makebox[20pt]{\raisebox{40pt}{\rotatebox[origin=c]{90}{COMPAS (sex)}}}\hspace{3pt}
		\includegraphics[width=0.285\textwidth]{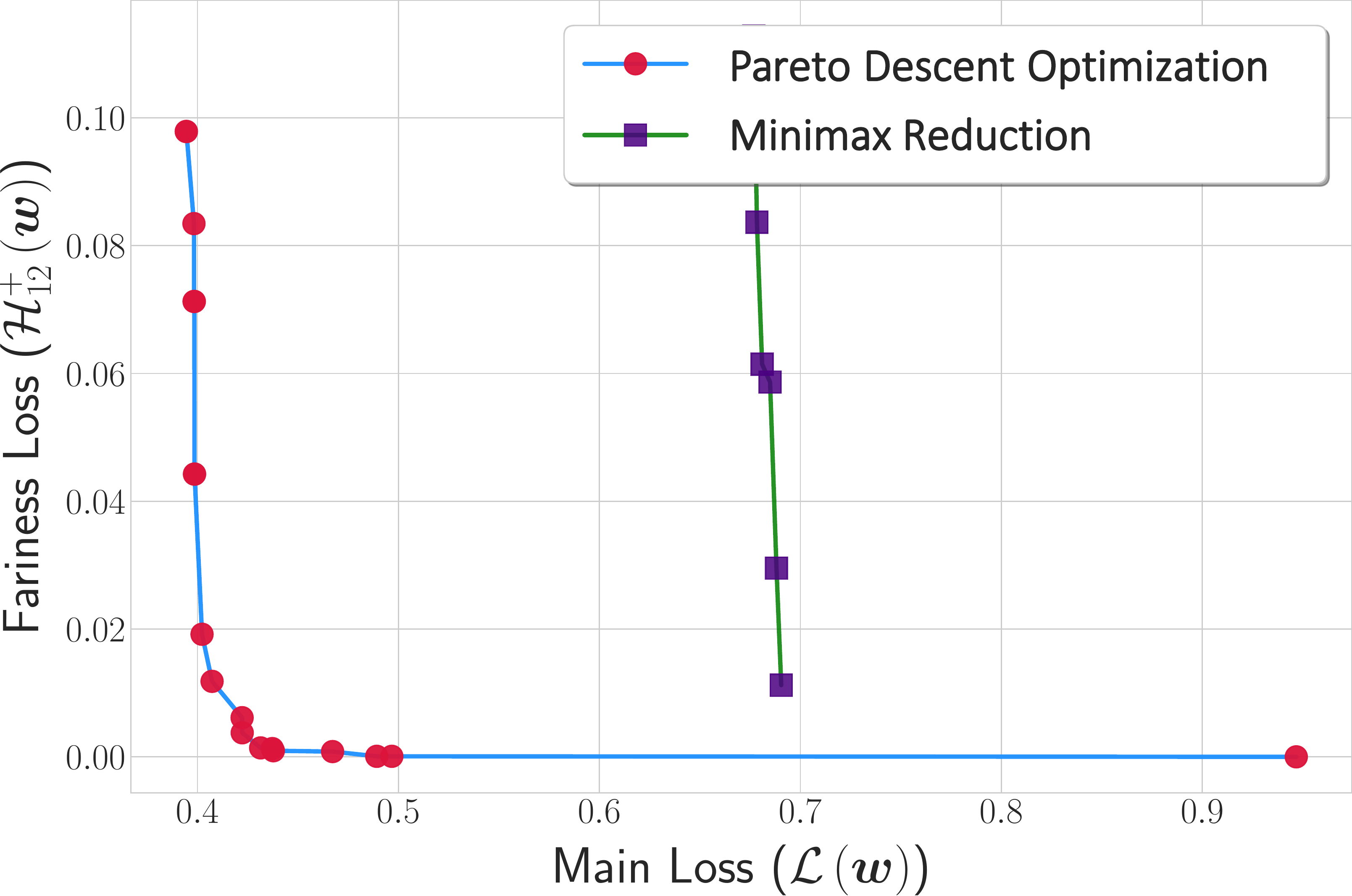}
		\label{fig:pf_loss_compas_1_SVM}
		}
		\hfill
		\subfigure{
			\centering 
			\includegraphics[width=0.285\textwidth]{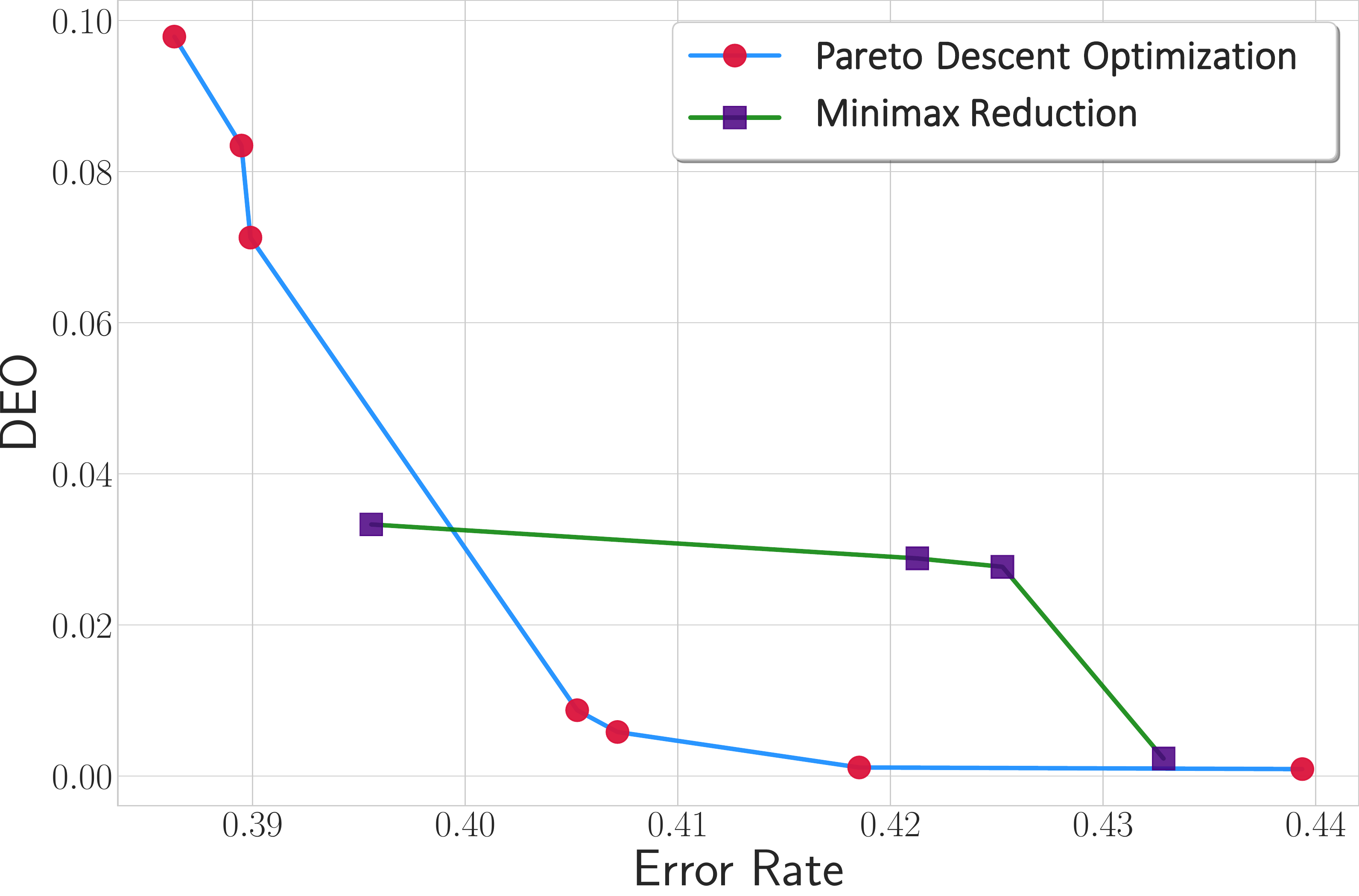}
			\label{fig:pf_err_train_compas_1_SVM}
			}
			\hfill
	\subfigure{  
		\centering 
		\includegraphics[width=0.285\textwidth]{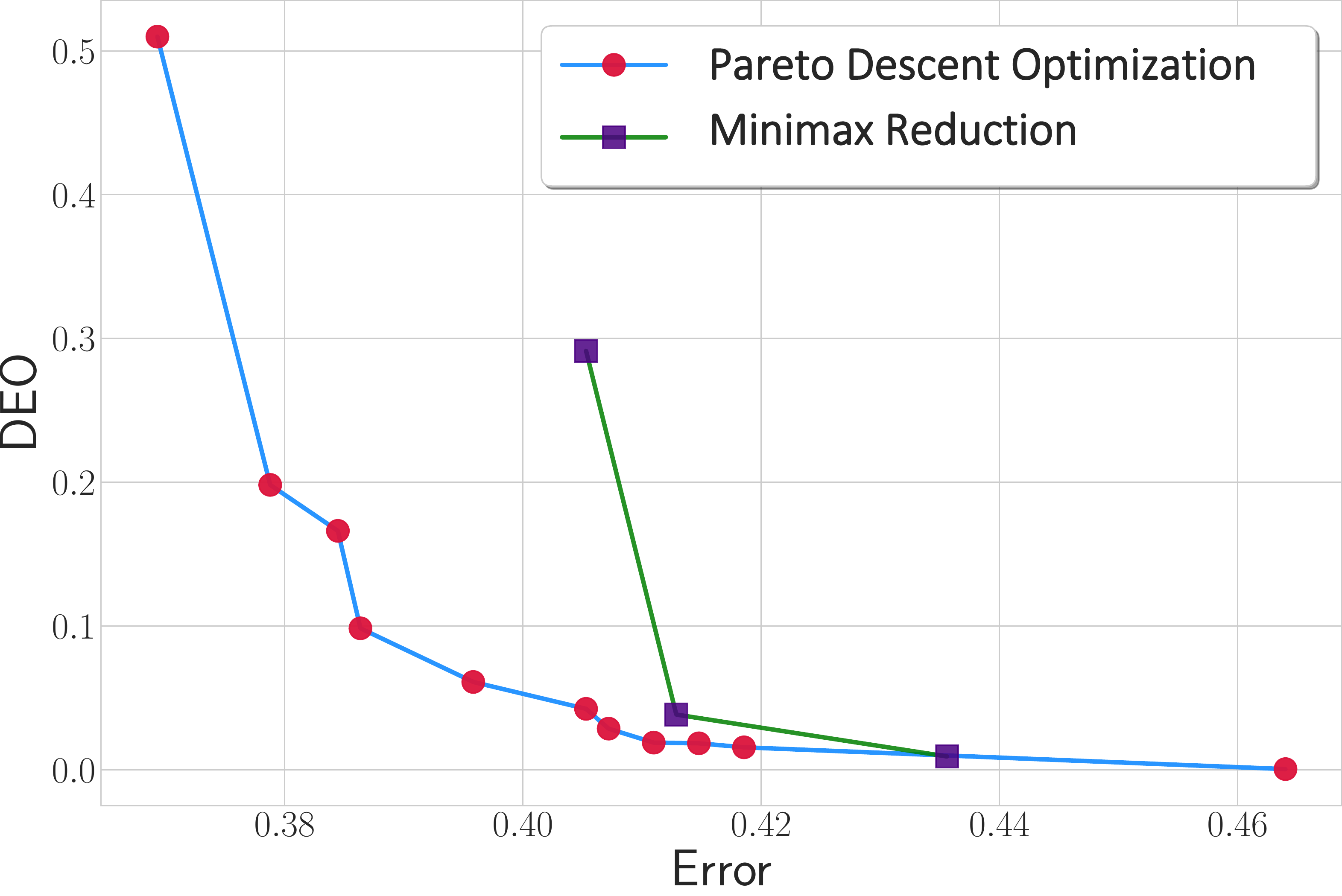}
		\label{fig:pf_err_test_compas_1_SVM}
		}
		
	\setcounter{subfigure}{0}
	\subfigure[Loss-Loss trade-off]{
		\centering
		\makebox[20pt]{\raisebox{40pt}{\rotatebox[origin=c]{90}{COMPAS (race)}}}\hspace{3pt}
		\includegraphics[width=0.285\textwidth]{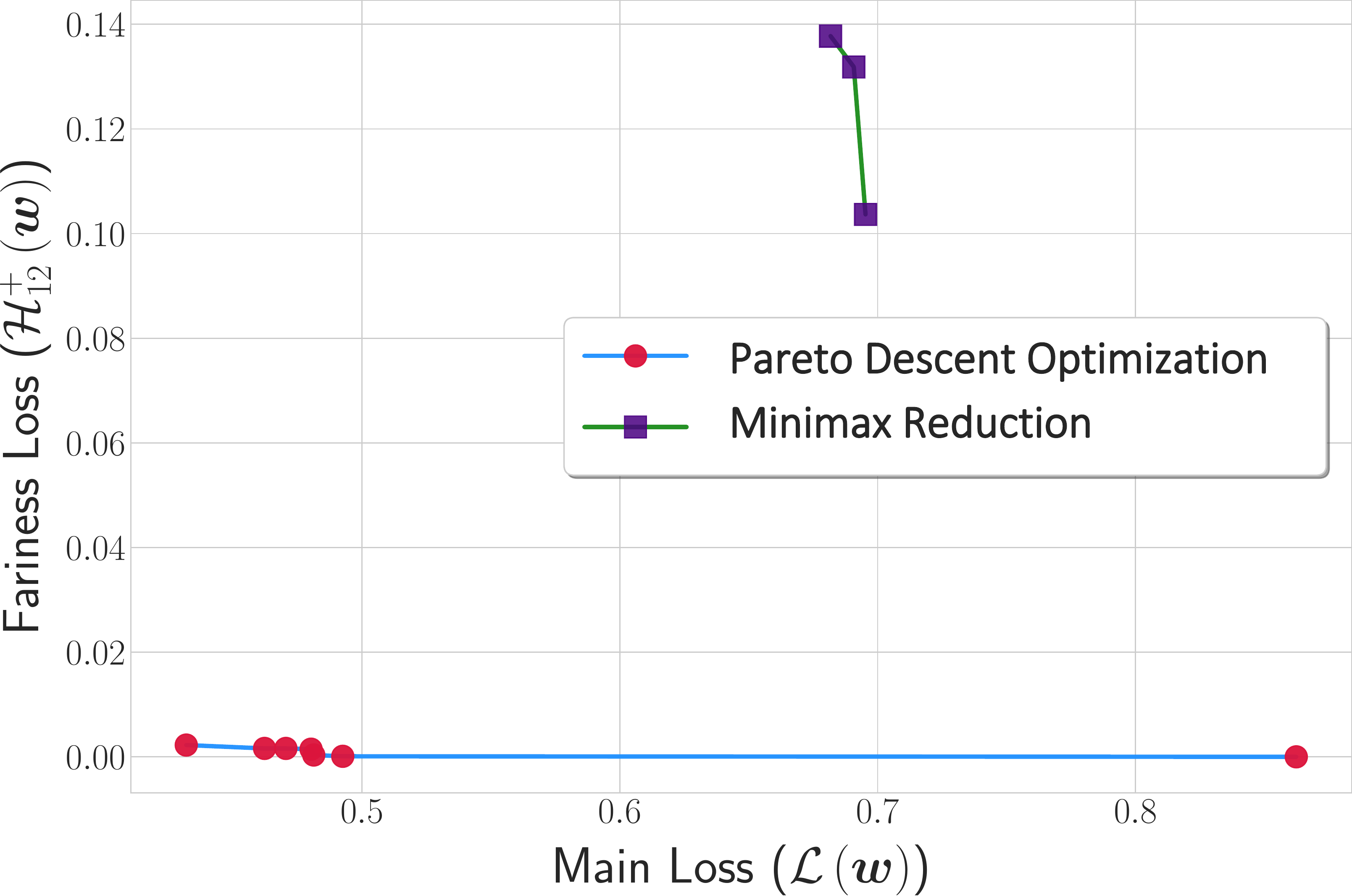}
		\label{fig:pf_loss_compas_2_SVM}
	}
	\hfill
	\subfigure[Error-DEO trade-off (train)]{  
			\centering 
			\includegraphics[width=0.285\textwidth]{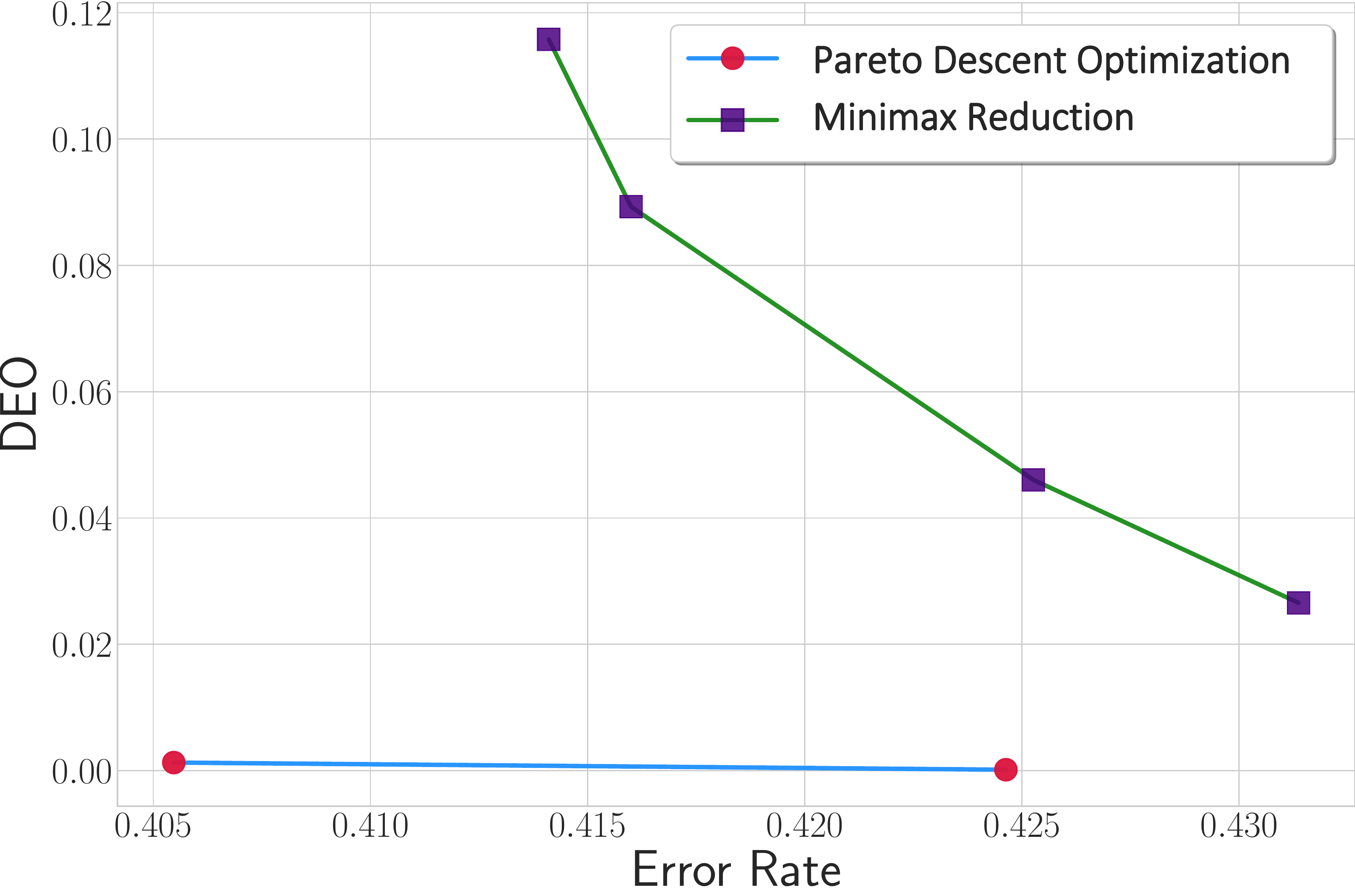}
		\label{fig:pf_err_compas_2_SVM}
	}
	\hfill
	\subfigure[Error-DEO trade-off (test)]{   
		\centering 
		\includegraphics[width=0.285\textwidth]{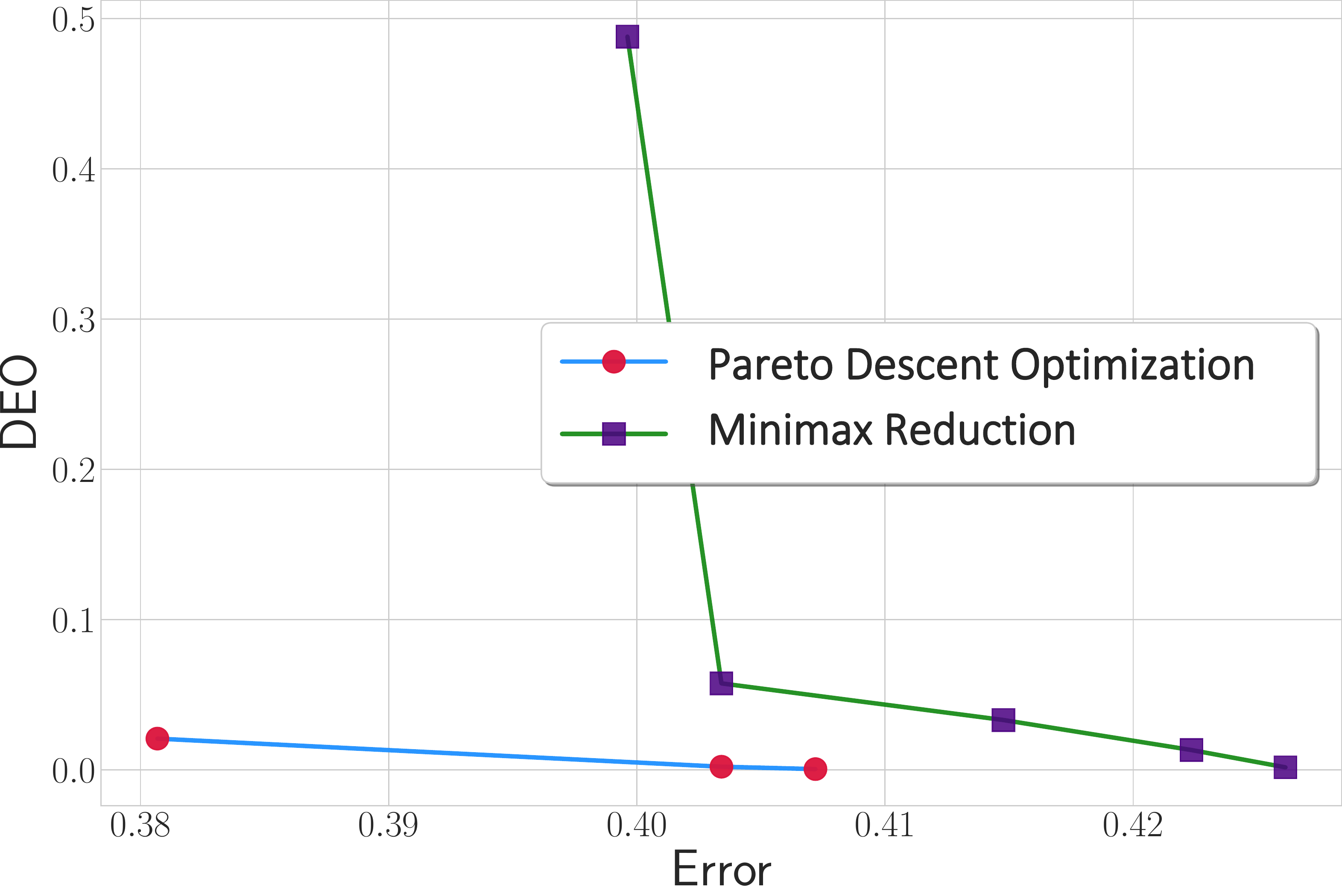}
		\label{fig:pf_err_test_compas_2_SVM}
		}
	\caption[]{Comparing the Pareto frontier extracted by our proposed algorithm \texttt{PB-PDO} and the the minimax reduction algorithm introduced in~\cite{agarwal2018reductions} using \textbf{linear SVM} as the loss function. We apply these algorithms on the Adult dataset with gender and the COMPAS dataset with sex and race as their sensitive features. The first column is the Pareto frontier on the trade-off between the main loss and the fairness loss (equality of opportunity as in~(\ref{eq:eo_obj})). We map this Pareto frontier to the error vs. DEO trade-off for both training and test datasets on the second and the third columns, respectively. Our proposed algorithm outperforms and dominates almost all the solutions found by the minimax reduction approach.
	}
	\label{fig:pf-data_SVM}
\end{figure*}

\begin{figure*}[t!]
	\centering
	\subfigure{
		\centering
		\makebox[20pt]{\raisebox{40pt}{\rotatebox[origin=c]{90}{Adult (gender)}}}\hspace{3pt}
		\includegraphics[width=0.285\textwidth]{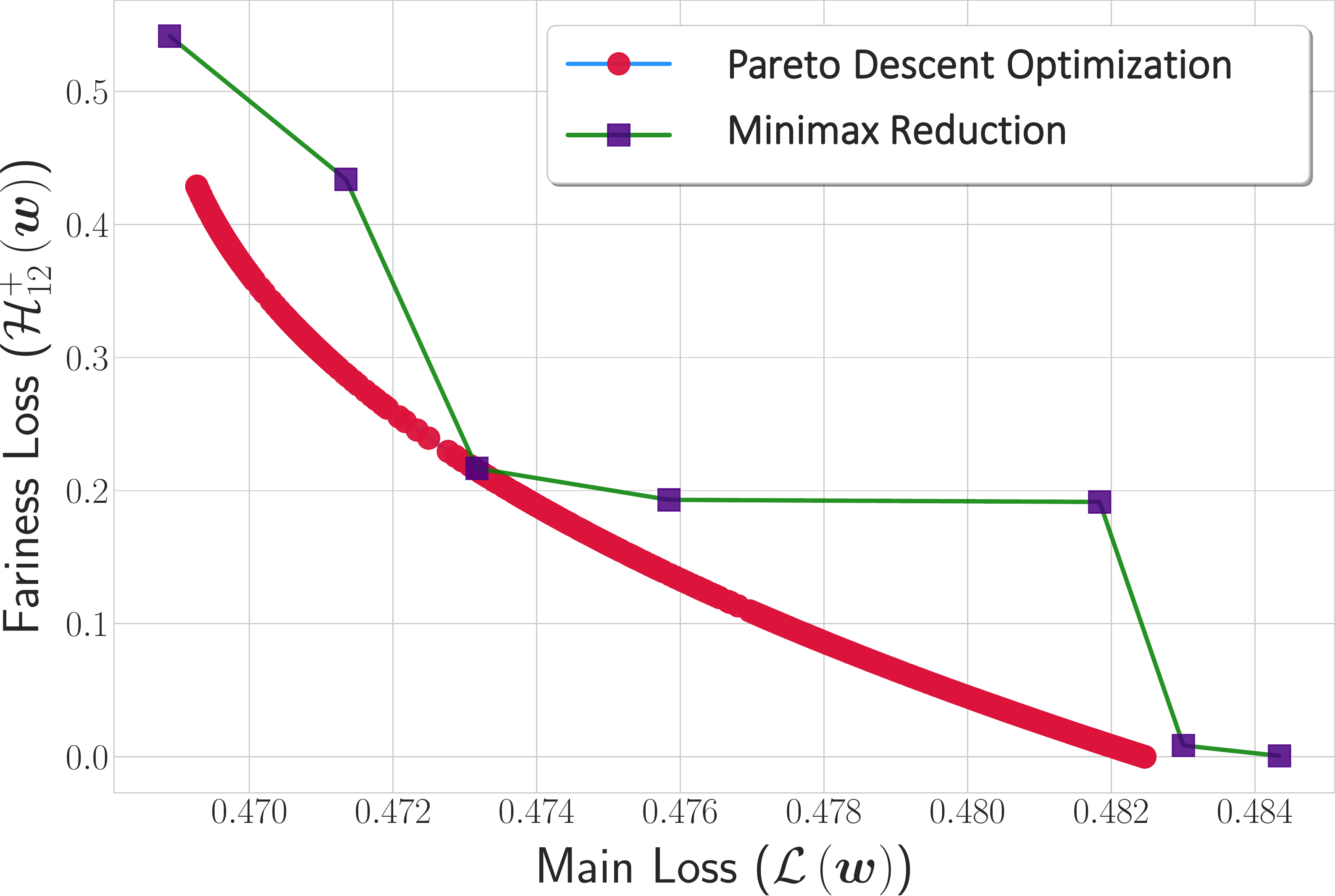}
		\label{fig:pf_loss_adult}
		}
		\hfill
		\subfigure{
			\centering 
			\includegraphics[width=0.285\textwidth]{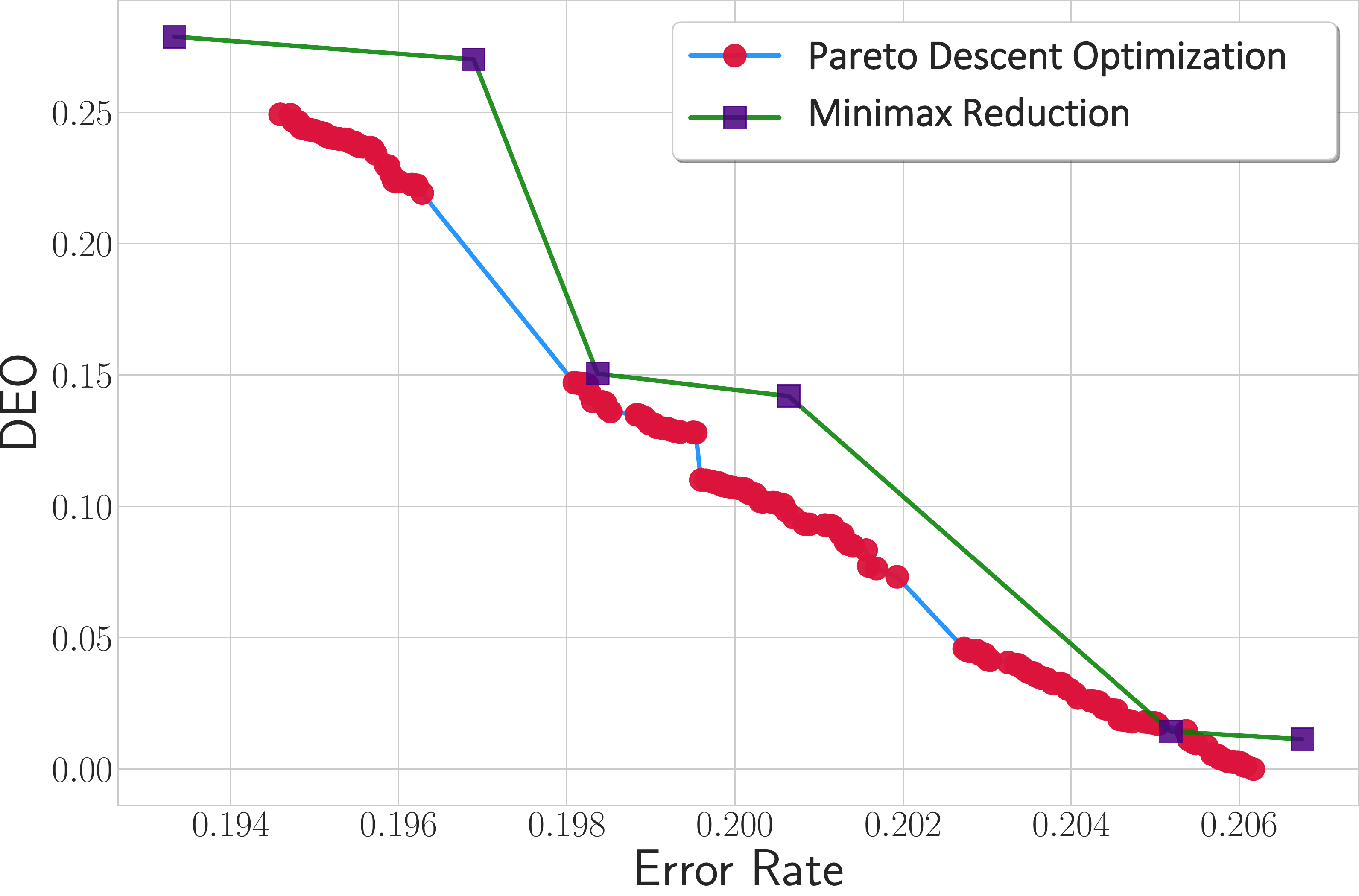}
			\label{fig:pf_err_train_adult}
			}
			\hfill
	\subfigure{  
		\centering 
		\includegraphics[width=0.285\textwidth]{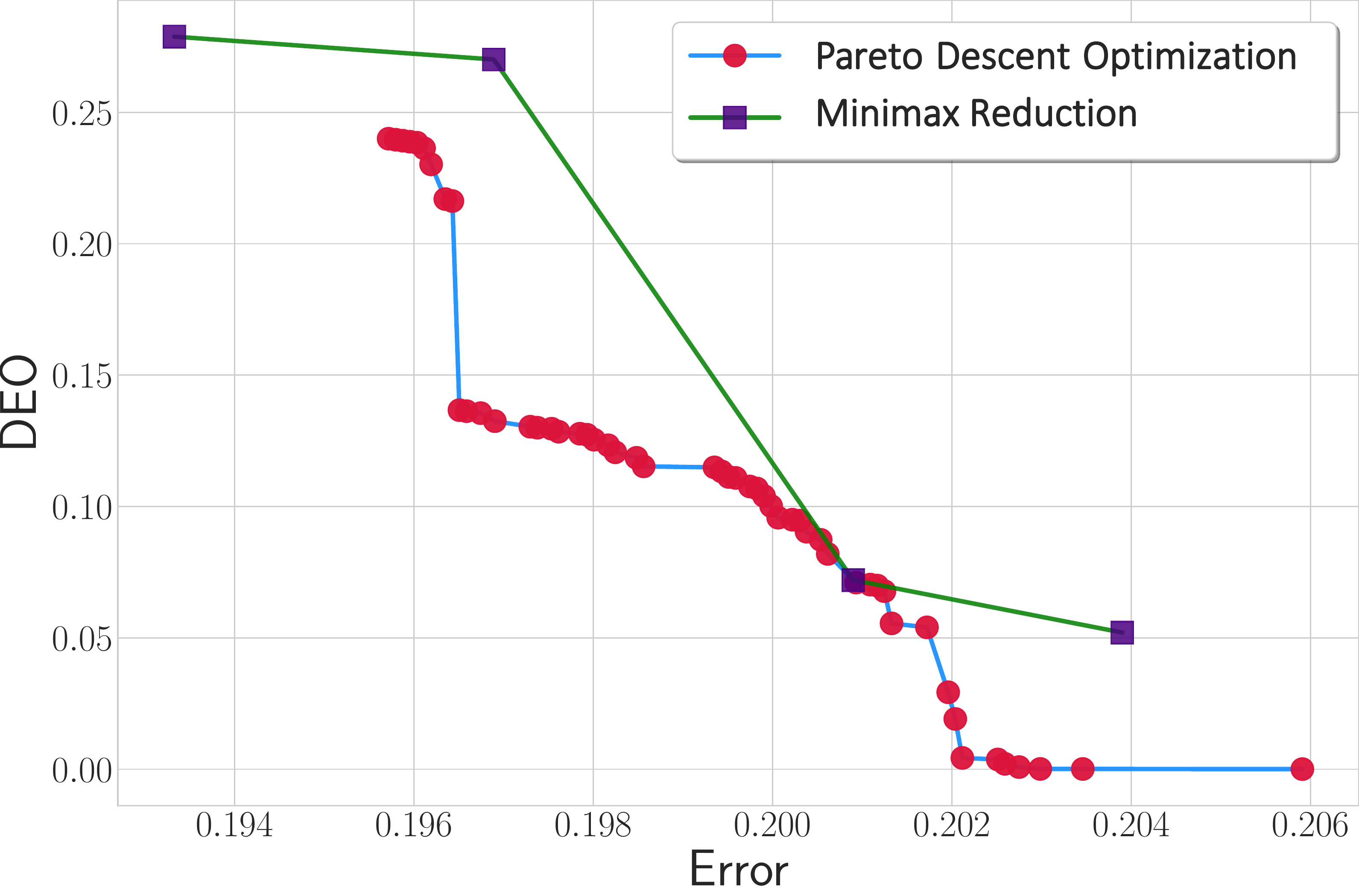}
		\label{fig:pf_err_test_adult}
		}
	
	\subfigure{
		\centering
		\makebox[20pt]{\raisebox{40pt}{\rotatebox[origin=c]{90}{COMPAS (sex)}}}\hspace{3pt}
		\includegraphics[width=0.285\textwidth]{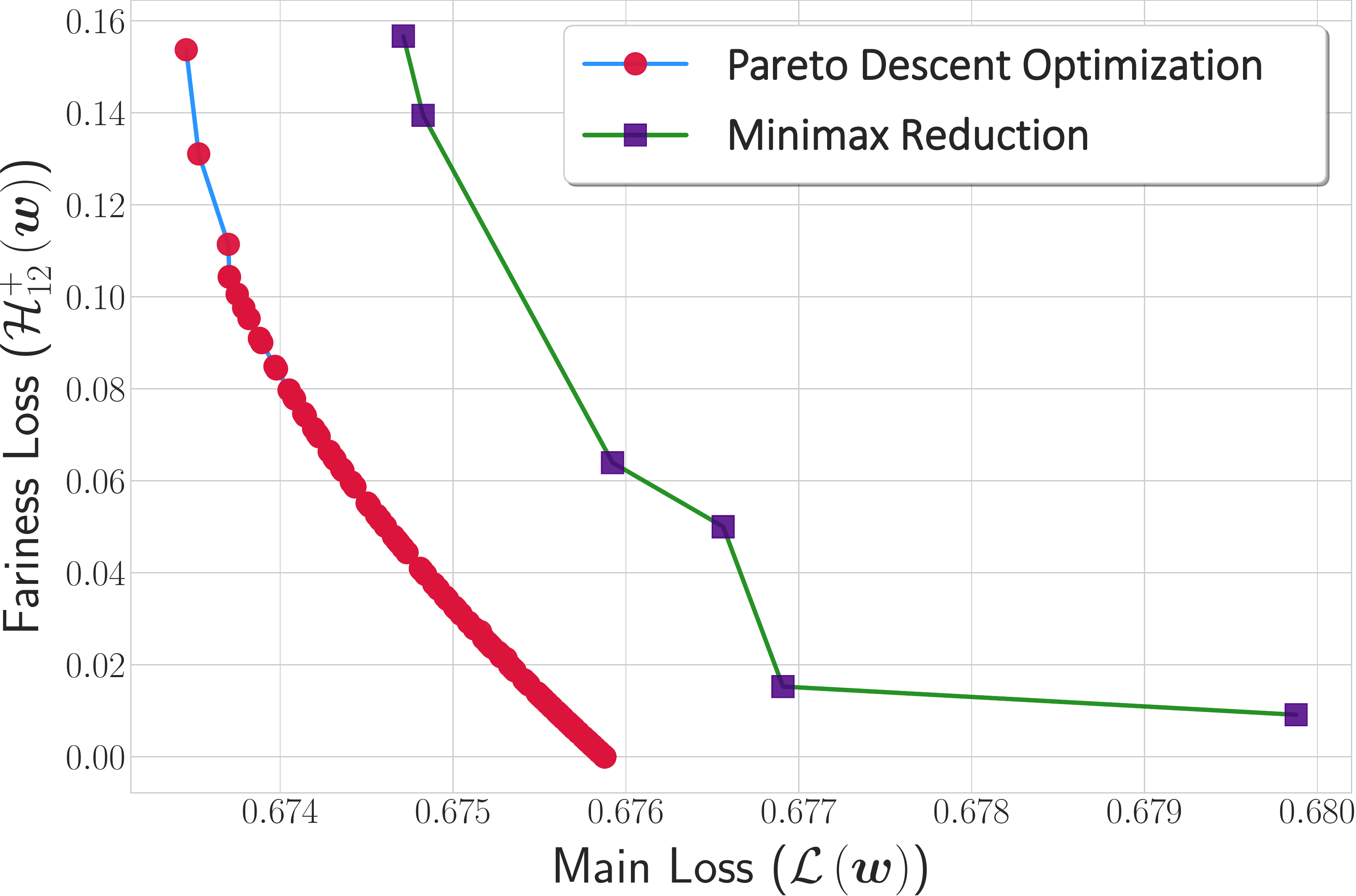}
		\label{fig:pf_loss_compas_1}
		}
		\hfill
		\subfigure{
			\centering 
			\includegraphics[width=0.285\textwidth]{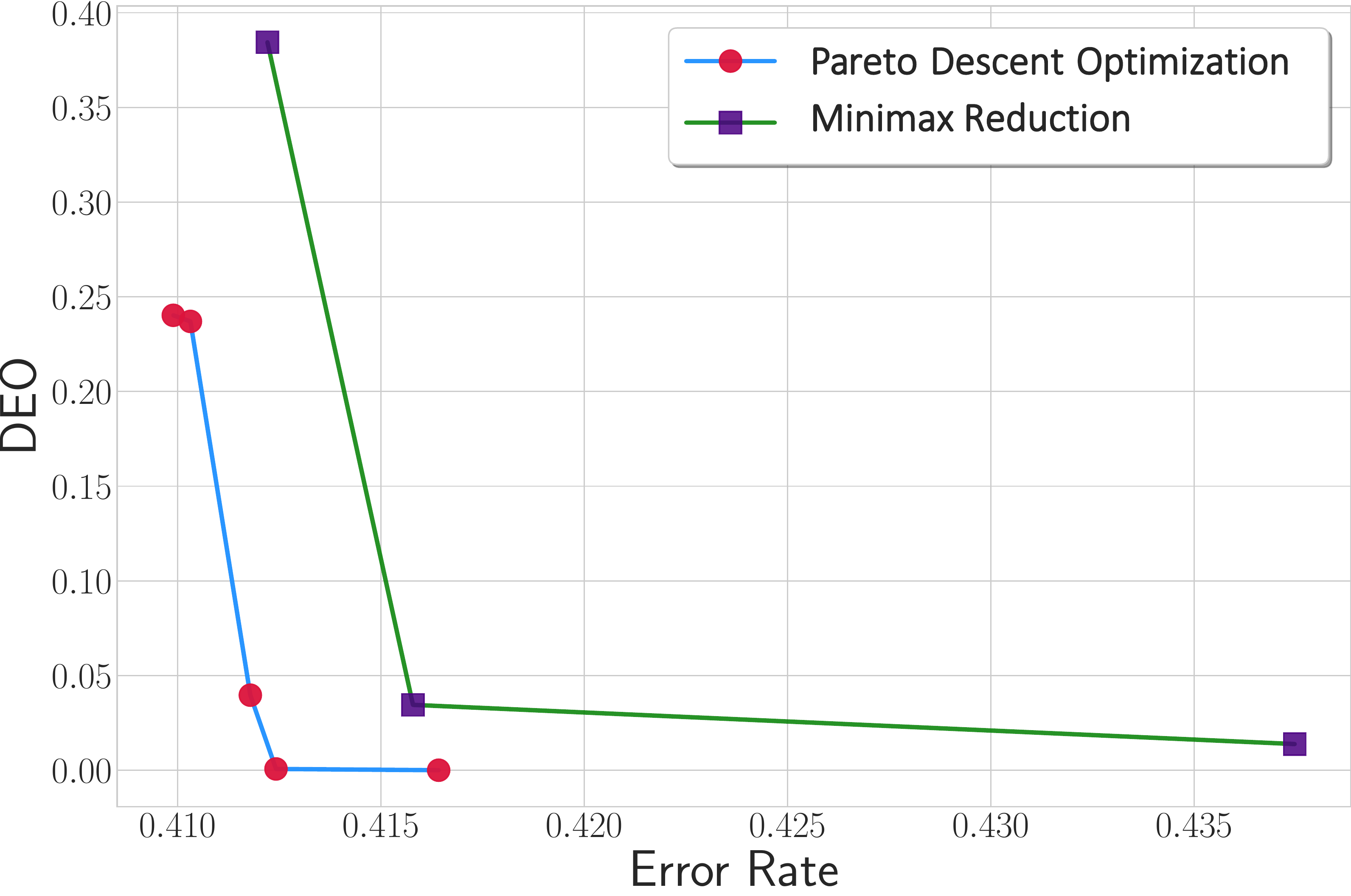}
			\label{fig:pf_err_train_compas_1}
			}
			\hfill
	\subfigure{  
		\centering 
		\includegraphics[width=0.285\textwidth]{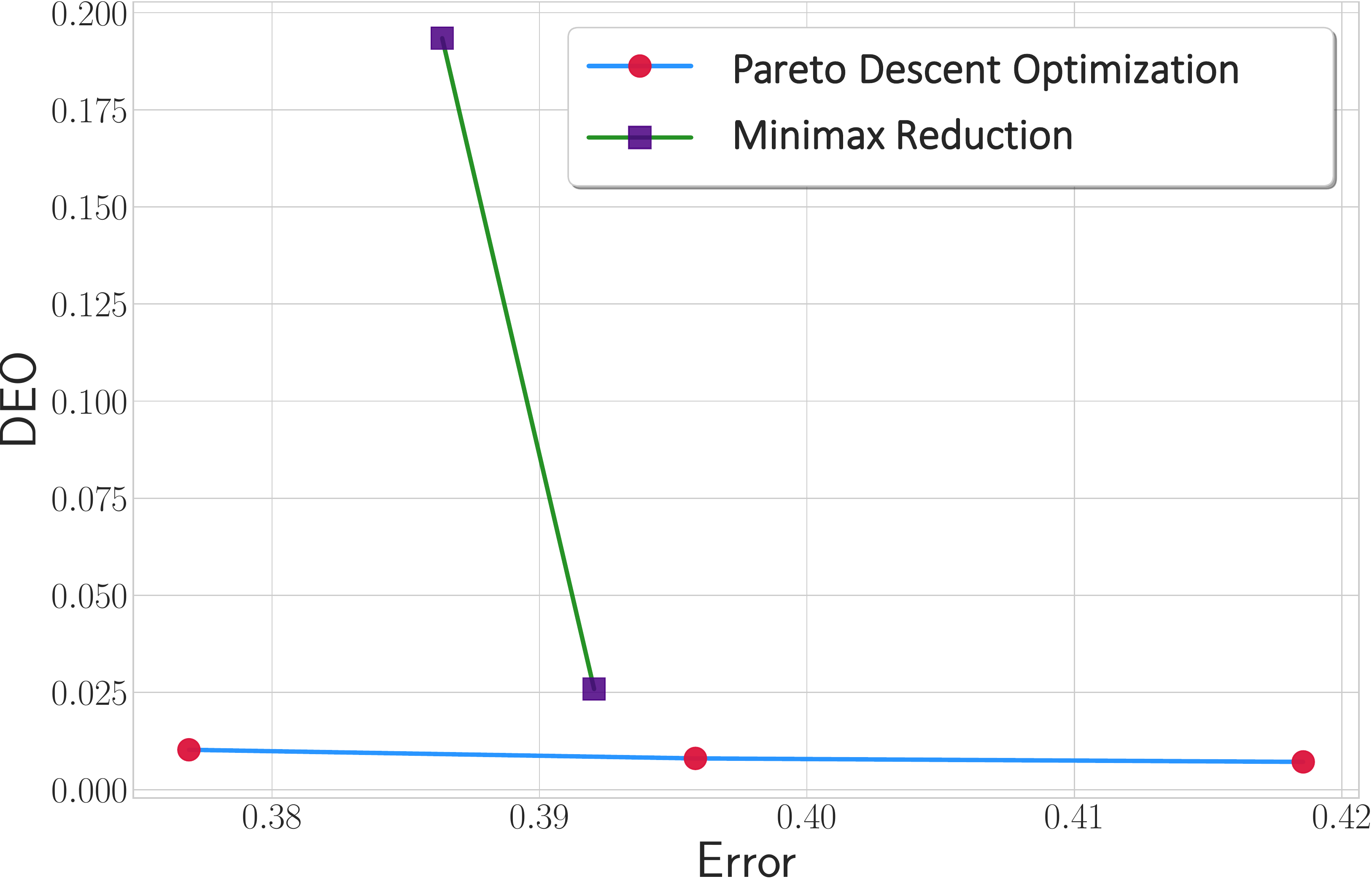}
		\label{fig:pf_err_test_compas_1}
		}
		
	\setcounter{subfigure}{0}
	\subfigure[Loss-Loss trade-off]{
		\centering
		\makebox[20pt]{\raisebox{40pt}{\rotatebox[origin=c]{90}{COMPAS (race)}}}\hspace{3pt}
		\includegraphics[width=0.285\textwidth]{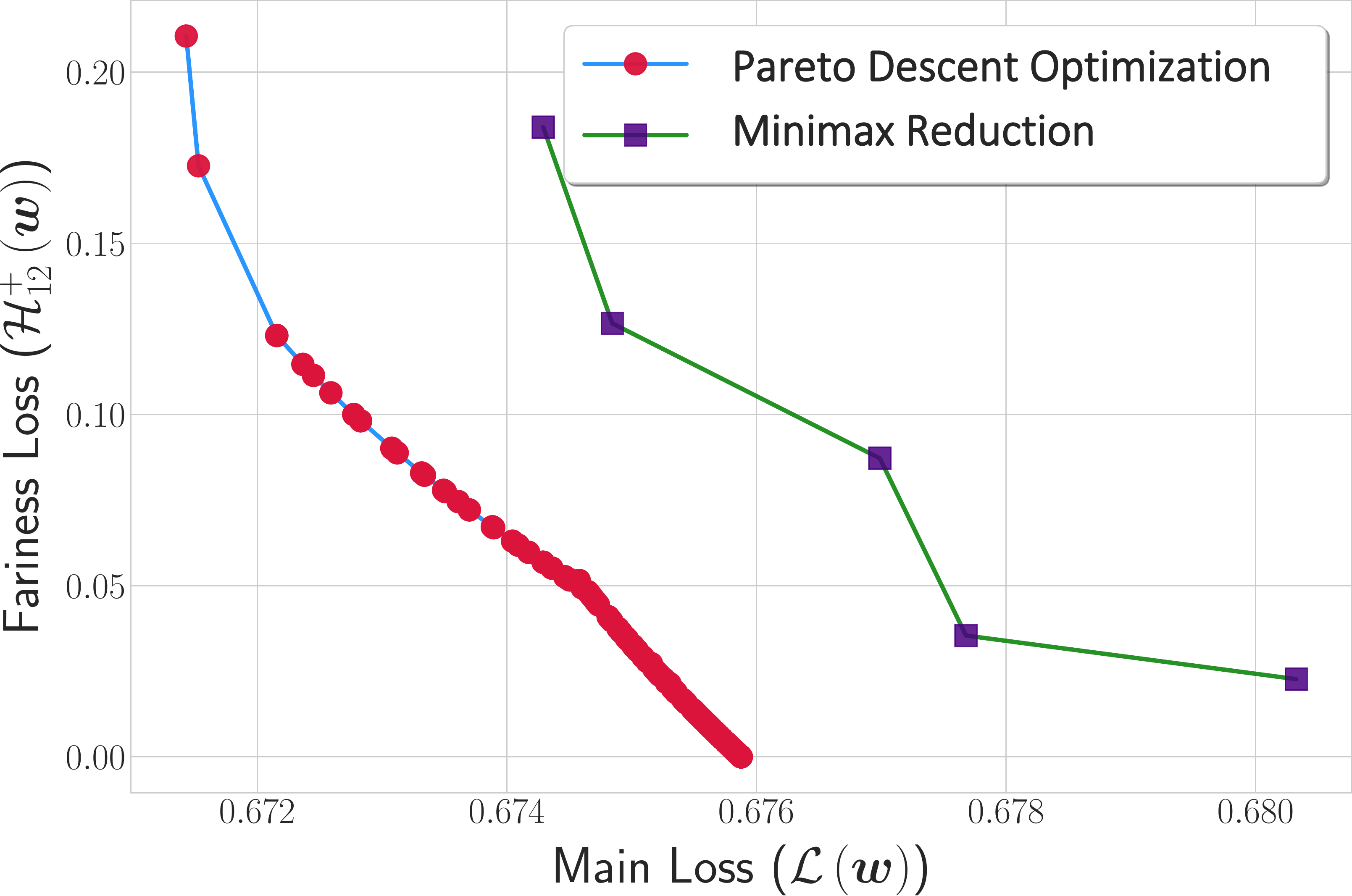}
		\label{fig:pf_loss_compas_2}
	}
	\hfill
	\subfigure[Error-DEO trade-off (train)]{  
			\centering 
			\includegraphics[width=0.285\textwidth]{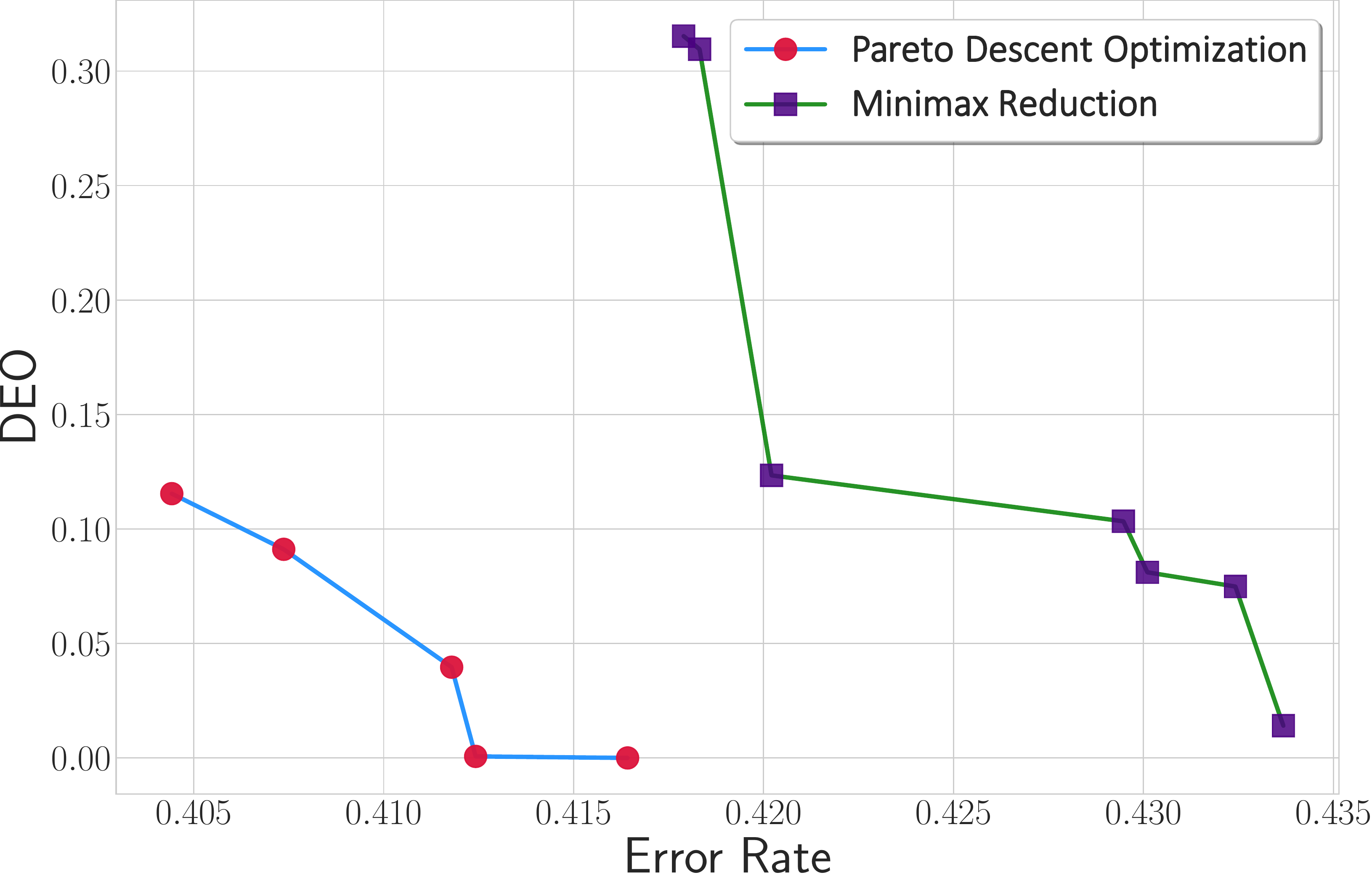}
		\label{fig:pf_err_compas_2}
	}
	\hfill
	\subfigure[Error-DEO trade-off (test)]{   
		\centering 
		\includegraphics[width=0.285\textwidth]{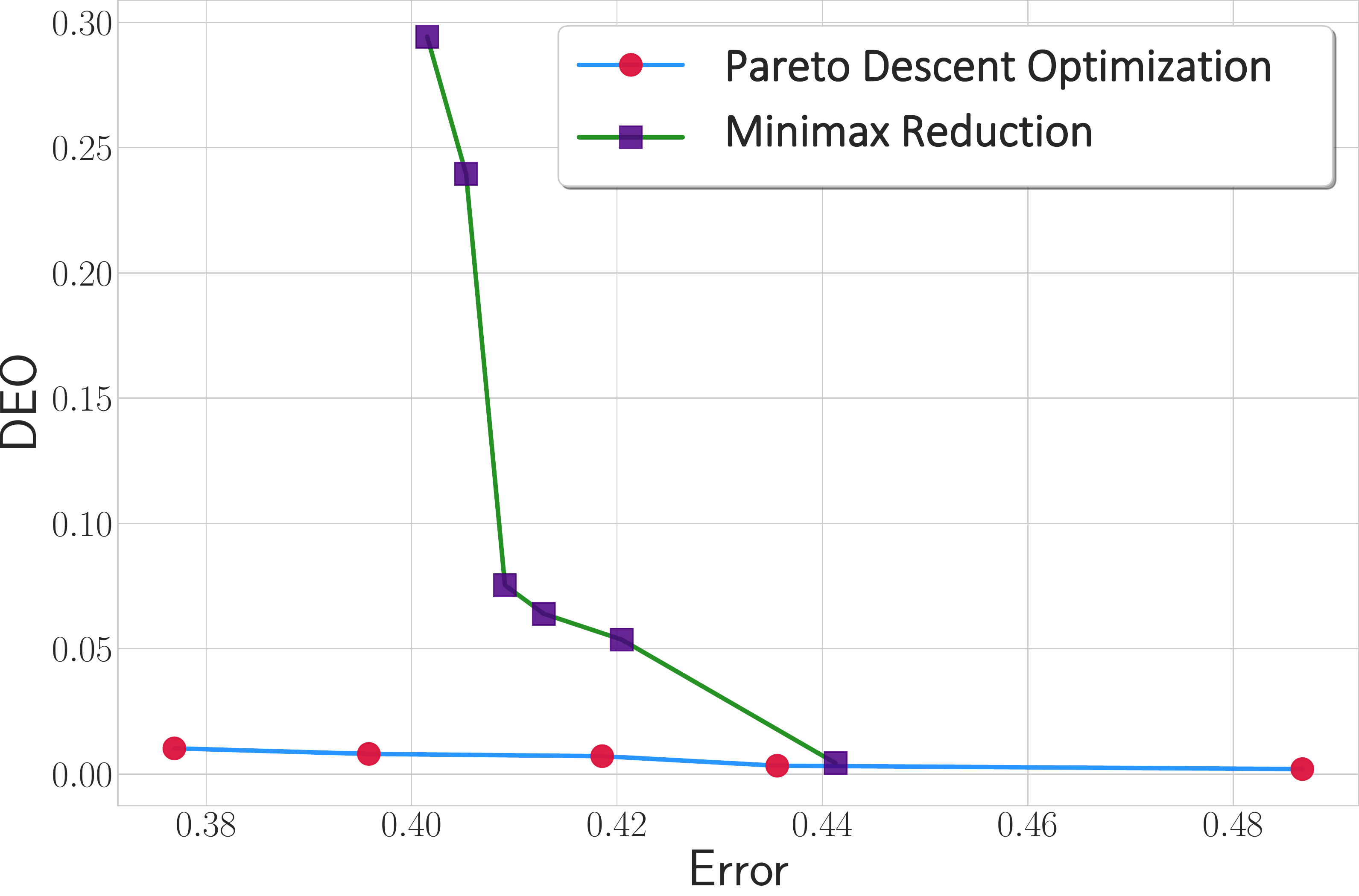}
		\label{fig:pf_err_test_compas_2}
		}
	\caption[]{Comparing the Pareto frontier extracted by our proposed algorithm \texttt{PB-PDO} and the minimax reduction algorithm introduced in~\cite{agarwal2018reductions} using \textbf{logistic regression} as the loss function. We apply these algorithms on the Adult dataset with gender and the COMPAS dataset with sex and race as their sensitive features. The first column is the Pareto frontier on the trade-off between the main loss and the fairness loss (equality of opportunity as in~(\ref{eq:eo_obj})). We map this Pareto frontier to the error vs. DEO trade-off for both training and test datasets on the second and the third columns, respectively. Our proposed algorithm outperforms and dominates almost all the solutions found by the minimax reduction approach.
	}
	\label{fig:pf-data}
\end{figure*}

\section{Additional Related Work}\label{sec:related}
Achieving a ``fair'' classifier is becoming the main propulsion of an expanding line of research, trying to define an all-inclusive definition for fairness and quantify it in the learning task. A notable attempt in this field by~\citet{hardt2016equality} introduced Equalized Odds and Equality of Opportunity in the context of supervised learning. However, there is not a consensus on the definition of fairness in algorithmic decision-making systems. Besides, there has been asserted that some of the current well-known definitions of fairness cannot be satisfied simultaneously~\citep{kleinberg2016inherent}. Nevertheless, the efforts in this field, in general, can be categorized as either mitigating the effect of biases on protected groups or individuals in current decision-making systems or modeling the long-term effect of the former methods~\citep{liu2018delayed,kannan2019downstream,liu2019disparate}. In this research, we are dealing with the former category. Being rife with various notions and viewpoints of fairness, we summarize some key research in this field, related to our framework as follows.
\paragraph{\textbf{Individual Fairness.}}
One worldview of the fairness problem is that similar individuals in the input features' space, should have similar outputs in the decision space~\citep{friedler2016possibility}. Following this rule, \citet{dwork2012fairness} introduced the individual notion of fairness by using task-specific similarity metrics that can be enforced like a Lipschitz continuity constraint to the classifier. However, \citet{friedler2016possibility} introduces another worldview, called structural bias, in which it applies to many real-world societal applications. Most of the cases, the bias, culturally or historically, has affected a group of individuals, all having a similar set of characteristics. Hence, using a similarity metric to assure fairness in individual level would be a simplistic approach that cannot be effective when such a structural bias presents. That is why a great body of literature revolves around the problem of group fairness.
\paragraph{Group Fairness.} The studies in the field of group fairness can be categorized in one of the three following strategies: The first approach is to pre-process data, either by modifying the sensitive features or map the data to a new space~\citep{donini2018empirical,dwork2012fairness,feldman2015certifying}. These methods are prone to fail because they are not accuracy-driven and oblivious to the training procedure time. The second approach tries to modify the existing classifiers to impose the fairness constraints~\citep{pleiss2017fairness,goh2016satisfying,calders2010three,calders2009building}. The third approach, to which our proposed framework belongs, tries to integrate the fairness constraints in the training~\citep{zafar2017parity,woodworth2017learning,zafar2017fairness,menon2018cost,kamishima2011fairness,agarwal2018reductions,zafar2015fairness}. Despite their success, they are mostly bound to binary classification or binary sensitive groups, which makes them impractical. Moreover, the efficiency of their solution is left as an open problem in~\citet{zafar2017parity}. We propose a unified approach to address the efficiency of the solution, as well as, generalization to complex multiple group and multi-label classification tasks.
\paragraph{\textbf{Fairness-Accuracy Trade-offs.}}
The trade-offs between accuracy, as the main learning objective, and different fairness measures have been investigated widely in different studies. In~\citet{blum2019recovering}, authors introduce two models for the bias source in data, namely, under-representation and labeling bias. They argue that the ERM with equality of opportunity constraint under these forms of biases can achieve the Bayes optimal classifier. 
Although~\citet{hu2020fair} suggests that minimizing ERM constrained to fairness objectives might not satisfy the Pareto principle for group welfare, but the group welfare was not in the objective of the optimization problem in the first place. Hence, that is possible that a defined group welfare might not be satisfied with this optimization unless its objective is added to the optimization. In~\citet{wick2019unlocking}, authors discuss that fairness objectives might be in accord with the main learning objective in a semi-supervised learning scenario, however, this is not generalizable to a supervised setting. Recently in~\cite{kim2020model}, using confusion tensor of different notions of fairness in a binary classification, authors show the trade-offs between the predictive performance of a model and fairness objectives, as well as between fairness objectives themselves. The trade-off between accuracy and equality of opportunity is investigated in~\citet{duttathere}, where they showed that because of noise on the underrepresented groups this trade-off exists. Then, they provide an algorithmic solution to find an ideal distribution where accuracy and equality of opportunity are in accord.

\paragraph{\textbf{Incompatibility of Notions.}}Introducing numerous definitions and notions for fairness problem raise this critical question of which one is the most inclusive. Despite the quest for finding the best notion for fairness, \citet{kleinberg2016inherent} and \citet{chouldechova2017fair} discuss that some of the mainstream and widely used notions of fairness are incompatible; unless there are some unrealistic assumptions about the classifier or data. Even if the notions are compatible, there might be trade-offs between different notions of fairness that make it impossible to have them all satisfied at once, in addition to the trade-offs between these notions and accuracy.

\paragraph{\textbf{Pareto Fairness.}}
Dealing with the aforementioned trade-off between any fairness notion and the main learning objective, a conspicuous solution is to find the point that has the most desirable compromise between those objectives. Such a point is called a Pareto efficient point, named after Italian economist Vilfredo Pareto~\citep{miettinen2012nonlinear}. From the outset of this awareness on fairness measure in decision-making algorithms, finding a Pareto efficient point was an indisputable goal~\citep{zafar2017parity,hu2018short,balashankar2019fair}. However, there has not been a unified framework to define the Pareto efficiency in the context of fairness problems, nor how to achieve it in general. \citet{kearns2019ethical} in a chapter named "from Parity to Pareto" discuss why it is necessary to achieve a Pareto efficient point in this setting, rather than statistical parity. Perhaps the closest one to our proposal in finding the Pareto efficient solutions for fairness problems is~\cite{balashankar2019fair}. However, their objectives are completely different from ours, where they consider two objectives and want to satisfy a balanced accuracy for each group with respect to their best accuracy they can achieve alone. Hence, they are finding a Pareto efficient point, totally different from ours, by introducing a new notion for fairness. This measure, a.k.a. Chebychev measure, has been studied for unsupervised methods~\citep{kamani2019efficient,samadi2018price}, but it is not common in supervised ones, where the target values are available. On the other hand, our proposal is not introducing a new fairness notion, yet it can be applied to any existing fairness measure. Moreover, their proposed method is for an unbalanced skewed dataset with respect to different sensitive groups, which is only one form of biases in the fairness domain~\citep{blum2019recovering}, while our approach is agnostic to the bias source. Finally, they do not provide any convergence analysis for their algorithm, which makes their approach more heuristic. 
\paragraph{Bilevel Optimization.}
In this article, we lay the groundwork for Pareto efficiency in the fairness domain, and propose a proper multi-objective optimization. To solve it and guarantee the convergence to a Pareto efficient point, we will cast the problem as a bilevel programming. Bilevel programming is a well-known optimization framework that has two levels, inner and outer. In this setting, a solution of the inner problem is used to solve the outer problem. Bilevel optimization is based on a renowned two-player game, called the Stackelberg game~\citep{basar1999dynamic}. In this game, two players are called leader and follower, and they both want to minimize their specific objective functions. Recently, there has been a surge in the applications of this optimization problem, such as in hyperparameter optimization~\citep{franceschi2018bilevel,okuno2018hyperparameter}, class imbalance problem~\citep{kamani2020targeted,kamani2020multiobjective,kamani2018skeleton,kamani2016shape}, or modeling different meta-learning approaches~\citep{kamani2020targeted,kamani2019targeted,fallah2019convergence,rajeswaran2019meta}.

\section{Conclusions \& Future Work}\label{sec:conc}
This paper advocates the notion of Pareto efficient fairness  in dealing with fairness problems, and to achieve the optimal trade-offs between accuracy and other fairness criteria. By casting the fairness problem as a multi-objective optimization task and introducing Pareto descent optimizer, we can efficiently surpass existing methods for satisfying fairness criteria without significantly degrading the accuracy. We showed that other notions of fairness can be reduced to the notion of Pareto fairness effortlessly, making the Pareto descent algorithm applicable to a range of fairness problems. Moreover, using our proposed framework finding PEF solutions from desired parts of the Pareto frontier of the problem is straightforward.

Besides, this paper leaves a few interesting directions worthy of exploration. The proposed algorithm is based on a gradient descent method. While the stochastic version is evident, a thorough convergence analysis for such methods is required. Also, it is interesting to understand the feasibility of different fairness criteria in the context of Pareto efficiency fairness, and how they are affecting the overall loss.

\clearpage
\bibliographystyle{plainnat}
\bibliography{references}

\begin{thebibliography}{69}
\providecommand{\natexlab}[1]{#1}
\providecommand{\url}[1]{\texttt{#1}}
\expandafter\ifx\csname urlstyle\endcsname\relax
  \providecommand{\doi}[1]{doi: #1}\else
  \providecommand{\doi}{doi: \begingroup \urlstyle{rm}\Url}\fi

\bibitem[Agarwal et~al.(2018)Agarwal, Beygelzimer, Dudik, Langford, and
  Wallach]{agarwal2018reductions}
Alekh Agarwal, Alina Beygelzimer, Miroslav Dudik, John Langford, and Hanna
  Wallach.
\newblock A reductions approach to fair classification.
\newblock In \emph{International Conference on Machine Learning}, pages 60--69,
  2018.

\bibitem[Angwin et~al.(2016)Angwin, Larson, Mattu, and Kirchner]{propublica}
Julia Angwin, Jeff Larson, Surya Mattu, and Lauren Kirchner.
\newblock Machine bias. propublica.
\newblock \url{https://github.com/propublica/compas-analysis}, May 2016.

\bibitem[Balashankar et~al.(2019)Balashankar, Lees, Welty, and
  Subramanian]{balashankar2019fair}
Ananth Balashankar, Alyssa Lees, Chris Welty, and Lakshminarayanan Subramanian.
\newblock What is fair? exploring {Pareto}-efficiency for fairness constrained
  classifiers.
\newblock \emph{arXiv preprint arXiv:1910.14120}, 2019.

\bibitem[Barocas et~al.(2017)Barocas, Hardt, and
  Narayanan]{barocas2017fairness}
Solon Barocas, Moritz Hardt, and Arvind Narayanan.
\newblock Fairness in machine learning.
\newblock \emph{NIPS Tutorial}, 2017.

\bibitem[Basar and Olsder(1999)]{basar1999dynamic}
Tamer Basar and Geert~Jan Olsder.
\newblock \emph{Dynamic Noncooperative Game Theory}, volume~23.
\newblock SIAM, 1999.

\bibitem[Blum and Stangl(2020)]{blum2019recovering}
Avrim Blum and Kevin Stangl.
\newblock Recovering from biased data: Can fairness constraints improve
  accuracy?
\newblock In \emph{Symposium on Foundations of Responsible Computing (FORC)},
  volume~1, 2020.

\bibitem[Bogen and Rieke(2018)]{bogen2018help}
Miranda Bogen and Aaron Rieke.
\newblock Help wanted: an examination of hiring algorithms.
\newblock \emph{Equity, and Bias, Upturn (December 2018)}, 2018.

\bibitem[Boyd and Vandenberghe(2004)]{boyd2004convex}
Stephen Boyd and Lieven Vandenberghe.
\newblock \emph{Convex optimization}.
\newblock Cambridge university press, 2004.

\bibitem[Bubeck et~al.(2015)]{bubeck2015convex}
S{\'e}bastien Bubeck et~al.
\newblock Convex optimization: Algorithms and complexity.
\newblock \emph{Foundations and Trends{\textregistered} in Machine Learning},
  8\penalty0 (3-4):\penalty0 231--357, 2015.

\bibitem[Calders and Verwer(2010)]{calders2010three}
Toon Calders and Sicco Verwer.
\newblock Three naive {Bayes} approaches for discrimination-free
  classification.
\newblock \emph{Data Mining and Knowledge Discovery}, 21\penalty0 (2):\penalty0
  277--292, 2010.

\bibitem[Calders et~al.(2009)Calders, Kamiran, and
  Pechenizkiy]{calders2009building}
Toon Calders, Faisal Kamiran, and Mykola Pechenizkiy.
\newblock Building classifiers with independency constraints.
\newblock In \emph{2009 IEEE International Conference on Data Mining
  Workshops}, pages 13--18. IEEE, 2009.

\bibitem[Chouldechova(2017)]{chouldechova2017fair}
Alexandra Chouldechova.
\newblock Fair prediction with disparate impact: A study of bias in recidivism
  prediction instruments.
\newblock \emph{Big Data}, 5\penalty0 (2):\penalty0 153--163, 2017.

\bibitem[Cortes et~al.(2020)Cortes, Mohri, Gonzalvo, and
  Storcheus]{cortes2020agnostic}
Corinna Cortes, Mehryar Mohri, Javier Gonzalvo, and Dmitry Storcheus.
\newblock Agnostic learning with multiple objectives.
\newblock \emph{Advances in Neural Information Processing Systems}, 33, 2020.

\bibitem[Das and Dennis(1997)]{das1997closer}
Indraneel Das and John~E Dennis.
\newblock A closer look at drawbacks of minimizing weighted sums of objectives
  for {Pareto} set generation in multicriteria optimization problems.
\newblock \emph{Structural optimization}, 14\penalty0 (1):\penalty0 63--69,
  1997.

\bibitem[Das and Dennis(1998)]{das1998normal}
Indraneel Das and John~E Dennis.
\newblock Normal-boundary intersection: A new method for generating the pareto
  surface in nonlinear multicriteria optimization problems.
\newblock \emph{SIAM journal on optimization}, 8\penalty0 (3):\penalty0
  631--657, 1998.

\bibitem[Donini et~al.(2018)Donini, Oneto, Ben-David, Shawe-Taylor, and
  Pontil]{donini2018empirical}
Michele Donini, Luca Oneto, Shai Ben-David, John~S Shawe-Taylor, and
  Massimiliano Pontil.
\newblock Empirical risk minimization under fairness constraints.
\newblock In \emph{Advances in Neural Information Processing Systems}, pages
  2791--2801, 2018.

\bibitem[Dutta et~al.(2020)Dutta, Wei, Yueksel, Chen, Liu, and
  Varshney]{duttathere}
Sanghamitra Dutta, Dennis Wei, Hazar Yueksel, Pin-Yu Chen, Sijia Liu, and
  Kush~R Varshney.
\newblock Is there a trade-off between fairness and accuracy? a perspective
  using mismatched hypothesis testing.
\newblock In \emph{International Conference on Machine Learning}, 2020.

\bibitem[Dwork and Ilvento(2018)]{dwork2018group}
Cynthia Dwork and Christina Ilvento.
\newblock Group fairness under composition.
\newblock In \emph{Proceedings of the 2018 Conference on Fairness,
  Accountability, and Transparency (FAT* 2018)}, 2018.

\bibitem[Dwork and Ilvento(2019)]{dwork2018fairness}
Cynthia Dwork and Christina Ilvento.
\newblock Fairness under composition.
\newblock \emph{10th Innovations in Theoretical Computer Science}, 2019.

\bibitem[Dwork et~al.(2012)Dwork, Hardt, Pitassi, Reingold, and
  Zemel]{dwork2012fairness}
Cynthia Dwork, Moritz Hardt, Toniann Pitassi, Omer Reingold, and Richard Zemel.
\newblock Fairness through awareness.
\newblock In \emph{Proceedings of the 3rd Innovations in Theoretical Computer
  Science Conference}, pages 214--226. ACM, 2012.

\bibitem[Fallah et~al.(2020)Fallah, Mokhtari, and
  Ozdaglar]{fallah2019convergence}
Alireza Fallah, Aryan Mokhtari, and Asuman Ozdaglar.
\newblock On the convergence theory of gradient-based model-agnostic
  meta-learning algorithms.
\newblock In \emph{International Conference on Artificial Intelligence and
  Statistics}, pages 1082--1092. PMLR, 2020.

\bibitem[Feldman et~al.(2015)Feldman, Friedler, Moeller, Scheidegger, and
  Venkatasubramanian]{feldman2015certifying}
Michael Feldman, Sorelle~A Friedler, John Moeller, Carlos Scheidegger, and
  Suresh Venkatasubramanian.
\newblock Certifying and removing disparate impact.
\newblock In \emph{Proceedings of the 21th ACM SIGKDD International Conference
  on Knowledge Discovery and Data Mining}, pages 259--268. ACM, 2015.

\bibitem[Franceschi et~al.(2018)Franceschi, Frasconi, Salzo, Grazzi, and
  Pontil]{franceschi2018bilevel}
Luca Franceschi, Paolo Frasconi, Saverio Salzo, Riccardo Grazzi, and
  Massimiliano Pontil.
\newblock Bilevel programming for hyperparameter optimization and
  meta-learning.
\newblock In \emph{International Conference on Machine Learning}, pages
  1568--1577. PMLR, 2018.

\bibitem[Friedler et~al.(2016)Friedler, Scheidegger, and
  Venkatasubramanian]{friedler2016possibility}
Sorelle~A Friedler, Carlos Scheidegger, and Suresh Venkatasubramanian.
\newblock On the (im)possibility of fairness.
\newblock \emph{arXiv preprint arXiv:1609.07236}, 2016.

\bibitem[Fukuda and Drummond(2014)]{fukuda2014survey}
Ellen~H Fukuda and Luis Mauricio~Gra{\~n}a Drummond.
\newblock A survey on multiobjective descent methods.
\newblock \emph{Pesquisa Operacional}, 34\penalty0 (3):\penalty0 585--620,
  2014.

\bibitem[Ghadimi and Wang(2018)]{ghadimi2018approximation}
Saeed Ghadimi and Mengdi Wang.
\newblock Approximation methods for bilevel programming.
\newblock \emph{arXiv preprint arXiv:1802.02246}, 2018.

\bibitem[Goh et~al.(2016)Goh, Cotter, Gupta, and
  Friedlander]{goh2016satisfying}
Gabriel Goh, Andrew Cotter, Maya Gupta, and Michael~P Friedlander.
\newblock Satisfying real-world goals with dataset constraints.
\newblock In \emph{Advances in Neural Information Processing Systems}, pages
  2415--2423, 2016.

\bibitem[Hardt et~al.(2016)Hardt, Price, and Srebro]{hardt2016equality}
Moritz Hardt, Eric Price, and Nati Srebro.
\newblock Equality of opportunity in supervised learning.
\newblock In \emph{Advances in Neural Information Processing Systems}, pages
  3315--3323, 2016.

\bibitem[Hillermeier et~al.(2001)]{hillermeier2001nonlinear}
Claus Hillermeier et~al.
\newblock \emph{Nonlinear multiobjective optimization: a generalized homotopy
  approach}, volume 135.
\newblock Springer Science \& Business Media, 2001.

\bibitem[Hu and Chen(2018)]{hu2018short}
Lily Hu and Yiling Chen.
\newblock A short-term intervention for long-term fairness in the labor market.
\newblock In \emph{Proceedings of the 2018 World Wide Web Conference}, pages
  1389--1398. International World Wide Web Conferences Steering Committee,
  2018.

\bibitem[Hu and Chen(2020)]{hu2020fair}
Lily Hu and Yiling Chen.
\newblock Fair classification and social welfare.
\newblock In \emph{Proceedings of the 2020 Conference on Fairness,
  Accountability, and Transparency}, pages 535--545, 2020.

\bibitem[Jin et~al.(2017)Jin, Ge, Netrapalli, Kakade, and
  Jordan]{jin2017escape}
Chi Jin, Rong Ge, Praneeth Netrapalli, Sham~M Kakade, and Michael~I Jordan.
\newblock How to escape saddle points efficiently.
\newblock In \emph{Proceedings of the 34th International Conference on Machine
  Learning-Volume 70}, pages 1724--1732. JMLR. org, 2017.

\bibitem[Kamani(2020)]{kamani2020multiobjective}
Mohammad~Mahdi Kamani.
\newblock Multiobjective optimization approaches for bias mitigation in machine
  learning.
\newblock 2020.

\bibitem[Kamani et~al.(2016)Kamani, Farhat, Wistar, and Wang]{kamani2016shape}
Mohammad~Mahdi Kamani, Farshid Farhat, Stephen Wistar, and James~Z Wang.
\newblock Shape matching using skeleton context for automated bow echo
  detection.
\newblock In \emph{2016 IEEE International Conference on Big Data (Big Data)},
  pages 901--908. IEEE, 2016.

\bibitem[Kamani et~al.(2018)Kamani, Farhat, Wistar, and
  Wang]{kamani2018skeleton}
Mohammad~Mahdi Kamani, Farshid Farhat, Stephen Wistar, and James~Z Wang.
\newblock Skeleton matching with applications in severe weather detection.
\newblock \emph{Applied Soft Computing}, 70:\penalty0 1154--1166, 2018.

\bibitem[Kamani et~al.(2019{\natexlab{a}})Kamani, Farhang, Mahdavi, and
  Wang]{kamani2019targeted}
Mohammad~Mahdi Kamani, Sadegh Farhang, Mehrdad Mahdavi, and James~Z Wang.
\newblock Targeted meta-learning for critical incident detection in weather
  data.
\newblock \emph{International Conference on Machine Learning, Workshop on
  "Climate Change: How Can AI Help?"}, 2019{\natexlab{a}}.

\bibitem[Kamani et~al.(2019{\natexlab{b}})Kamani, Haddadpour, Forsati, and
  Mahdavi]{kamani2019efficient}
Mohammad~Mahdi Kamani, Farzin Haddadpour, Rana Forsati, and Mehrdad Mahdavi.
\newblock Efficient fair principal component analysis.
\newblock \emph{arXiv preprint arXiv:1911.04931}, 2019{\natexlab{b}}.

\bibitem[Kamani et~al.(2020)Kamani, Farhang, Mahdavi, and
  Wang]{kamani2020targeted}
Mohammad~Mahdi Kamani, Sadegh Farhang, Mehrdad Mahdavi, and James~Z Wang.
\newblock Targeted data-driven regularization for out-of-distribution
  generalization.
\newblock In \emph{Proceedings of the 26th ACM SIGKDD International Conference
  on Knowledge Discovery \& Data Mining}, pages 882--891, 2020.

\bibitem[Kamishima et~al.(2011)Kamishima, Akaho, and
  Sakuma]{kamishima2011fairness}
Toshihiro Kamishima, Shotaro Akaho, and Jun Sakuma.
\newblock Fairness-aware learning through regularization approach.
\newblock In \emph{2011 IEEE 11th International Conference on Data Mining
  Workshops}, pages 643--650. IEEE, 2011.

\bibitem[Kannan et~al.(2019)Kannan, Roth, and Ziani]{kannan2019downstream}
Sampath Kannan, Aaron Roth, and Juba Ziani.
\newblock Downstream effects of affirmative action.
\newblock In \emph{Proceedings of the Conference on Fairness, Accountability,
  and Transparency}, pages 240--248. ACM, 2019.

\bibitem[Kearns and Roth(2019)]{kearns2019ethical}
Michael Kearns and Aaron Roth.
\newblock \emph{The Ethical Algorithm: The Science of Socially Aware Algorithm
  Design}.
\newblock Oxford University Press, 2019.

\bibitem[Kearns et~al.(2019)Kearns, Roth, and
  Sharifi-Malvajerdi]{kearns2019average}
Michael Kearns, Aaron Roth, and Saeed Sharifi-Malvajerdi.
\newblock Average individual fairness: Algorithms, generalization and
  experiments.
\newblock In \emph{Advances in Neural Information Processing Systems}, 2019.

\bibitem[Kim et~al.(2020)Kim, Chen, and Talwalkar]{kim2020model}
Joon~Sik Kim, Jiahao Chen, and Ameet Talwalkar.
\newblock Model-agnostic characterization of fairness trade-offs.
\newblock \emph{arXiv preprint arXiv:2004.03424}, 2020.

\bibitem[Kleinberg et~al.(2017)Kleinberg, Mullainathan, and
  Raghavan]{kleinberg2016inherent}
Jon Kleinberg, Sendhil Mullainathan, and Manish Raghavan.
\newblock Inherent trade-offs in the fair determination of risk scores.
\newblock In \emph{8th Innovations in Theoretical Computer Science Conference
  (ITCS 2017)}. Schloss Dagstuhl-Leibniz-Zentrum fuer Informatik, 2017.

\bibitem[Kusner et~al.(2017)Kusner, Loftus, Russell, and
  Silva]{kusner2017counterfactual}
Matt~J Kusner, Joshua Loftus, Chris Russell, and Ricardo Silva.
\newblock Counterfactual fairness.
\newblock In \emph{Advances in neural information processing systems}, pages
  4066--4076, 2017.

\bibitem[Lambrecht and Tucker(2019)]{lambrecht2019algorithmic}
Anja Lambrecht and Catherine Tucker.
\newblock Algorithmic bias? an empirical study of apparent gender-based
  discrimination in the display of stem career ads.
\newblock \emph{Management Science}, 65\penalty0 (7):\penalty0 2966--2981,
  2019.

\bibitem[Lin et~al.(2019)Lin, Zhen, Li, Zhang, and Kwong]{lin2019pareto}
Xi~Lin, Hui-Ling Zhen, Zhenhua Li, Qing-Fu Zhang, and Sam Kwong.
\newblock Pareto multi-task learning.
\newblock In \emph{Advances in Neural Information Processing Systems}, pages
  12037--12047, 2019.

\bibitem[Lipton et~al.(2017)Lipton, Chouldechova, and McAuley]{lipton2017does}
Zachary~C Lipton, Alexandra Chouldechova, and Julian McAuley.
\newblock Does mitigating {ML}’s disparate impact require disparate
  treatment?
\newblock \emph{Stat}, 1050:\penalty0 19, 2017.

\bibitem[Liu et~al.(2018)Liu, Dean, Rolf, Simchowitz, and
  Hardt]{liu2018delayed}
Lydia~T Liu, Sarah Dean, Esther Rolf, Max Simchowitz, and Moritz Hardt.
\newblock Delayed impact of fair machine learning.
\newblock In \emph{International Conference on Machine Learning}, pages
  3150--3158. PMLR, 2018.

\bibitem[Liu et~al.(2020)Liu, Wilson, Haghtalab, Kalai, Borgs, and
  Chayes]{liu2019disparate}
Lydia~T Liu, Ashia Wilson, Nika Haghtalab, Adam~Tauman Kalai, Christian Borgs,
  and Jennifer Chayes.
\newblock The disparate equilibria of algorithmic decision making when
  individuals invest rationally.
\newblock In \emph{Proceedings of the 2020 Conference on Fairness,
  Accountability, and Transparency}, pages 381--391, 2020.

\bibitem[Ma et~al.(2020)Ma, Du, and Matusik]{ma2020efficient}
Pingchuan Ma, Tao Du, and Wojciech Matusik.
\newblock Efficient continuous pareto exploration in multi-task learning.
\newblock In \emph{International Conference on Machine Learning}, pages
  6522--6531. PMLR, 2020.

\bibitem[Mahapatra and Rajan(2020)]{mahapatramulti}
Debabrata Mahapatra and Vaibhav Rajan.
\newblock Multi-task learning with user preferences: Gradient descent with
  controlled ascent in pareto optimization.
\newblock In \emph{International Conference on Machine Learning}, pages
  6597--6607. PMLR, 2020.

\bibitem[Marcinkowski et~al.(2020)Marcinkowski, Kieslich, Starke, and
  L{\"u}nich]{marcinkowski2020implications}
Frank Marcinkowski, Kimon Kieslich, Christopher Starke, and Marco L{\"u}nich.
\newblock Implications of ai (un-) fairness in higher education admissions: the
  effects of perceived ai (un-) fairness on exit, voice and organizational
  reputation.
\newblock In \emph{Proceedings of the 2020 Conference on Fairness,
  Accountability, and Transparency}, pages 122--130, 2020.

\bibitem[Martinez et~al.(2020)Martinez, Bertran, and Sapiro]{martinezminimax}
Natalia Martinez, Martin Bertran, and Guillermo Sapiro.
\newblock Minimax pareto fairness: A multi objective perspective.
\newblock In \emph{International Conference on Machine Learning}, pages
  6755--6764. PMLR, 2020.

\bibitem[Menon and Williamson(2018)]{menon2018cost}
Aditya~Krishna Menon and Robert~C Williamson.
\newblock The cost of fairness in binary classification.
\newblock In \emph{Conference on Fairness, Accountability and Transparency},
  pages 107--118, 2018.

\bibitem[Miettinen(2012)]{miettinen2012nonlinear}
Kaisa Miettinen.
\newblock \emph{Nonlinear Multiobjective Optimization}, volume~12.
\newblock Springer Science \& Business Media, 2012.

\bibitem[Nesterov and Polyak(2006)]{nesterov2006cubic}
Yurii Nesterov and Boris~T Polyak.
\newblock Cubic regularization of {Newton} method and its global performance.
\newblock \emph{Mathematical Programming}, 108\penalty0 (1):\penalty0 177--205,
  2006.

\bibitem[Okuno et~al.(2018)Okuno, Takeda, and Kawana]{okuno2018hyperparameter}
Takayuki Okuno, Akiko Takeda, and Akihiro Kawana.
\newblock Hyperparameter learning for bilevel nonsmooth optimization.
\newblock \emph{arXiv preprint arXiv:1806.01520}, 2018.

\bibitem[Peitz and Dellnitz(2018)]{peitz2018gradient}
Sebastian Peitz and Michael Dellnitz.
\newblock Gradient-based multiobjective optimization with uncertainties.
\newblock In \emph{NEO 2016}, pages 159--182. Springer, 2018.

\bibitem[Pleiss et~al.(2017)Pleiss, Raghavan, Wu, Kleinberg, and
  Weinberger]{pleiss2017fairness}
Geoff Pleiss, Manish Raghavan, Felix Wu, Jon Kleinberg, and Kilian~Q
  Weinberger.
\newblock On fairness and calibration.
\newblock In \emph{Advances in Neural Information Processing Systems}, pages
  5680--5689, 2017.

\bibitem[Rajeswaran et~al.(2019)Rajeswaran, Finn, Kakade, and
  Levine]{rajeswaran2019meta}
Aravind Rajeswaran, Chelsea Finn, Sham Kakade, and Sergey Levine.
\newblock Meta-learning with implicit gradients.
\newblock \emph{Advances in neural information processing systems}, 2019.

\bibitem[Samadi et~al.(2018)Samadi, Tantipongpipat, Morgenstern, Singh, and
  Vempala]{samadi2018price}
Samira Samadi, Uthaipon Tantipongpipat, Jamie~H Morgenstern, Mohit Singh, and
  Santosh Vempala.
\newblock The price of fair pca: One extra dimension.
\newblock In \emph{Advances in Neural Information Processing Systems}, pages
  10976--10987, 2018.

\bibitem[Tolan et~al.(2019)Tolan, Miron, G{\'o}mez, and
  Castillo]{tolan2019machine}
Song{\"u}l Tolan, Marius Miron, Emilia G{\'o}mez, and Carlos Castillo.
\newblock Why machine learning may lead to unfairness: Evidence from risk
  assessment for juvenile justice in catalonia.
\newblock In \emph{Proceedings of the Seventeenth International Conference on
  Artificial Intelligence and Law}, pages 83--92, 2019.

\bibitem[Wick et~al.(2019)Wick, Tristan, et~al.]{wick2019unlocking}
Michael Wick, Jean-Baptiste Tristan, et~al.
\newblock Unlocking fairness: a trade-off revisited.
\newblock In \emph{Advances in Neural Information Processing Systems}, pages
  8780--8789, 2019.

\bibitem[Woodworth et~al.(2017)Woodworth, Gunasekar, Ohannessian, and
  Srebro]{woodworth2017learning}
Blake Woodworth, Suriya Gunasekar, Mesrob~I Ohannessian, and Nathan Srebro.
\newblock Learning non-discriminatory predictors.
\newblock In \emph{Conference on Learning Theory}, pages 1920--1953. PMLR,
  2017.

\bibitem[Zafar et~al.(2017{\natexlab{a}})Zafar, Valera, Gomez~Rodriguez, and
  Gummadi]{zafar2017fairness}
Muhammad~Bilal Zafar, Isabel Valera, Manuel Gomez~Rodriguez, and Krishna~P
  Gummadi.
\newblock Fairness beyond disparate treatment \& disparate impact: Learning
  classification without disparate mistreatment.
\newblock In \emph{Proceedings of the 26th International Conference on World
  Wide Web}, pages 1171--1180. International World Wide Web Conferences
  Steering Committee, 2017{\natexlab{a}}.

\bibitem[Zafar et~al.(2017{\natexlab{b}})Zafar, Valera, Rodriguez, Gummadi, and
  Weller]{zafar2017parity}
Muhammad~Bilal Zafar, Isabel Valera, Manuel Rodriguez, Krishna Gummadi, and
  Adrian Weller.
\newblock From parity to preference-based notions of fairness in
  classification.
\newblock In \emph{Advances in Neural Information Processing Systems}, pages
  229--239, 2017{\natexlab{b}}.

\bibitem[Zafar et~al.(2017{\natexlab{c}})Zafar, Valera, Rogriguez, and
  Gummadi]{zafar2015fairness}
Muhammad~Bilal Zafar, Isabel Valera, Manuel~Gomez Rogriguez, and Krishna~P
  Gummadi.
\newblock Fairness constraints: Mechanisms for fair classification.
\newblock In \emph{Artificial Intelligence and Statistics}, pages 962--970.
  PMLR, 2017{\natexlab{c}}.

\bibitem[Zemel et~al.(2013)Zemel, Wu, Swersky, Pitassi, and
  Dwork]{zemel2013learning}
Rich Zemel, Yu~Wu, Kevin Swersky, Toni Pitassi, and Cynthia Dwork.
\newblock Learning fair representations.
\newblock In \emph{International Conference on Machine Learning}, pages
  325--333, 2013.

\end{thebibliography}
\newpage
\appendix

\appendixwithtoc

\clearpage
\section{Pareto Frontier: Solution Methods and Geometry}\label{app:MOO_prem}
In this part, we discuss some preliminary notions related to multiobjective optimization that can be useful for the rest of the manuscript. 
\subsection{Preliminaries}
The notion of optimality in a scalar optimization is straightforward. To have a similar notion in a multiobjective optimization, we need the definition of dominance between two arbitrary solutions of vector-valued functions ({\it e.g.}, see~\citet{miettinen2012nonlinear}). For a single-objective optimization we can compare two solutions based on their scalar value objective functions. In a multi-objective setting, similarly, we can define the notion of dominance as follows:
\begin{definition}[Dominance]
Consider a vector-valued objective function    $\bm{\mathrm{h}}: \mathbb{R}^d  \mapsto \mathbb{R}_{+}^m$ with $m$ objectives $\bm{\mathrm{h}}(\bm{w}) = \left[\mathrm{h}_1(\bm{w}), \ldots, \mathrm{h}_m(\bm{w})\right]^\top$. We say the solution $\bm{w}_1$ \textit{dominates} the solution $\bm{w}_2$ if $\mathrm{h}_i(\bm{w}_1) \leq \mathrm{h}_i(\bm{w}_2)$ for all $1\leq i \leq m$, and ${h}_j(\bm{w}_1) < {h}_j(\bm{w}_2)$ for at least one $1 \leq j \leq m$. We denote this  dominance as:
\begin{equation}\label{eq:dominance}
    \bm{\mathrm{h}}(\bm{w}_1) \prec_{m} \bm{\mathrm{h}}(\bm{w}_2)\;.
\end{equation}
\label{def:dom}
\end{definition}

In Definition~\ref{def:pef}, we defined the notion of Pareto efficient fairness. Now, we can characterize the existence of a PEF solution  for a given problem with fairness objectives. In particular, the following theorem provides a {\it necessary condition} for the existence of a PEF solution.
\begin{theorem}\label{theorem:exis}
Let $\bm{\mathrm{h}}(\bm{w}) = \left[\mathrm{h}_1(\bm{w}), \ldots, \mathrm{h}_m(\bm{w})\right]^\top$ be the objective vector we ought to minimize. If the individual functions  are bounded below, {\it i.e.}, $\mathrm{h}_i(\bm{w}) \geq C,\;\forall i \in \left[m\right], \bm{w}\in \mathcal{W}$, and $C\in\mathbb{R}$, then the set of Pareto efficient fair solutions is non-empty.
\end{theorem}

Before formally proving the existence of Pareto efficient  fair solution as stated in Theorem~\ref{theorem:exis}, we would like  to further  elaborate on its implications. We remark that the condition described in Theorem~\ref{theorem:exis} is about whether a Pareto efficient fair solution exists for a problem. Hence, there might be some extreme cases, where this solution does not exist or is not good enough. For instance, when objectives are completely contradicting each other but they are bounded below, a PEF solution exists but can be far away from the optimal point of each objective. Noting the inherent conflict between the fairness and accuracy measures, our ultimate goal is to find the optimal trade-off between these two measures. Thus, in most cases, with bounded objectives, well-defined PEF solutions exist and can be extracted. The following example will brighten the condition in Theorem~\ref{theorem:exis} with two extreme cases where a PEF solution does not exist, and  where it exists, but not well-defined.

\begin{figure}[h]
    \centering
    \subfigure[]{
    \includegraphics[width=0.3\linewidth]{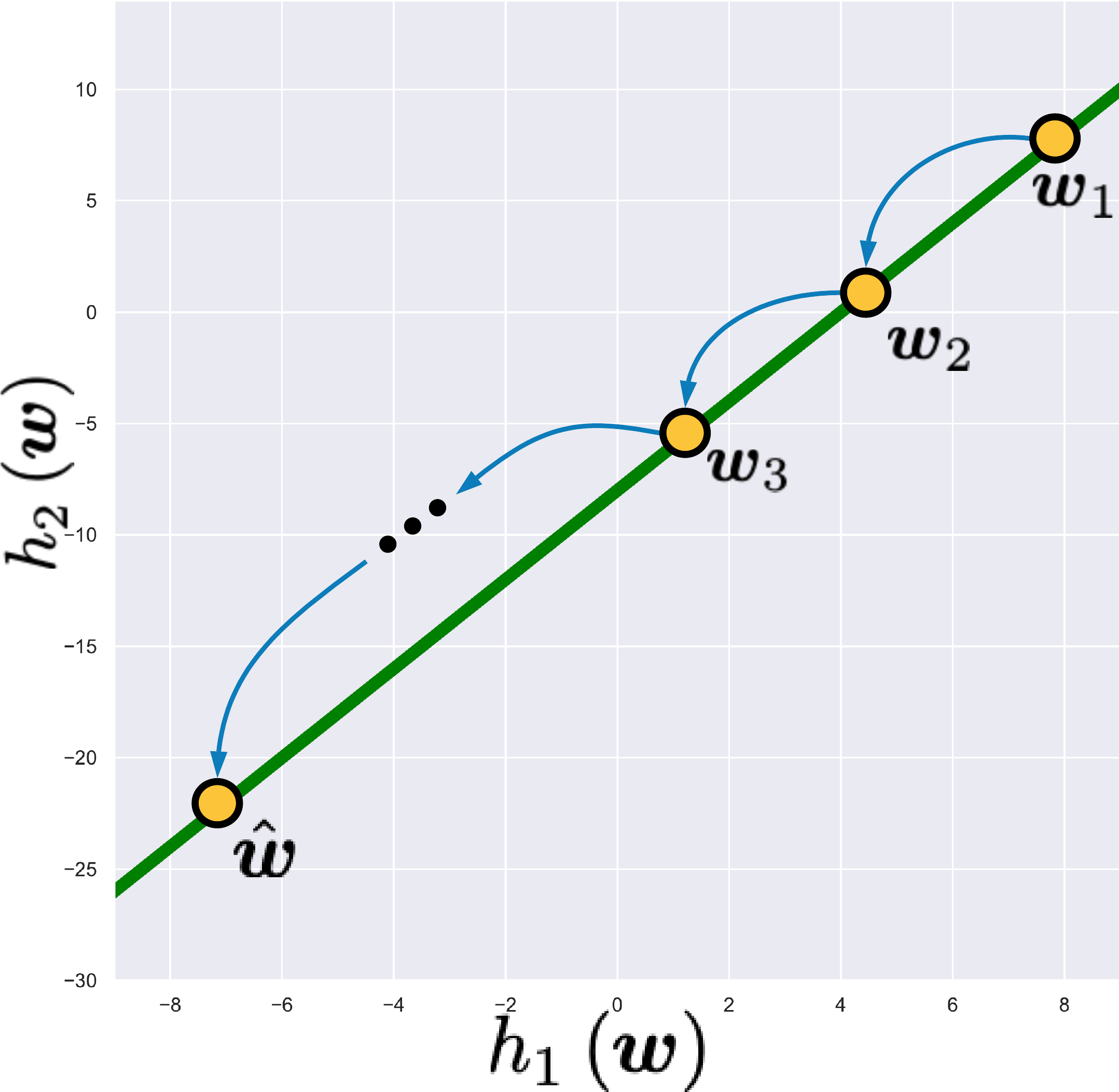}
    \label{fig:ex1}
    }
    \subfigure[]{
    \includegraphics[width=0.3\linewidth]{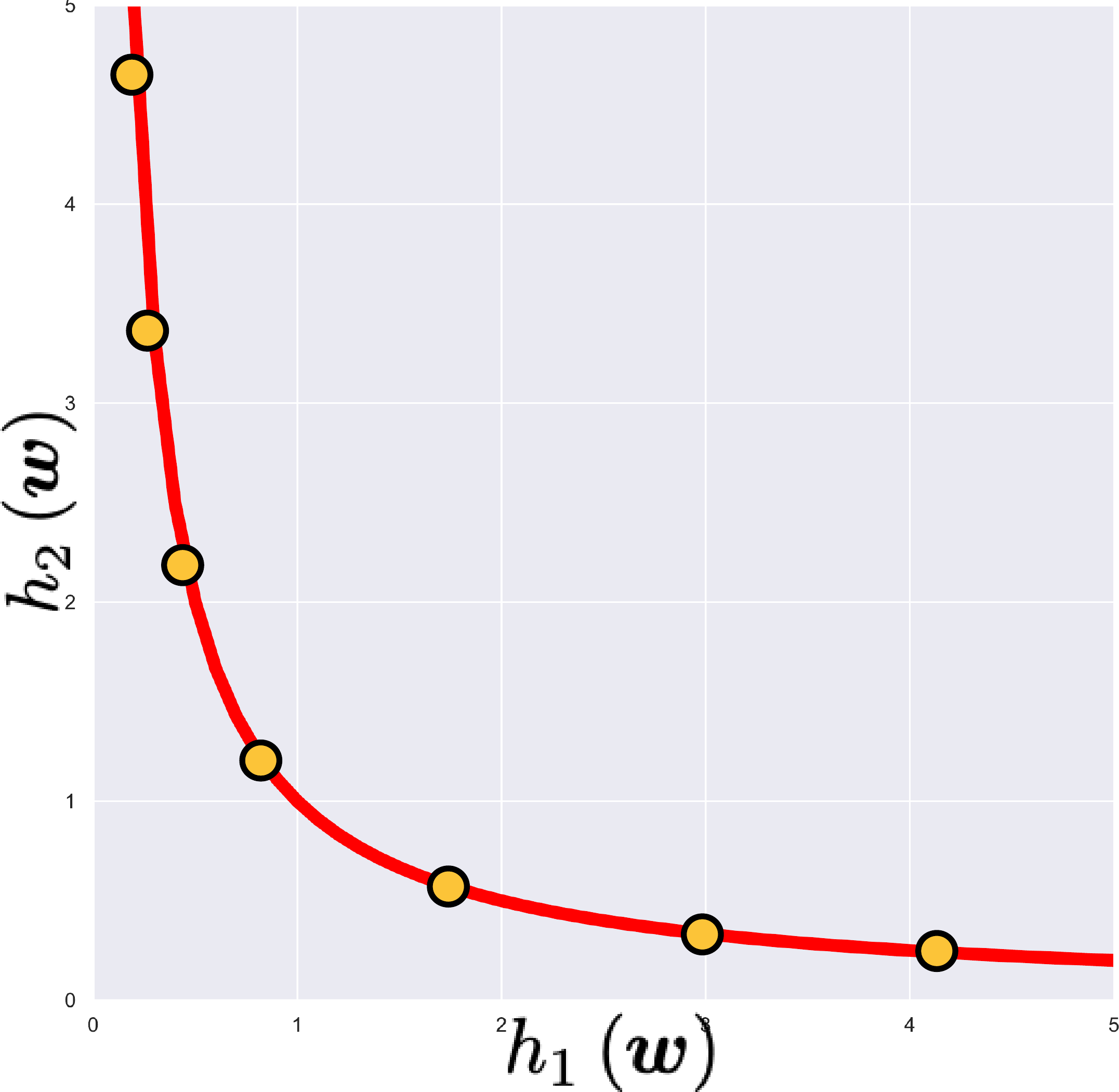}
    \label{fig:ex2}
    }
\caption{Value of $\mathrm{h}_1(\bm{w})$ versus $\mathrm{h}_2(\bm{w})$ for different $\bm{w}$s for arbitrary parameters in Example~\ref{ex:1}.}\label{fig:example}
\end{figure}
\begin{exmp}\label{ex:1}
(a) Consider that we have two objectives to satisfy and the parameter space is a single dimension space $w \in \mathcal{W} \subseteq \mathbb{R}^1$. Then, consider that the objectives are linear functions in the form of $\mathrm{h}_1(w) = aw + b$, and $\mathrm{h}_2(w) = cw + d$, where $a,b,c$, and $d$ are constants with $a,b > 0$. Next, the relationship between the two objectives are:
\begin{equation}
    \mathrm{h}_2(w) = \frac{c}{a}\mathrm{h}_1(w) + \frac{ad - bc}{a}\;.
    \nonumber
\end{equation}
Since, both the objectives are not bounded, when we want to minimize both together, we are not able to find a point that cannot be dominated by other points. This situation is depicted in Figure~\ref{fig:ex1}, for $a=2$, $b=3$, $c=4$, and $d=-2$. As can be inferred, for every chosen point like $\w_1$ there exists a $\hat{\w}$ that dominates it. Meaning, for objective vector of $\bm{\mathrm{h}}(\w)=\left[\mathrm{h}_1(\w),\mathrm{h}_2(\w)\right]^\top$, we have $\bm{\mathrm{h}}(\hat{\w}) \prec_m \ldots  \prec_m \bm{\mathrm{h}}(\w_3) \prec_m \bm{\mathrm{h}}(\w_2) \prec_m \bm{\mathrm{h}}(\w_1)$; hence, the set of PEF solutions would be empty.

(b) Now, consider that we have two objectives in a single dimension parameter space with the form of $\mathrm{h}_1(w) = w^2$ and $\mathrm{h}_2(w)=1/w^2$. This means that both objectives are bounded below, but completely contradicting each other. In terms of the objectives, they have the relationship of:
\begin{equation}
    \mathrm{h}_2(w) = \frac{1}{\mathrm{h}_1(w)}, \quad \mathrm{h}_1(w) > 0 \;\; \forall w \in \mathcal{W}.
    \nonumber
\end{equation}
In this extreme case, all the points belong to the Pareto frontier because they cannot dominate each other as depicted in Figure~\ref{fig:ex2}. Thus, the PEF solution set is not well-defined.
\end{exmp}

Now we can prove Theorem~\ref{theorem:exis} as follows:
\begin{proof}
To prove the existence of a PEF solution, we define an auxiliary function and establish the conditions that guarantee the existence of the optimal solution. To this end, consider the following function defined at any point $\bar{\bm{w}} \in \mathcal{W}$:
\begin{equation}
\begin{aligned}
\Omega\left(\bm{w}; \bar{\bm{w}}\right) \triangleq  \quad & \sum_{i=1}^{m}\mathrm{h}_i\left(\bm{w}\right)\\
\quad \quad \quad \textrm{s.t.} \quad & \boldsymbol{\mathrm{h}}(\bm{w}) \prec_{m} \boldsymbol{\mathrm{h}}(\bar{\bm{w}})\;.\\
\label{eq:aux-fn}
\end{aligned}
\end{equation}
Any feasible solution to $\arg\min_{\bm{w}} \Omega\left(\bm{w};\bar{\bm{w}}\right)$ yields a solution that is as good as $\bar{\bm{w}}$. Therefore, solving the above optimization results in solutions (if any exists) that are better than $\bar{\bm{w}}$.

Now, we turn to characterizing the existence of a PEF solution as optimal solutions of the auxiliary problem defined in~(\ref{eq:aux-fn}). First, it is easy to see by a simple contradictory argument that any \textit{efficient} solution  $\bm{w}^*$ should be an optimal solution to $\arg\min_{\bm{w}}\Omega\left(\bm{w};\bm{w}^*\right)$, {\it i.e.}, $\bm{w}^* \in \arg\min_{\bm{w}}\Omega\left(\bm{w};\bm{w}^*\right)$, with the function value of $\sum_{i=1}^{m}\mathrm{h}_i\left(\bm{w}^*\right)$. Of course, this can be utilized as an efficiency check (i.e,  any efficient solution $\bm{w}^*$ should belong to the solution set of auxiliary function defined at that point $\arg\min_{\bm{w}}\Omega(\bm{w};\bm{w}^*)$). We also note that, if $\bm{w}^*$ is an optimal solution to $\Omega\left(\bm{w};\bar{\bm{w}}\right)$, for a given point $\bar{\bm{w}}$, then $\bm{w}^*$ is efficient and we have:
\begin{equation}
    \mathrm{h}_i\left(\bm{w}^*\right) \leq \mathrm{h}_i\left(\bar{\bm{w}}\right),\quad i = 1,\ldots, m.
\end{equation}
Therefore, if $\bar{\bm{w}}$ is not in the optimal solution set of $\Omega\left(\bm{w};\bar{\bm{w}}\right)$, then, any optimal solution to $\Omega\left(\bm{w};\bar{\bm{w}}\right)$ would be efficient and better than $\bar{\bm{w}}$, {\it i.e.}, $\sum_{i=1}^{m}\mathrm{h}_i\left(\bm{w}^*\right) \leq \sum_{i=1}^{m}\mathrm{h}_i\left(\bar{\bm{w}}\right)$. This can be easily shown by contradiction.

Equipped with the above auxiliary function and understanding the basic facts about its solution set, we only left with finding suitable conditions that guarantee the existence of optimal solutions of $\Omega\left(\bm{w};\bar{\bm{w}}\right)$. Indeed, for any $\bar{\bm{w}} \in \mathcal{W}$, the minimization of the function in~(\ref{eq:aux-fn}) has an optimal solution, {\it i.e.} $\bar{\bm{w}}$, if the individual functions $\mathrm{h}_i(.), \; i =1,\ldots,m$ are lower-semicontinuous over $\mathcal{W}$, and if the set $\left\{ \bm{w} \in \mathcal{W} ; \mathrm{h}_i(\bm{w}) \leq \mathrm{h}_i\left(\bar{\bm{w}}\right), \; \forall\; 1\leq i\leq m\right\}$  is compact. This holds when the constraint set is compact and the functions are continuous or bounded below as stated in the theorem~\footnote{We note that, an arbitrary function $\mathrm{h}: \mathbb{R}^d\rightarrow \left[-\infty,+\infty\right]$ is lower-semicontinuous throughout $\mathbb{R}^d$ if the set $\left\{ x| \mathrm{h}(x)\leq \epsilon\right\}$ is closed for every $\epsilon \in \mathbb{R}$ or the epigraph of $\mathrm{h}$ is closed set in $\mathbb{R}^{d+1}$. Obviously, this condition holds for any convex function.}.

We note that, when $\Omega\left(\bm{w};\bar{\bm{w}}\right)$ is unbounded below ({\it i.e.}, achieving value $-\infty$), or it is finite but not attainable by any feasible solution, then no optimal solution exists. To see the former case, we consider the case that $\Omega\left(\bm{w};\bar{\bm{w}}\right) = -\infty$ for some $\bm{w},\bar{\bm{w}} \in \mathcal{W}$. Now, using the optimization problem of the function in~(\ref{eq:aux-fn}), we write its dual problem as:
\begin{equation}
    d(\bar{\bm{w}}) = \sup_{\boldsymbol{\lambda} \in \mathbb{R}^m_+} \inf_{\bm{w}\in \mathcal{W}} \sum_{i=1}^{m}\mathrm{h}_i\left(\bm{w}\right) + \boldsymbol{\lambda}^\top \left(\boldsymbol{\mathrm{h}}(\bm{w}) - \boldsymbol{\mathrm{h}}(\bar{\bm{w}})\right)\;,
    \label{eq:aux-dual}
\end{equation}
where $\boldsymbol{\lambda}$ is the dual parameter. Based on the weak duality theorem we know that we will have $d(\bar{\bm{w}}) \leq \mathrm{h}(\bar{\bm{w}})$ for every $\bar{\bm{w}}$. Then, if we have the objectives are not bounded below, we can have $\mathrm{h}\left(\bar{\bm{w}}\right) = -\infty$ for some $\bar{\bm{w}} \in \mathcal{W}$, which implies that the dual problem also achieves a $ -\infty$ and becomes infeasible. However, as long as the primal problem is feasible, a duality gap exists which in turn implies the efficiency of the primal solution--meaning that the primal solution is Pareto efficient by the definition of the auxiliary optimization problem. 

Putting all together, we conclude that when the individual functions are bounded below and the constraint set is compact, then the solution set of auxiliary function is non-empty which immediately implies the existence of an efficient solution. We not that these conditions can be relaxed to include other families of vector-valued optimization problems, but here we only focused on the characteristics that govern the optimization problem for finding a PEF solution. 
\end{proof}

\subsection{Pareto frontier}\label{app:pf_prem}
The set of all feasible Pareto efficient points in a multiobjective optimization problem is called Pareto Frontier. In most problems, finding a Pareto efficient point is not enough, and we need to extract the set of these points in order to better understand the relationship between the conflicting objectives for a more rationale decision. Although extracting points on the Pareto frontier of a multiobjective optimization has been investigated vastly in different domains~\citep{hillermeier2001nonlinear}, this is still a challenging problem. These algorithmic efforts can mainly be categorized into four different groups (beside heuristic approaches). In this part, we first describe each of these four categories, and then discuss the geometrical characteristics of the Pareto frontier and how the points on this set can be related to the weights of the objects in the optimization problem.

The main four groups of algorithmic approaches for extracting Pareto frontier points are as follows:
\paragraph{Normal Boundary Intersection (NBI).} This approach has been proposed by~\cite{das1998normal} to find evenly distributed points on the Pareto frontier and address the problem of different scales in the objectives by previous methods. In this method, we need to adjust each dimension of the objective vector with the respective global minimum of that objective, hence their minimum values translate to zero. Next, they need to find points on the convex hull of individual minima from different objectives. Using the points and the normal vector of this convex hull toward the origin, they solve an optimization problem to move from those points along the normal vector until they reach a point from the feasible set boundary. The resulting solutions are stationary points of the objective vector. The main drawback of this approach is that finding the global minimum of each objective is computationally expensive. More importantly, finding the solutions in the parameter space corresponding to the points from the mentioned convex hull in the objective space is not trivial.

\paragraph{Geometrical Exploration.} There are several different geometrical strategies to find points from the Pareto frontier set, given we already found one. These methods rely on the geometrical properties such as rank condition on the Hessian matrix of the Pareto frontier and use it to traverse it, starting from a point in that set. In~\cite{hillermeier2001nonlinear}, the author elaborates on these properties in detail and propose two different Homotopy strategies. One of them is the local exploration of the Pareto frontier using a linear program based on Jacobian and Hessian matrices of the objectives. In~\cite{ma2020efficient}, the authors used the  same idea and propose an algorithm to solve this linear program, to trace points on the Pareto frontier. Using second-order information of the stationary points, especially in high-dimensional problems with a high number of objectives is not computationally feasible.

\paragraph{Weight Perturbation.} The second geometrical idea from~\cite{hillermeier2001nonlinear} is weight perturbation. As it was mentioned in the previous section, there is an ordering relationship between mixing weights $\alpha^*_i,\; i\in[m]$ of the final solution and the location of that solution on the Pareto frontier. In~\cite{hillermeier2001nonlinear}, the author used the same idea and proposed an algorithm to find nearby points on Pareto frontier by perturbing the mixing weight of the stationary point found by other algorithms like our Algorithm~\ref{alg:fairpareto}. Similar to the previous approach, the second-order information needed for this algorithm to work is disadvantageous.

\paragraph{Preference-based solutions.} A novel class of approaches to find points on the Pareto frontier on the specified part of the objective space are called preference-based methods. These approaches use a predefined preference vector over different objectives. The value of each objective's preference with respect to other objectives' shows how much we want to be close to the minimum value of that objective compared to other objectives' in the Pareto frontier. In~\cite{lin2019pareto}, authors use different preference vectors to segment the Pareto frontier into equal parts, and then, design an optimization problem to enforce the final solution of each run to fall in one of those desired parts. By repeating this algorithm they can find evenly distributed points from different parts of the Pareto frontier. However, their approach heavily depends on the location of the initial solution. Recently, in~\cite{mahapatramulti}, authors use the preference vector and design an optimization to converge to a single specific point from the Pareto frontier.

\subsection{The geometry of Pareto frontier}\label{app:pf_geo}
In this section, we investigate the geometrical characteristics of the Pareto frontier curve. For a simple case of bi-objective optimization, where the Pareto frontier curve of these two objectives is smooth and convex we can have the following property:
\begin{proposition}
For a smooth and convex Pareto frontier of a vector objective, an ordering in the optimal weights, $\al$, entails an ordering over their corresponding objectives.
\label{con:pf}
\end{proposition}
\begin{proof}
To show this property on smooth and convex Pareto curves, we consider a simple case of two objectives, however, it can be easily generalized to more than two objectives. For this case, we want to find a point from the Pareto frontier. To do so, as discussed before, we minimize the scalarized version of the objectives using mixing parameters $1-\alpha$ and $\alpha$ for $\mathrm{h}_1\left(\w\right)$ and $\mathrm{h}_2\left(\w\right)$, respectively. As it was used by other studies~\citep{lin2019pareto,das1997closer}, we can replace $\alpha$ by $\frac{\sin{\theta}}{\sin{\theta}+\cos{\theta}},\; \theta \in \left[0,\frac{\pi}{2}\right]$, and optimize for $\theta$ instead of $\alpha$. Then the scalar objective we are trying to minimize for two objectives, $\mathrm{h}_1\left(\w\right)$ and $\mathrm{h}_2\left(\w\right)$, becomes:
\begin{equation}
    h\left(\w\right) = \frac{\cos{\theta} \cdot \mathrm{h}_1\left(\w\right) + \sin{\theta}  \cdot \mathrm{h}_2\left(\w\right)}{\sin{\theta}+\cos{\theta}} ,\quad \forall \theta \in \left[0,\frac{\pi}{2}\right], \w \in \mathcal{W}.
\end{equation}
Now, it is clear that the transformation from $\mathrm{h}_1\left(\w\right)$ and $\mathrm{h}_2\left(\w\right)$ to the new objective is a rotation by the angle of $\theta$ following by normalization by $\sin{\theta}+\cos{\theta}$. As can be inferred from Figure~\ref{fig:pf2}, it seems that the solution of the optimization is the minimum point on the curve based on the new rotated coordinate.
\begin{figure}[t]
    \centering
    \includegraphics[width=0.3\linewidth]{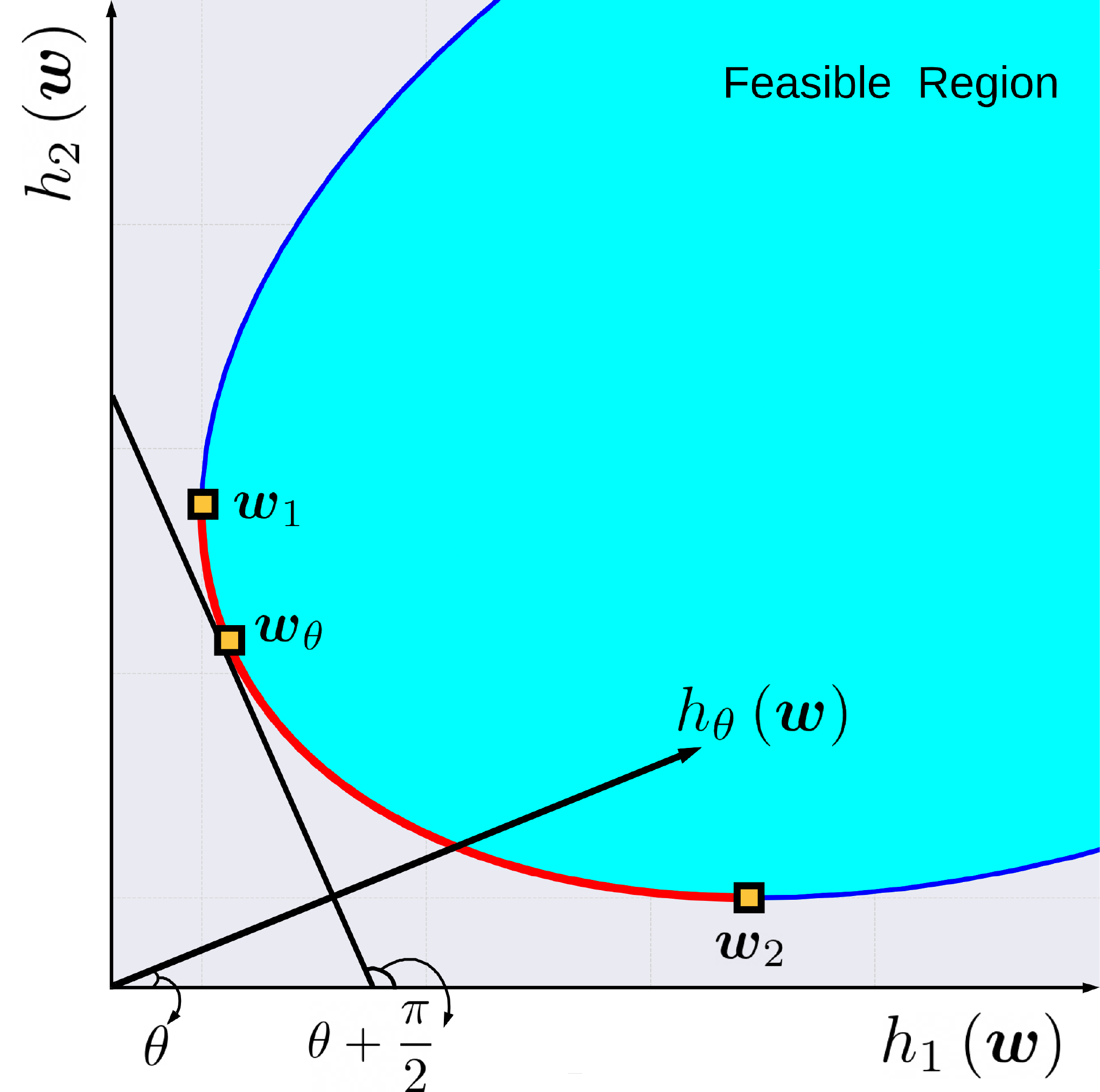}
\caption{Schematic view of the relationship between weighting parameters and Pareto optimal solution properties. If we replace $\alpha$ in a bi-objective optimization problem with $\frac{\sin{\theta}}{\sin{\theta}+\cos{\theta}}$, it can be shown that for every $\theta \in \left[0,\frac{\pi}{2}\right]$ the Pareto efficient solution $\w_\theta$ is the tangent point of the Pareto frontier curve and the line perpendicular to the rotated coordinated of $\mathrm{h}_1\left(\w\right)$ with angle $\theta$, {\it i.e.}, $h_\theta\left(\w\right)=\cos{\theta} \cdot \mathrm{h}_1\left(\w\right) + \sin{\theta}  \cdot \mathrm{h}_2\left(\w\right)$. } \label{fig:pf2}
\end{figure}

For two extreme points in $\theta$, that are, $\theta_1 = 0$ and $\theta_2 = \frac{\pi}{2}$, the minimization objective reduces to $h\left(\w\right) = \mathrm{h}_1\left(\w\right)$ and $h\left(\w\right) = \mathrm{h}_2\left(\w\right)$, respectively. Hence, for these two cases, we find the minimum solution for each objective, respectively. For values between these two extremes, $\theta \in \left(0,\frac{\pi}{2}\right)$, the minimum point is the tangent point of the curve with the line perpendicular to the rotated coordinate, as it is shown in Figure~\ref{fig:pf2}. To show that this point is the Pareto solution based on the current $\alpha$ value, we can write down the slope of the curve on this tangent point:
\begin{align}
    \frac{\partial \mathrm{h}_2\left(\w\right)}{\partial \mathrm{h}_1\left(\w\right)}\bigg\rvert_{\w = \w_\theta} &= \tan\left(\theta + \frac{\pi}{2}\right) \nonumber\\
    & = - \cot{\theta} \nonumber \\
    & = - \frac{1-\alpha}{\alpha}\;,
    \label{eq:optimal_gradient}
\end{align}
then, replacing the partial derivatives with gradients of each function with respect to $\w$, we can get:
\begin{align}
    \nabla  \mathrm{h}_2\left(\w_\theta\right) = - \frac{1-\alpha}{\alpha} \cdot \nabla  \mathrm{h}_1\left(\w_\theta\right) \nonumber\\
    \left(1-\alpha\right) \nabla  \mathrm{h}_1\left(\w_\theta\right) + \alpha \nabla  \mathrm{h}_2\left(\w_\theta\right) = 0\;,
\end{align}
which is the condition derived from KKT conditions in (\ref{eq:kkt}) for the Pareto optimality of a point. Hence, it is clear that in the case of smooth and convex Pareto curve, for each $\alpha$ and subsequently for each $\theta$, the Pareto efficient point is the tangent point described above. For the two extreme points mentioned before we have:
\begin{align}
    \theta_1 = 0 &\;\Leftrightarrow\; \alpha_1 = 1 \;\Leftrightarrow\; \frac{\partial \mathrm{h}_2\left(\w_1\right)}{\partial \mathrm{h}_1\left(\w_1\right)} = -\infty\\
    \theta_2 = \frac{\pi}{2} &\;\Leftrightarrow\; \alpha_2 = 0 \;\Leftrightarrow\; \frac{\partial \mathrm{h}_2\left(\w_2\right)}{\partial \mathrm{h}_1\left(\w_2\right)} = 0\;,
\end{align}
then for every angle in between, $\theta_1 \leq \theta \leq \theta_2$, we have  $-\infty \leq \frac{\partial \mathrm{h}_2\left(\w_\theta\right)}{\partial \mathrm{h}_1\left(\w_\theta\right)} \leq 0$, which implies the non-increasing curve. Next, for two angles in this range, $\theta_1 \leq \theta_3 \leq \theta_4 \leq \theta_2$, we will have $1 \geq \alpha_3 \geq \alpha_4 \geq 0$. Then, it can be easily verified that we have $\mathrm{h}_1\left(\w_{\theta_4}\right) \geq \mathrm{h}_1\left(\w_{\theta_3}\right)$ and $\mathrm{h}_2\left(\w_{\theta_4}\right) \leq \mathrm{h}_2\left(\w_{\theta_3}\right)$, which entails the ordering on the objectives as well.

\end{proof}
Although we can show this property for a convex and smooth Pareto frontier, most of the time, these conditions are not met. Hence, the aforementioned mapping might not be valid in general, nonetheless, it can shed light on how different weight parameters are mapped to different regions in the Pareto frontier curve. In the following sections, we introduce our novel first-order approach for tracing points from any Pareto frontier, regardless of these conditions.

\section{Reduction of Known Fairness Notions to PEF}\label{app:peeo}
In addition to the Pareto efficient equality of opportunity, we show how other notions of fairness can be reduced to PEF. In this section we show the reduction of equalized odds and disparate mistreatment to PEF as other examples of reductions.
\paragraph{Pareto efficient equalized odds}\label{sec:peeod}
As previously stated, satisfying equalized odds requires both true positive rates and false positive rates to be equal among all groups. We can use the same notion, to define its new variant, \textit{Pareto efficient equalized odds}. For satisfying equal true positive rates, we use the same objectives as (\ref{eq:eo_obj}) for Pareto efficient equality of opportunity. To satisfy equal false positive rates among each group, we will, instead, focus on having equal true negative rate for each group, $\mathsf{TN}_k = 1 - \mathsf{FP}_k =  \mathbb{P}\left[\hat{y}=-1| a = s_k, y=-1\right]$ for all $k \in \{1,\ldots,c\}$. This condition can be reflected using the corresponding loss, thus, we denote the set of samples for each group from negative class as  $\mathcal{S}_k^- = \left\{ (\bm{x}_i,a_i,y_i) \in \mathcal{D} \, | \, a_i = s_k, y_i = -1,  1 \leq k \leq c\right\}$. Then, the loss is:
\begin{equation}\label{eq:neg-loss}
    \mathcal{L}_k^-(\bm{w}) = \frac{1}{|\mathcal{S}_k^-|}\sum_{(\bm{x}_i,y_i)\in\mathcal{S}_k^-} \ell\left(\bm{w}; (\bm{x}_i,y_i)\right),\;\; k \in \{1,\ldots,c\}.
\end{equation}
Using (\ref{eq:neg-loss}) and similarly to true positive rates, we can define pairwise objectives to be added as additional fairness objectives to the objective vector:
\begin{equation}\label{eq:eod_obj}
    \mathcal{H}_{i,j}^-(\bm{w}) = \phi\left(\mathcal{L}_i^-(\bm{w}) - \mathcal{L}_j^-(\bm{w}) \right), \;\; 1\leq i,j\leq c, i\neq j,
\end{equation}
where $\phi(.)$, again, is the smooth penalization function, and the objective vector that needs to be minimized is $\boldsymbol{\mathrm{h}}_{\text{EOD}}(\bm{w}) = \big[\mathcal{L}(\bm{w}), \mathcal{H}_{1,2}^+(\bm{w}), \ldots, \mathcal{H}_{c-1,c}^+(\bm{w}), \mathcal{H}_{1,2}^-(\bm{w}), \ldots, \mathcal{H}_{c-1,c}^-(\bm{w}) \big]^\top \in \mathbb{R}^m_{+}$, where $m=1 + 2 \times{c \choose 2}$. Similarly, a solution has the property of Pareto efficient equalized odds if it belongs to the PEF solution set of the constructed vector objective $\boldsymbol{\mathrm{h}}_{\text{EOD}}(\bm{w})$ in the optimization of~(\ref{eq:mul-fair}).
\paragraph{Pareto efficient disparate mistreatment}\label{sec:dismis}
Analogous to two previously reduced notion of fairness to PEF, we can define Pareto efficient disparate mistreatment. In this notion, we need to have an equal misclassification error rate among all groups in the sensitive feature. To that end, we denote the loss of each group as:
\begin{equation}\label{eq:group-loss}
    \mathcal{L}_k(\bm{w}) = \frac{1}{|\mathcal{S}_k|}\sum_{(\bm{x}_i,y_i)\in\mathcal{S}_k} \ell\left(\bm{w}; (\bm{x}_i,y_i)\right),\;\; k \in \{1,\ldots,c\},
\end{equation}
where $\mathcal{S}_k = \left\{ (\bm{x}_i,a_i,y_i) \in \mathcal{D} \, | \, a_i = s_k,  1 \leq k \leq c\right\}$ denotes each group's set. Now, the fairness objectives for each pair of sensitive groups can be defined as:
\begin{equation}
    \mathcal{H}_{i,j}(\bm{w}) = \phi\left(\mathcal{L}_i(\bm{w}) - \mathcal{L}_j(\bm{w}) \right), \;\; 1\leq i,j\leq c, i\neq j,
\end{equation}
to satisfy equal empirical loss on each group, in addition to the main loss. Thus, the objective vector is $\boldsymbol{\mathrm{h}}_{\text{DM}}(\bm{w}) = \big[\mathcal{L}(\bm{w}), \mathcal{H}_{1,2}(\bm{w}), \ldots, \mathcal{H}_{c-1,c}(\bm{w}) \big]^\top \in \mathbb{R}^m_{+}$, where $m=1 + {c \choose 2}$. Following the previous definitions, a solution has the Pareto efficient disparate mistreatment effect if it belongs to the PEF solution set of the optimization in~(\ref{eq:mul-fair}) using $\boldsymbol{\mathrm{h}}_{\text{DM}}(\bm{w})$.

\section{Missing Proofs}\label{app:missp}
In this appendix, we provide the missing proofs from main body.
\subsection{Proof of Theorem~\ref{theo:converge_pareto}}\label{app:theo:converge_pareto}
\begin{proof}
Denoting the solution of optimization~(\ref{eq:bilevel}) by $\w_b^*$ and the set of Pareto efficient fair solutions of vector optimization~(\ref{eq:mul-fair}) by $\Omega_\text{P}$. If $\w_b^* \notin \Omega_\text{P}$, then there would be a $\w_\text{P} \in \Omega_\text{P}$ that dominates $\w_b^*$, which means that:
\begin{align}
    \mathrm{h}_i\left(\w_\text{P}\right) &\leq  \mathrm{h}_i\left(\w_b^*\right), \quad \forall i\in[m]\nonumber\\
    \mathrm{h}_i\left(\w_\text{P}\right) & <  \mathrm{h}_i\left(\w_b^*\right), \quad \text{for at least one}\; i\in[m] \nonumber
\end{align}
Then, for any $\al \in \Delta_m$, we can write:
\begin{align}
    \sum_{i=1}^m \alpha_i \mathrm{h}_i\left(\w_\text{P}\right) & < \sum_{i=1}^{m} \alpha_i \mathrm{h}_i\left(\w_b^*\right)\nonumber\\
    \Psi\left(\bm{\mathrm{h}}\left(\w_\text{P}\right), \al \right) & < \Psi\left(\bm{\mathrm{h}}\left(\w_b^*\right), \al \right),
\end{align}
where it contradicts the assumption that $\w_b^*$ is the solution of optimization~(\ref{eq:bilevel}). Hence, $\w_b^* \in \Omega_\text{P}$.
\end{proof}

\subsection{Proof of Lemma~\ref{theo:paret}}\label{app:lemma}
\begin{proof}
First, for the sake of simplicity, we use the notation of $\bm{\mathrm{g}}^{(t)} = \bm{\mathrm{g}}(\w^{(t)})$.
We define the set of feasible solutions for the inner level optimization problem in~(\ref{eq:bilevel}):
\begin{equation}
    {\Theta}(\w^{(t)}) = \left\{ \sum_{i=1}^{m} \alpha_i \bm{\mathrm{g}}_i^{(t)} \;|\; \alpha_i \geq 0 \;\;\forall i \in \{1,\ldots,m \}, \sum_{i=1}^m \alpha_i = 1  \right\}\;.
\end{equation}
Now, we consider the solution of the inner level optimization of~(\ref{eq:bilevel}) to be $\al^*\left(\w^{(t)}\right)$ for the iteration $t$. Then the gradient of outer level can be written as:
\begin{equation}
    \bar{\bm{\mathrm{g}}}^{(t)} = \nabla_{\w} \Psi\left(\bm{\mathrm{h}}\left(\w^{(t)}\right),\al^*\left(\w^{(t)}\right)\right) = \sum_{i=1}^m \alpha_i^*\left(\w^{(t)}\right) \bm{\mathrm{g}}_i\left(\w^{(t)}\right)
    \label{eq:descent_dir}\;.
\end{equation}
Hence, we consider that, if (\ref{eq:paretcond}) was wrong, then there would be at least one $\bm{\mathrm{g}}_i^{(t)} \in {\Theta}(\w^{(t)})$, where $\langle\bar{\bm{\mathrm{g}}}^{(t)}, \bm{\mathrm{g}}_i^{(t)}\rangle < 0$. Then, we would consider the following optimization problem:
\begin{equation}\label{eq:proof-opt}
    \underset{\gamma \in [0,1]}{\min} \left\|(1-\gamma) \bm{\mathrm{g}}_i^{(t)} + \gamma \bar{\bm{\mathrm{g}}}^{(t)} \right\|_2^2\;,
\end{equation}
note that $(1-\gamma) \bm{\mathrm{g}}_i^{(t)} + \gamma \bar{\bm{\mathrm{g}}}^{(t)} \in \Theta(\w^{(t)})$, thus, it is similar to the inner level optimization problem in (\ref{eq:bilevel}). On the other side, based on (\ref{eq:descent_dir}), we know  that $\|\bar{\bm{\mathrm{g}}}^{(t)}\|_2^2$ is the solution to the optimization in the inner level of~(\ref{eq:bilevel}). It shows that the function in (\ref{eq:proof-opt}) is monotonically decreasing in $\gamma$, and the minimum happens at $\gamma^*=1$. By writing the first-order condition on this point we have:
\begin{align}\nonumber
    2\left( \bar{\bm{\mathrm{g}}}^{(t)} - \bm{\mathrm{g}}_i^{(t)}\right)^\top \left( \bm{\mathrm{g}}_i^{(t)} + \gamma\left(\bar{\bm{\mathrm{g}}}^{(t)} - \bm{\mathrm{g}}_i^{(t)}\right)\right) &\leq 0 \\\nonumber
    \left(\bar{\bm{\mathrm{g}}}^{(t)}\right)^\top\left( \bar{\bm{\mathrm{g}}}^{(t)} - \bm{\mathrm{g}}_i^{(t)}\right) & \stackrel{\text{\ding{192}}}{\leq} 0 \\\label{eq:proof-cont}
    \|\bar{\bm{\mathrm{g}}}^{(t)}\|_2^2 & \leq \left(\bar{\bm{\mathrm{g}}}^{(t)}\right)^\top\bm{\mathrm{g}}_i^{(t)}\;,\end{align}
where \ding{192} is resulted from using $\gamma^*=1$. Inequality in (\ref{eq:proof-cont}) contradicts our assumption on $\left(\bar{\bm{\mathrm{g}}}^{(t)}\right)^\top\bm{\mathrm{g}}_i^{(t)} < 0$. Therefore, $\bar{\bm{\mathrm{g}}}^{(t)}$ is a descent direction for all objectives, meaning, $-\langle\bar{\bm{\mathrm{g}}}^{(t)},\bm{\mathrm{g}}_i^{(t)}\rangle < 0$ for every $1 \leq i \leq m$.
\end{proof}

\subsection{Proof of Proposition~\ref{prop:grad_kl}}\label{app:grad_kl}
\begin{proof}
We start by the gradient of the Softmax function. For an input vector $\bm{z}=\left[z_1,\ldots,z_m\right]$ applying the Softmax function would yield to $\sigma_i \triangleq \sigma_i\left(\bm{z}\right) = e^{z_i}/\sum_{j\in[m]} e^{z_j},\; \forall i\in[m]$. Then, it is straightforward to show the derivative of this function with respect to each input is:
\begin{equation}\label{eq:softmax_grad}
  \frac{\partial \sigma_i}{\partial z_j}=\begin{cases}
    \sigma_i\left(1-\sigma_i\right), & \text{if $i=j$}.\\
    \sigma_i\sigma_j, & \text{otherwise}.
  \end{cases}
\end{equation}
Now, if we consider $z_i = \pi_i \mathrm{h_i}\left(\bm{w}\right),\; \forall i\in[m]$, we can write the partial gradient of the objective $\mathrm{h}_\mathsf{KL}$ with respect to $z_j$ is:
\begin{align}
    \frac{\partial \mathrm{h}_\mathsf{KL}\left(\bm{w},\bm{\pi}\right)}{\partial z_j} &= \sum_{i\in[m]} \frac{\partial \sigma_i}{\partial z_j} \cdot \left( \log\left(m\sigma_i\right) + 1 \right) \nonumber\\
    & \stackrel{\text{\ding{192}}}{=} \sigma_j \left(1-\sigma_j\right) \left( \log\left(m\sigma_j\right) + 1 \right) - \sum_{i \in [m], i\neq j} \sigma_i\sigma_j \left( \log\left(m\sigma_i\right) + 1 \right) \nonumber\\
    &=\sigma_j \left( \log\left(m\sigma_j\right) + 1 \right) - \sum_{i \in [m]} \sigma_i\sigma_j \left( \log\left(m\sigma_i\right) + 1 \right) \nonumber\\
    &=\sigma_j \left( \log\left(m\sigma_j\right) + 1 \right) - \sigma_j \sum_{i \in [m]} \sigma_i \left( \log\left(m\sigma_i\right) + 1 \right) \nonumber\\
    &\stackrel{\text{\ding{193}}}{=}\sigma_j \left( \log\left(m\sigma_j\right) + 1 \right) - \sigma_j  \left( \mathrm{h}_\mathsf{KL}\left(\bm{w},\bm{\pi}\right) + 1 \right) \nonumber\\
    &= \sigma_j \left(\log\left(m\sigma_j\right) - \mathrm{h}_\mathsf{KL}\left(\bm{w},\bm{\pi}\right)\right)\;,
\end{align}
where in \ding{192}, we use~(\ref{eq:softmax_grad}), and in \ding{193}, we use~(\ref{eq:kl-loss}) and the fact that $\sum_{i\in[m]}\sigma_i = 1$. Now, by considering the definition of $z_j$ and noting that $\bm{\mathrm{g}}_i = \partial \mathrm{h}_i\left(\bm{w}\right)/\partial \bm{w}$, the proof is complete.
\end{proof}

\subsection{Proof of Lemma~\ref{lemma:ge}}\label{app:lemma:ge}
\begin{proof}\label{pr:ge}
We start by expanding the error term  for every $\w$ and its corresponding $\hat{\al} = \hat{\al}\left(\w\right)$ and $\al^* = \al^*\left(\w\right)$ as follows. For brevity, we replace $\Psi\left(\bm{\mathrm{h}}\left(\w\right),{\al}\left(\w\right)\right)$ and  $\Phi\left(\bm{\mathrm{h}}\left(\w\right),{\al}\left(\w\right)\right)$ with  $\Psi\left(\w,{\al}\right)$ and  $\Phi\left(\w,{\al}\right)$, respectively, and their gradients accordingly. 
\begin{align}\label{eq:error_expan}
    \llVert \nabla\Psi\left(\w, \hat{\al}\right) -\nabla\Psi\left(\w, \al^*\right) \rrVert & = \left\lVert \nabla_{\w} \Psi\left(\w, \hat{\al}\right) + \nabla \hat{\al}\left(\w\right)\nabla_{\al} \Psi\left(\w, \hat{\al}\right) \right. \nonumber\\
    & \qquad\qquad\qquad \qquad \left. -\nabla_{\w} \Psi\left(\w, \al^*\right) - \nabla \al^*\left(\w\right)\nabla_{\al} \Psi\left(\w, \al^*\right)\right\rVert \nonumber\\
    & \leq \llVert \nabla_{\w} \Psi\left(\w, \hat{\al}\right) - \nabla_{\w} \Psi\left(\w, \al^*\right) \rrVert \nonumber\\ 
    & \quad + \llVert \nabla \hat{\al}\left(\w\right) \left(\nabla_{\al} \Psi\left(\w, \hat{\al}\right) - \nabla_{\al} \Psi\left(\w, \al^*\right) \right)\rrVert \nonumber\\ 
    &\quad + \llVert \left(\nabla \hat{\al}\left(\w\right) - \nabla \al^*\left(\w\right)\right) \nabla_{\al} \Psi\left(\w, \al^*\right) \rrVert\;.
\end{align}
Then, using the linear relationship of the outer level function $\Psi(\w,\al) = \al^\top\bm{\mathrm{h}}$ and bounded gradient in Assumption~\ref{ass:f1}, we have:
\begin{align}
    \llVert  \nabla_{\w} \Psi\left(\w, \hat{\al}\right) -  \nabla_{\w} \Psi\left(\w, \al^* \right) \rrVert & = \llVert {\mathrm{G}}\hat{\al} -  {\mathrm{G}}\al^*\rrVert \nonumber\\
    & \leq \llVert \hat{\al} - \al^*\rrVert \llVert \mathrm{G} \rrVert_\text{F} \nonumber\\
    & \leq\mathsf{B}_{\bm{\mathrm{h}}}\sqrt{m} \llVert \hat{\al} - \al^*\rrVert\;.
    \label{eq:error_expan_aux1}
\end{align}
Replacing~(\ref{eq:error_expan_aux1}) in~(\ref{eq:error_expan}), we have:
\begin{align}\label{eq:error_expan1}
\llVert \nabla\Psi\left(\w, \hat{\al}\right) -\nabla\Psi\left(\w, \al^*\right) \rrVert & \leq \mathsf{B}_{\bm{\mathrm{h}}}\sqrt{m} \llVert \hat{\al} - \al^*\rrVert + \llVert \nabla \hat{\al}\left(\w\right) \rrVert_{\text{F}} \llVert\nabla_{\al} \Psi\left(\w, \hat{\al}\right) - \nabla_{\al} \Psi\left(\w, \al^*\right) \rrVert \nonumber\\
 &\quad + \llVert \nabla \hat{\al}\left(\w\right) - \nabla \al^*\left(\w\right)\rrVert_\text{F} \llVert\nabla_{\al} \Psi\left(\w, \al^*\right) \rrVert.
\end{align}
 Again, we use the linear relationship of the outer level, and hence, we can write $\nabla_{\al} \Psi\left(\w,\hat{\al}\right) = \nabla_{\al} \Psi\left(\w, \al^* \right) = \bm{\mathrm{h}}$, which makes the second term in~(\ref{eq:error_expan1}) zero and for the third term we use the Assumption~\ref{ass:f} on the function values and have $\llVert \bm{\mathrm{h}}\rrVert \leq \mathsf{D}_{\bm{\mathrm{h}}}$. Also, replace $ \nabla \hat{\al}\left(\w\right)$ and $\nabla \al^*\left(\w\right)$ with approximation in~(\ref{eq:grad2}), we will have:
 \begin{align}\label{eq:error_expan2}
 \left\lVert \nabla\Psi\left(\w, \hat{\al}\right) - \right. & \left. \nabla\Psi\left(\w, \al^*\right) \right\rVert \nonumber\\
   &\leq \mathsf{B}_{\bm{\mathrm{h}}} \sqrt{m}\llVert \hat{\al} - \al^* \rrVert + \mathsf{D}_{\bm{\mathrm{h}}} \sqrt{m} \left\lVert \nabla^2_{\w\al} \Phi\left(\w,\hat{\al}\right)\left[\nabla^2_{\al\al} \Phi\left(\w,\hat{\al}\right)\right]^{-1} \right. \nonumber\\
   & \qquad\qquad\qquad\qquad\qquad\qquad\qquad\left. - \nabla^2_{\w\al} \Phi\left(\w,\al^*\right)\left[\nabla^2_{\al\al}\Phi\left(\w,\al^*\right)\right]^{-1} \right\rVert_\text{F} \nonumber \\
    & \leq \mathsf{B}_{\bm{\mathrm{h}}} \sqrt{m}\llVert \hat{\al} - \al^* \rrVert \nonumber \\
    & \quad + \mathsf{D}_{\bm{\mathrm{h}}} \sqrt{m} \llVert \left(\nabla^2_{\w\al} \Phi\left(\w,\hat{\al}\right) - \nabla^2_{\w\al} \Phi\left(\w,\al^*\right)\right)\left[\nabla^2_{\al\al} \Phi\left(\w,\al^*\right)\right]^{-1}\rrVert_\text{F} \nonumber \\
    &\quad + \mathsf{D}_{\bm{\mathrm{h}}} \sqrt{m} \llVert \nabla^2_{\w\al} \Phi\left(\w,\hat{\al}\right) \left(\left[\nabla^2_{\al\al} \Phi\left(\w,\hat{\al}\right)\right]^{-1} - \left[\nabla^2_{\al\al} \Phi\left(\w,\al^*\right)\right]^{-1}\right)\rrVert_\text{F}\;.
 \end{align}
Next, we use the fact that $\nabla^2_{\al\al} \Phi\left(\w,\al^*\right) = \nabla^2_{\al\al} \Phi\left(\w,\hat{\al}\right) = \mathrm{G}^\top\mathrm{G}$, which makes the last term zero. Also, using Assumption~\ref{ass:inner}, we have the strong convexity property of the inner problem $\llVert \mathrm{G}^\top\mathrm{G}\rrVert \geq \mu_\Phi$; and the Lipschitz continuity of $\nabla^2_{\w\al} \Phi\left(\w,\al\right)$ from Assumption~\ref{ass:inner2}:
\begin{align}\label{eq:error_expan3}
 \llVert \nabla\Psi\left(\w, \hat{\al}\right) -\nabla\Psi\left(\w, \al^*\right) \rrVert &  \leq \mathsf{B}_{\bm{\mathrm{h}}} \sqrt{m}\llVert \hat{\al} - \al^* \rrVert \nonumber \\
 &+ \mathsf{D}_{\bm{\mathrm{h}}} \sqrt{m} \llVert \nabla^2_{\w\al} \Phi\left(\w,\hat{\al}\right) - \nabla^2_{\w\al} \Phi\left(\w,\al^*\right)\rrVert_\text{F}  / \llVert\nabla^2_{\al\al} \Phi\left(\w,\al^*\right)\rrVert_\text{F} \nonumber\\
& \leq \left(\mathsf{B}_{\bm{\mathrm{h}}} \sqrt{m} +  \frac{\mathsf{D}_{\bm{\mathrm{h}}} \sqrt{m}}{\mu_\Phi} L_{\w\al}\right) \llVert \hat{\al} - \al^* \rrVert\;, 
 \end{align}
\end{proof}

\subsection{Proof of Lemma~\ref{lemma:smooth}}\label{app:lemma:smooth}
\begin{proof}
It is straightforward from the definition of $\Psi(\w,\al)$ in~(\ref{eq:bilevel}). We start by expanding the term using the definition of the $\Psi$ function. For brevity we replace $\Psi\left(\bm{\mathrm{h}}\left(\w\right),{\al}\left(\w\right)\right)$ and  $\Phi\left(\bm{\mathrm{h}}\left(\w\right),{\al}\left(\w\right)\right)$ with  $\Psi\left(\w,{\al}\right)$ and  $\Phi\left(\w,{\al}\right)$, respectively, and their gradients accordingly:
\begin{align}
    \left\lVert \nabla \Psi\left(\w_1,\al^*\left(\w_1\right)\right)  \right. &\left. - \nabla \Psi\left(\w_2,\al^*\left(\w_2\right)\right) \right\rVert \nonumber\\
    & = \llVert \mathrm{G}\left(\w_1\right)\al^*\left(\w_1\right) - \mathrm{G}\left(\w_2\right)\al^*\left(\w_2\right)\rrVert \nonumber\\
    & = \llVert \left(\mathrm{G}\left(\w_1\right) - \mathrm{G}\left(\w_2\right) \right)\al^*\left(\w_1\right) + \mathrm{G}\left(\w_2\right) \left(\al^*\left(\w_1\right) - \al^*\left(\w_2\right) \right) \rrVert \nonumber\\
    & \leq \llVert \left(\mathrm{G}\left(\w_1\right) - \mathrm{G}\left(\w_2\right) \right)\al^*\left(\w_1\right)\rrVert + \llVert \mathrm{G}\left(\w_2\right) \left(\al^*\left(\w_1\right) - \al^*\left(\w_2\right) \right) \rrVert \nonumber\\
    & \leq \llVert \alpha_1^*\left(\w_1\right)\left(\bm{\mathrm{g}}_1\left(\w_1\right) - \bm{\mathrm{g}}_1\left(\w_2\right)\right) + \ldots + \alpha_m^*\left(\w_1\right)\left(\bm{\mathrm{g}}_m\left(\w_1\right) - \bm{\mathrm{g}}_m\left(\w_2\right)\right) \rrVert \nonumber \\
    & \quad + \llVert \mathrm{G}\left(\w_2\right)\rrVert_\text{F} \llVert \al^*\left(\w_1\right) - \al^*\left(\w_2\right) \rrVert\,.
    \label{eq:lemma:smooth1}
\end{align}
Then, we use the triangle inequality for the first term. As for the second term, we use the definition in~(\ref{eq:grad2}) to bound the gradient of the inner level variable as:
\begin{align}
    \llVert\nabla \al^*\left(\w\right)\rrVert & = \llVert\nabla^2_{\bm{w}\bm{\alpha}}\Phi\left(\bm{w},\bm{\alpha}^*(\bm{w})\right)\left[\nabla^2_{\bm{\alpha}\bm{\alpha}}\Phi\left(\bm{w},\bm{\alpha}^*(\bm{w})\right)\right]^{-1}\rrVert \nonumber\\
    & \leq \llVert \nabla^2_{\bm{w}\bm{\alpha}}\Phi\left(\bm{w},\bm{\alpha}^*(\bm{w})\right)\rrVert / \llVert\nabla^2_{\bm{\alpha}\bm{\alpha}}\Phi\left(\bm{w},\bm{\alpha}^*(\bm{w})\right)\rrVert \nonumber\\
    & \leq \mathsf{H}_{\Phi}/\mu_\Phi\;,
    \label{eq:lemma:smooth:aux1}
\end{align}
where we use the Assumptions~\ref{ass:inner2} and~\ref{ass:inner} for bounding of $\nabla^2_{\w\al} \Phi\left(\w,\al^*\right)$ and strong convexity parameter of $\nabla^2_{\al\al} \Phi\left(\w,\al^*\right)$. This implies that $ \al^*\left(\w\right)$ is Lipschitz continuous with the constant of $\mathsf{H}_{\Phi}/\mu_\Phi$. Plugging back~(\ref{eq:lemma:smooth:aux1}) into~(\ref{eq:lemma:smooth1}), we will have:
\begin{align}
    \llVert \nabla \Psi\left(\w_1,\al^*\left(\w_1\right)\right) - \nabla \Psi\left(\w_2,\al^*\left(\w_2\right)\right) \rrVert & \leq \alpha_1^*\left(\w_1\right)\llVert \bm{\mathrm{g}}_1\left(\w_1\right) - \bm{\mathrm{g}}_1\left(\w_2\right)\rrVert + \ldots \nonumber \\
    & \quad + \alpha_m^*\left(\w_1\right)\llVert\bm{\mathrm{g}}_m\left(\w_1\right) - \bm{\mathrm{g}}_m\left(\w_2\right) \rrVert \nonumber \\
    & \quad + \mathsf{B}_{\bm{\mathrm{h}}} \sqrt{m}\cdot \frac{\mathsf{H}_{\Phi}}{\mu_\Phi} \llVert \w_1 - \w_2\rrVert\;.
    \label{eq:lemma:smooth2}
\end{align}
Now, we use the Assumption~\ref{ass:f2}, for smoothness of each objective function with their respective parameter $L_i, i\in[m]$:
\begin{align}
    \llVert \nabla \Psi\left(\w_1,\al^*\left(\w_1\right)\right) - \nabla \Psi\left(\w_2,\al^*\left(\w_2\right)\right) \rrVert & \leq L_1\alpha_1^*\left(\w_1\right)\llVert\w_1 - \w_2\rrVert + \ldots \nonumber\\ 
    & \quad + L_m\alpha_m^*\left(\w_1\right)\llVert\w_1 - \w_2\rrVert \nonumber\\
    &\quad +  \frac{\mathsf{H}_{\Phi}\mathsf{B}_{\bm{\mathrm{h}}} \sqrt{m}}{\mu_\Phi} \llVert \w_1 - \w_2\rrVert\;.
    \label{eq:lemma:smooth3}
\end{align}
Finally, we replace each smoothness parameter with $L_{\text{max}} \triangleq \max\left\{L_1,\ldots,L_m\right\}$ and note that $\al^*\left(\w\right) \in \Delta_m$, we will have:
\begin{align}
    \llVert \nabla \Psi\left(\w_1,\al^*\left(\w_1\right)\right) - \nabla \Psi\left(\w_2,\al^*\left(\w_2\right)\right) \rrVert & \leq  \left(L_\text{max} + \frac{\mathsf{H}_{\Phi}\mathsf{B}_{\bm{\mathrm{h}}} \sqrt{m}}{\mu_\Phi} \right) \llVert \w_1 - \w_2\rrVert\;.
\end{align}
\end{proof}

\subsection{Proof of Theorem~\ref{theorem:convex_convergence}}\label{app:theorem:convex_convergence}
Equipped with Lemma~\ref{lemma:ge}, we can now turn to the proof of convergence for the convex outer objectives:
\begin{proof}
Following the convexity of each objective function of $f_i$, we can write the following inequality:
\begin{equation}\label{eq:theorem:convex_convergence:aux1}
    \mathrm{h}_i(\w^{(t)}) - \mathrm{h}_i(\w^*) \leq \nabla \mathrm{h}_i(\w^{(t)})^{\top} (\w^{(t)} - \w^*)\;,
\end{equation}
where $\w^*$ is the solution to the problem in~(\ref{eq:bilevel}). Considering that the weights for each objectives $\al_i$ are always non-negative values, their weighted sum is also convex and we have:

\begin{align}
    \Psi\left(\w^{(t)},\hat{\al}\left(\w^{(t)}\right) \right) - \Psi\left(\w^*,\al^*\left(\w^*\right) \right) &  \leq \nabla_{\w} \Psi\left(\w^{(t)},\hat{\al}\left(\w^{(t)}\right)\right)^\top \left(\w^{(t)} - \w^*\right) \nonumber\\
    & \stackrel{\text{\ding{192}}}{=} \frac{1}{\eta}\langle \w^{(t)} - \w^{(t+1)},\w^{(t)} - \w^*\rangle \nonumber\\
    & = \frac{1}{2\eta} \left(\llVert \w^{(t)} - \w^* \rrVert^2 + \llVert \w^{(t)} - \w^{(t+1)}
\rrVert^2 \right. \nonumber\\
& \qquad\qquad\qquad\qquad\qquad\qquad\left. - \llVert \w^{(t+1)} - \w^* \rrVert^2 \right), \label{eq:theorem:convex_convergence:aux2}
\end{align}
where, \ding{192} comes from the update rule of the outer level. Now, from the smoothness of the outer function from Lemma~\ref{lemma:smooth}, we have:
\begin{align}
     \Psi\left(\w^{(t+1)},\al^*\left(\w^{(t+1)}\right) \right) &  \leq \Psi\left(\w^{(t)},\al^*\left(\w^{(t)}\right)\right) + \nabla\Psi\left(\w^{(t)},\al^*\left(\w^{(t)}\right)\right)^\top\left(\w^{(t+1)} - \w^{(t)}\right) \nonumber\\
     &  \quad + \frac{L_\Psi}{2}\llVert \w^{(t+1)} - \w^{(t)} \rrVert^2\;.
      \label{eq:theorem:convex_convergence1}.
\end{align}
Then, by plugging the inequality of~(\ref{eq:theorem:convex_convergence:aux2}) in~(\ref{eq:theorem:convex_convergence1}) and the fact that $\Psi\left(\w^{(t)},\al^*\left(\w^{(t)}\right) \right) \leq \Psi\left(\w^{(t)},\hat{\al}\left(\w^{(t)}\right) \right)$, we will have:
\begin{align}
    \Psi\left(\w^{(t+1)},\al^*\left(\w^{(t+1)}\right) \right) & \leq \Psi\left(\w^*,\al^*\left(\w^*\right) \right) + \nabla\Psi\left(\w^{(t)},\al^*\left(\w^{(t)}\right)\right)^\top\left(\w^{(t+1)} - \w^{(t)}\right) \nonumber\\
     & \quad +  \frac{1}{2\eta} \left(\llVert \w^{(t)} - \w^* \rrVert^2 - \llVert \w^{(t+1)} - \w^* \rrVert^2 \right) \nonumber\\
     &\quad + \left(\frac{1}{2\eta} + \frac{L_\Psi}{2}\right) \llVert \w^{(t+1)} - \w^{(t)} \rrVert^2\;.
    \label{eq:theorem:convex_convergence2}
\end{align}
Then, we are adding and subtracting the inexact gradient term of $\nabla\Psi\left(\w^{(t)},\hat{\al}\left(\w^{(t)}\right)\right)$. Also, we use the update rule of the outer level and denote the gradient error by $\bm{\Lambda}^{(t)} \triangleq \nabla \Psi\left(\w^{(t)},\hat{\al}\left(\w^{(t)}\right)\right) - \nabla \Psi\left(\w^{(t)},\al^*\left(\w^{(t)}\right)\right)$ for every iteration of the outer level, hence, we will have:
\begin{align}
    \Psi\left(\w^{(t+1)},\right. & \left.\al^*\left(\w^{(t+1)}\right) \right) \nonumber\\
    &\leq \Psi\left(\w^*,\al^*\left(\w^*\right) \right) \nonumber\\
    &\quad + \left\langle \nabla\Psi\left(\w^{(t)},\al^*\left(\w^{(t)}\right)\right) -  \nabla\Psi\left(\w^{(t)},\hat{\al}\left(\w^{(t)}\right)\right),\w^{(t+1)} - \w^{(t)}\right\rangle \nonumber\\
  &\quad + \nabla\Psi\left(\w^{(t)},\hat{\al}\left(\w^{(t)}\right)\right)^\top\left(\w^{(t+1)} - \w^{(t)}\right) \nonumber\\
  &\quad +  \frac{1}{2\eta} \left(\llVert \w^{(t)} - \w^* \rrVert^2 - \llVert \w^{(t+1)} - \w^* \rrVert^2 \right) + \left(\frac{1}{2\eta} + \frac{L_\Psi}{2}\right) \llVert \w^{(t+1)} - \w^{(t)} \rrVert^2 \nonumber\\
  & = \Psi\left(\w^*,\al^*\left(\w^*\right) \right) +  \left\langle \bm{\Lambda}^{(t)},\w^{(t)} - \w^{(t+1)}\right\rangle \nonumber\\ 
  & \quad + \frac{1}{2\eta} \left(\llVert \w^{(t)} - \w^* \rrVert^2 - \llVert \w^{(t+1)} - \w^* \rrVert^2 \right) - \eta \llVert \nabla\Psi\left(\w^{(t)},\hat{\al}\left(\w^{(t)}\right)\right) \rrVert^2  \nonumber\\
  &\quad+ \left( \frac{1}{2\eta} + \frac{L_\Psi}{2}\right) \eta^2 \llVert \nabla\Psi\left(\w^{(t)},\hat{\al}\left(\w^{(t)}\right)\right) \rrVert^2\;.
    \label{eq:theorem:convex_convergence3}
\end{align}
Next, we set the outer level learning rate to be $\eta \leq 1/L_\Psi$ and use the fact that $\langle \bm{a},\bm{b}\rangle \leq \lVert\bm{a}\rVert\lVert\bm{b}\rVert$ and $R = \max_{\w} \lVert\w_1 - \w_2\rVert$ for any $\w_1,\w_2 \in \Omega$:
\begin{align}
    \Psi\left(\w^{(t+1)},\al^*\left(\w^{(t+1)}\right) \right) & \leq \Psi\left(\w^*,\al^*\left(\w^*\right) \right) \nonumber\\
    &\quad +  \frac{1}{2\eta} \left(\llVert \w^{(t)} - \w^* \rrVert^2 - \llVert \w^{(t+1)} - \w^* \rrVert^2 \right) + R \llVert \bm{\Lambda}^{(t)} \rrVert\;.
     \label{eq:theorem:convex_convergence4}
\end{align}
Finally, we are use Lemma~\ref{lemma:sconv-converg} and~\ref{lemma:ge} to bound the gradient error:
\begin{align}
    \Psi\left(\w^{(t+1)},\al^*\left(\w^{(t+1)}\right) \right) & \leq \Psi\left(\w^*,\al^*\left(\w^*\right) \right) +  \frac{1}{2\eta} \left(\llVert \w^{(t)} - \w^* \rrVert^2 - \llVert \w^{(t+1)} - \w^* \rrVert^2 \right) \nonumber\\
       & \quad + R A_\Psi \exp\left(-\frac{K}{2\kappa}\right) \llVert \al^{(0)}\left(\w^{(t)}\right) - \al^*\left(\w^{(t)}\right)\rrVert\;.
    \label{eq:theorem:convex_convergence5}
\end{align}
Then, by summing up both sides over $t$ from $0$ to $T-1$ and dividing by $T$, and by denoting $\bar{\w} = \sum_{t=1}^{T}\w^{(t)}/T$ and $\w^{T+1}=\w^*$, we have:
\begin{align}
    \Psi\left(\bar{\w},\bar{\al} \right) - \Psi\left(\w^*,\al^*\left(\w^*\right)\right) &\leq \frac{1}{2\eta T}\llVert \w^{(0)} - \w^* \rrVert^2 \nonumber\\
    & \quad + \frac{RA_\Psi }{T}\exp\left(-\frac{K}{2\kappa}\right)\sum_{t=0}^{T-1}  \llVert \al^{(0)}\left(\w^{(t)}\right) - \al^*\left(\w^{(t)}\right)\rrVert,
\end{align}
where $\bar{\al} = \left[\bar{\alpha}_1,\ldots,\bar{\alpha}_m\right]$, and $\bar{\alpha}_i = \min_t \alpha^*_i\left(\w^{(t)}\right)$ for $1\leq t\leq T$ and $i \in [m]$.
\end{proof}

\subsection{Proof of Theorem~\ref{theorem:nonconvex_convergence}}\label{app:theorem:nonconvex_convergence}
\begin{proof}
When the main function is generally non-convex, we can follow the convergence analysis with a modification for the residual error because of the inner level approximation. Note that, the inner level function is strongly convex, and hence, the results from Lemma~\ref{lemma:ge} are valid. Now, by considering the smoothness of the outer level from Lemma~\ref{lemma:smooth} and the update rule of the outer level, we can write:
\begin{align}
     \Psi\left(\w^{(t+1)},\al^*\left(\w^{(t+1)}\right) \right) &  \leq \Psi\left(\w^{(t)},\al^*\left(\w^{(t)}\right)\right) + \nabla\Psi\left(\w^{(t)},\al^*\left(\w^{(t)}\right)\right)^\top\left(\w^{(t+1)} - \w^{(t)}\right) \nonumber\\
     &  \quad + \frac{L_\Psi}{2}\llVert \w^{(t+1)} - \w^{(t)} \rrVert^2 \nonumber\\
     & = \Psi\left(\w^{(t)},\al^*\left(\w^{(t)}\right)\right) \nonumber\\
     & \quad - \eta \nabla\Psi\left(\w^{(t)},\al^*\left(\w^{(t)}\right)\right)^\top \nabla\Psi\left(\w^{(t)},\hat{\al}\left(\w^{(t)}\right)\right) \nonumber \\
     & \quad  + \frac{L_\Psi\eta^2}{2}\llVert \nabla\Psi\left(\w^{(t)},\hat{\al}\left(\w^{(t)}\right)\right)  \rrVert^2.
     \label{eq:theorem:nonconvex_convergence1} 
\end{align}
Then, using the definition of $\bm{\Lambda}^{(t)}$ from Appendix~\ref{app:theorem:convex_convergence}, we can write:
\begin{align}
     \Psi\left(\w^{(t+1)},\al^*\left(\w^{(t+1)}\right) \right) &  \leq \Psi\left(\w^{(t)},\al^*\left(\w^{(t)}\right)\right) - \eta \llVert \nabla\Psi\left(\w^{(t)},\al^*\left(\w^{(t)}\right)\right) \rrVert^2 \nonumber\\
     & \quad  - \eta\left\langle\nabla\Psi\left(\w^{(t)},\al^*\left(\w^{(t)}\right)\right), \bm{\Lambda}^{(t)} \right\rangle \nonumber\\
     &\quad + \frac{L_\Psi\eta^2}{2}\llVert \nabla\Psi\left(\w^{(t)},\al^*\left(\w^{(t)}\right)\right) + \bm{\Lambda}^{(t)}  \rrVert^2 \nonumber \\
     & = \Psi\left(\w^{(t)},\al^*\left(\w^{(t)}\right)\right) + \left(\frac{L_\Psi\eta^2}{2} - \eta\right) \llVert \nabla\Psi\left(\w^{(t)},\al^*\left(\w^{(t)}\right)\right) \rrVert^2 \nonumber\\
     &\quad + \frac{L_\Psi\eta^2}{2} \llVert \bm{\Lambda}^{(t)}\rrVert^2 - \left(\eta - L_\Psi\eta^2\right)\left\langle\nabla\Psi\left(\w^{(t)},\al^*\left(\w^{(t)}\right)\right), \bm{\Lambda}^{(t)} \right\rangle\;.
    \label{eq:theorem:nonconvex_convergence2} 
\end{align}
Next, by setting the learning rate $\eta \leq 1/L_\Psi$, we have:
\begin{align}
    \Psi\left(\w^{(t+1)},\al^*\left(\w^{(t+1)}\right) \right) &  \leq \Psi\left(\w^{(t)},\al^*\left(\w^{(t)}\right)\right) + \left(\frac{L_\Psi\eta^2}{2} - \eta\right) \llVert \nabla\Psi\left(\w^{(t)},\al^*\left(\w^{(t)}\right)\right) \rrVert^2 \nonumber\\
    &\quad + \frac{L_\Psi\eta^2}{2} \llVert \bm{\Lambda}^{(t)}\rrVert^2\;.
    \label{eq:theorem:nonconvex_convergence3} 
\end{align}
 Finally, we use Lemmas~\ref{lemma:sconv-converg} and~\ref{lemma:ge} to bound the gradient error of the outer level:
 \begin{align}
   \Psi\left(\w^{(t+1)},\al^*\left(\w^{(t+1)}\right) \right) &  \leq \Psi\left(\w^{(t)},\al^*\left(\w^{(t)}\right)\right) + \left(\frac{L_\Psi\eta^2}{2} - \eta\right) \llVert \nabla\Psi\left(\w^{(t)},\al^*\left(\w^{(t)}\right)\right) \rrVert^2 \nonumber\\
     & + \frac{L_\Psi\eta^2}{2}\cdot A_\Psi^2\exp\left(-\frac{K}{\kappa}\right) \llVert \al^{(0)}\left(\w^{(t)}\right) - \al^*\left(\w^{(t)}\right)\rrVert^2.
    \label{eq:theorem:nonconvex_convergence4} 
 \end{align}
 Now, by rearranging the terms and summing both sides over $t$ we will have:
\begin{align}
    \eta\left(1-\frac{L_\Psi\eta}{2}\right) \sum_{t=0}^{T-1} \left\lVert \nabla\Psi\left(\w^{(t)},\right.\right. & \left.\left.\al^*\left(\w^{(t)}\right)\right) \right\rVert^2 \nonumber\\
    & \leq  \Psi\left(\w^{(0)},\al^*\left(\w^{(0)}\right)\right) - \Psi\left(\w^*,\al^*\left(\w^*\right)\right)  \nonumber\\ 
    & \quad + \frac{L_\Psi\eta^2A_\Psi^2}{2}\exp\left(-\frac{K}{\kappa}\right)\sum_{t=0}^{T-1} \llVert \al^{(0)}\left(\w^{(t)}\right) - \al^*\left(\w^{(t)}\right)\rrVert^2\;.\label{eq:nonconv_last}
\end{align}
Then, considering the choice for $\eta$, which makes $\eta\left(1-\frac{L_\Psi\eta}{2}\right)\geq \frac{\eta}{2}$, and divide both sides of~(\ref{eq:nonconv_last}) by $T$, we have:
\begin{align}
    \frac{1}{T}\sum_{t=0}^{T-1} \llVert \nabla\Psi\left(\w^{(t)},\al^*\left(\w^{(t)}\right)\right) \rrVert^2 &\leq \frac{2}{\eta T}\left(\Psi\left(\w^{(0)},\al^*\left(\w^{(0)}\right)\right) - \Psi\left(\w^*,\al^*\left(\w^*\right)\right)\right) \nonumber \\
    &\quad + \frac{L_\Psi\eta A_\Psi^2}{ T}\exp\left(-\frac{K}{\kappa}\right)\sum_{t=0}^{T-1} \llVert \al^{(0)}\left(\w^{(t)}\right) - \al^*\left(\w^{(t)}\right)\rrVert^2\;.
\end{align}
\end{proof}

\section{Meta-data of Datasets}\label{app:exp}
The following tables show the meta-data related to the datasets we used and their sensitive features.
\begin{table}[ht!]
\parbox{.45\linewidth}{
\centering
\begin{tabular}{cc|c|c|c|}
\cline{3-5}
 &  & \multicolumn{2}{c|}{Groups} & \multirow{2}{*}{Total} \\ \cline{3-4}
 &  & Female & Male &  \\ \hline
\multicolumn{1}{|c|}{\multirow{2}{*}{Labels}} & +1 & 1196 & 6912 & 8108 \\ \cline{2-5} 
\multicolumn{1}{|c|}{} & -1 & 9352 & 15101 & 24453 \\ \hline
\multicolumn{2}{|c|}{Total} & 10548 & 22013 & 32561 \\ \hline
\end{tabular}
\caption{Adult train dataset with gender as the sensitive feature}
}
\hfill
\parbox{.45\linewidth}{
\centering
\begin{tabular}{cc|c|c|c|}
\cline{3-5}
 &  & \multicolumn{2}{c|}{Groups} & \multirow{2}{*}{Total} \\ \cline{3-4}
 &  & Female & Male &  \\ \hline
\multicolumn{1}{|c|}{\multirow{2}{*}{Labels}} & +1 & 473 & 2627 & 3100 \\ \cline{2-5} 
\multicolumn{1}{|c|}{} & -1 & 3674 & 5887 & 9561 \\ \hline
\multicolumn{2}{|c|}{Total} & 4147 & 8514 & 12661 \\ \hline
\end{tabular}
\caption{Adult test dataset with gender as the sensitive feature}
}
\end{table}
\begin{table}[h!]
\centering
\begin{tabular}{cc|c|c|c|c|c|c|}
\cline{3-8}
 &  & \multicolumn{5}{c|}{Groups} & \multirow{2}{*}{Total} \\ \cline{3-7}
 &  & AIE & API & Black & White & Other &  \\ \hline
\multicolumn{1}{|c|}{\multirow{2}{*}{Labels}} & +1 & 37 & 265 & 402 & 7380 & 24 & 8108 \\ \cline{2-8} 
\multicolumn{1}{|c|}{} & -1 & 275 & 697 & 2645 & 20614 & 222 & 24453 \\ \hline
\multicolumn{2}{|c|}{Total} & 312 & 962 & 3047 & 27994 & 246 & 32561 \\ \hline
\end{tabular}
\caption{Adult train dataset with race as the sensitive feature. AIE: Amer-Indian-Eskimo, API: Asian-Pac-Islander }
\end{table}
\begin{table}[h!]
\centering
\begin{tabular}{cc|c|c|c|c|c|c|}
\cline{3-8}
 &  & \multicolumn{5}{c|}{Groups} & \multirow{2}{*}{Total} \\ \cline{3-7}
 &  & AIE & API & Black & White & Other &  \\ \hline
\multicolumn{1}{|c|}{\multirow{2}{*}{Labels}} & +1 & 16 & 104 & 132 & 2827 & 21 & 3100 \\ \cline{2-8} 
\multicolumn{1}{|c|}{} & -1 & 107 & 237 & 1049 & 8082 & 86 & 9561 \\ \hline
\multicolumn{2}{|c|}{Total} & 123 & 341 & 1181 & 10909 & 107 & 12661 \\ \hline
\end{tabular}
\caption{Adult test dataset with race as the sensitive feature. AIE: Amer-Indian-Eskimo, API: Asian-Pac-Islander}
\end{table}

\begin{table}[h!]

\parbox{.45\linewidth}{
\centering
\begin{tabular}{cc|c|c|c|}
\cline{3-5}
 &  & \multicolumn{2}{c|}{Groups} & \multirow{2}{*}{Total} \\ \cline{3-4}
 &  & Female & Male &  \\ \hline
\multicolumn{1}{|c|}{\multirow{2}{*}{Labels}} & +1 & 422 & 1416 & 1838 \\ \cline{2-5} 
\multicolumn{1}{|c|}{} & -1 & 609 & 2831 & 1006 \\ \hline
\multicolumn{2}{|c|}{Total} & 1031 & 4247 & 5278 \\ \hline
\end{tabular}
\caption{COMPAS dataset with sex as the sensitive feature}
}
\hfill 
\parbox{.45\linewidth}{
\centering
\resizebox{0.45\textwidth}{!}{%
\begin{tabular}{cc|c|c|c|}
\cline{3-5}
 &  & \multicolumn{2}{c|}{Groups} & \multirow{2}{*}{Total} \\ \cline{3-4}
 &  & Caucasian & Not Caucasian &  \\ \hline
\multicolumn{1}{|c|}{\multirow{2}{*}{Labels}} & +1 & 859 & 979 & 1838 \\ \cline{2-5} 
\multicolumn{1}{|c|}{} & -1 & 1244 & 2196 & 1006 \\ \hline
\multicolumn{2}{|c|}{Total} & 2103 & 3175 & 5278 \\ \hline
\end{tabular}}
\caption{COMPASS dataset with race as a sensitive feature}
}
\end{table}
\end{document}